\documentclass[11pt]{article}

\usepackage[utf8]{inputenc}
\usepackage{amsthm,amsmath,amsfonts,amssymb,mathtools}
\usepackage{graphicx}
\usepackage{algorithm2e}
\RestyleAlgo{ruled}
\usepackage{algorithmic}
\usepackage{enumitem}
\usepackage{comment}
\usepackage[dvipsnames]{xcolor}
\usepackage[colorlinks=true,linkcolor=Maroon,
    urlcolor=MidnightBlue,citecolor=MidnightBlue]{hyperref}
\usepackage[numbers]{natbib}
\usepackage[noabbrev,capitalize]{cleveref}
\usepackage{appendix}
\crefname{thm}{Theorem}{Theorems}
\Crefname{thm}{Theorem}{Theorems}
\crefname{prop}{Proposition}{Propositions}
\Crefname{prop}{Proposition}{Propositions}
\crefname{cor}{Corollary}{Corollaries}
\Crefname{cor}{Corollary}{Corollaries}
\crefname{defn}{Definition}{Definitions}
\Crefname{defn}{Definition}{Definitions}
\crefname{assume}{Assumption}{Assumptions}
\Crefname{assume}{Assumption}{Assumptions}
\crefname{ex}{Example}{Examples}
\Crefname{ex}{Example}{Examples}
\crefname{condi}{Condition}{Conditions}
\Crefname{condi}{Condition}{Conditions}
\allowdisplaybreaks

\usepackage{microtype}
\usepackage[margin=1in]{geometry}
\usepackage{newtxtext,newtxmath}   

\usepackage{authblk}

\newtheorem{thm}{Theorem}[section]

\newtheorem{prop}{Proposition}[section]

\newcommand{\vep}{\varepsilon}

\newcommand{\lmn}{\lambda_{\mathrm{min}}}
\newcommand{\lmx}{\lambda_{\mathrm{max}}}
\newcommand{\wps}{\psi}

\newcommand{\wrh}{\rho}
\newcommand{\hr}{\rho}
\newcommand{\ptac}{\mathcal{P}_2^{\mathrm{ac}}}

\usepackage{my-macros}
\numberwithin{definition}{section}
\numberwithin{example}{section}
\numberwithin{remark}{section}
\numberwithin{proposition}{section}
\numberwithin{corollary}{section}
\numberwithin{observation}{section}
\usepackage{tikz}
\usetikzlibrary{arrows.meta,calc}
\usetikzlibrary{positioning}

\usepackage[font=small,labelfont=bf]{caption}

\providecommand{\keywords}[1]{\small \textbf{\textit{Keywords---}} #1}

\begin{document}
	
	\title{No-Regret Generative Modeling via Parabolic Monge-Amp\`{e}re PDE}
	
	\author{Nabarun Deb}
	\author{Tengyuan Liang}
	\affil{The University of Chicago}
	
	\maketitle
	
	\begin{abstract}
		We introduce a novel generative modeling framework based on a discretized parabolic Monge-Amp\`{e}re PDE, which emerges as a continuous limit of the Sinkhorn algorithm commonly used in optimal transport. Our method performs iterative refinement in the space of Brenier maps using a mirror gradient descent step. We establish theoretical guarantees for generative modeling through the lens of no-regret analysis, demonstrating that the iterates converge to the optimal Brenier map under a variety of step-size schedules. As a technical contribution, we derive a new Evolution Variational Inequality tailored to the parabolic Monge-Amp\`{e}re PDE, connecting geometry, transportation cost, and regret. Our framework accommodates non-log-concave target distributions, constructs an optimal sampling process via the Brenier map, and integrates favorable learning techniques from generative adversarial networks and score-based diffusion models. As direct applications, we illustrate how our theory paves new pathways for generative modeling and variational inference.
	\end{abstract}
	
	
	\section{Introduction}\label{sec:intro}
	
	Probabilistic generative models have demonstrated impressive empirical performance in modeling complex probability distributions for tasks deemed challenging under conventional statistical frameworks. Over the past decade, many probabilistic generative models have been proposed and analyzed, partly inspired by the theory of optimal transport (OT), including generative adversarial networks (GANs) \cite{arjovsky2017wasserstein,goodfellow2020generative,hur2024reversible,liang2021well,lubeck2022neural}, as well as flow-based and diffusion-based probabilistic models (DPMs) \cite{chen2024probability, chen2022sampling, liang2024denoising, song2020score, song2021maximum, song2019generative}, among others. 
	
	At a high level, GANs and DPMs aim to model and learn the push-forward map between a \emph{reference distribution} $\nu$, which is easy to sample from, such as a Gaussian distribution, and the \emph{target probability distribution} $\mu$. With the learned push-forward map $\boldsymbol{\tau}$, one can generate new samples $X$ from the target distribution by transforming fresh samples $Y \sim \nu$, specifically $X = \boldsymbol{\tau}(Y)$. Despite abundant empirical evidence demonstrating their effectiveness in practice, these generative models possess certain unfavorable theoretical limitations, which we discuss below.
	
	\begin{itemize}
		\item[(i)] Ease of sampling: GANs model the push-forward map as a one-step map, namely $\mu = \boldsymbol{\tau}\# \nu$, where sampling is easy; DPMs, on the other hand, represent the push-forward map as a composition of a series of non-linear transformations, expressed as $\mu = (\boldsymbol{\tau}_{1} \circ \boldsymbol{\tau}_{2} \circ \cdots \circ \boldsymbol{\tau}_{T} )\# \nu$, where sampling is less transparent since it requires inverting a sequence of non-linear maps known as denoising steps.
		\item[(ii)] Ease of learning: Brenier's theorem \cite{brenier1991polar} demonstrates that the optimal transport map $\boldsymbol{\tau} = \nabla \psi$ corresponds to a convex potential $\psi$. However, it is known in GANs that enforcing convexity constraints to regularize learning presents technical challenges. Additionally, learning the one-step map from uncoupled data $X \sim \mu, Z \sim \nu$ proves difficult. In contrast, DPMs propose decomposing the task into a sequence of sub-tasks where each map is a small deviation from the identity map, namely $\boldsymbol{\tau}_k \approx \mathrm{id}$, $k=1, \ldots, T$. Furthermore, for each step $k$, the deviation $\boldsymbol{\tau}_k - \mathrm{id}$ can be learned reliably from coupled data constructed using diffusions.
	\end{itemize}
	
	In this paper, we introduce a novel generative modeling framework inspired by parabolic Monge-Amp\`{e}re partial differential equation (PDE), a nonlinear PDE that naturally arises in OT as the limit of the Sinkhorn algorithm. Notably, our theoretical framework simultaneously integrates the ease of sampling and the ease of learning, two key features discussed earlier.
	
	The core idea is as follows: we refine the Brenier map $\nabla \psi_k$ iteratively using a discretized version of a mirror gradient flow, thus constructing a chain $\nabla \psi_1, \nabla \psi_2, \ldots, \nabla \psi_T \stackrel{T\rightarrow \infty}{\rightarrow} \nabla \psi$. In particular, let $f(x) :=  -\log(\dd \mu/\dd x)$ be the target density, and define the density at step $k$ as $\rho_k := (\nabla \psi_k)\# e^{-g}$ with $g(x) := -\log(\dd \nu/\dd x)$. We then study the following vector field iterations
	\begin{align*}
		\frac{\nabla \psi_{k+1} - \nabla \psi_{k}}{\eta_k} =  - \boldsymbol{\xi}_k, ~~\text{where}~~\boldsymbol{\xi}_k := \nabla  \big( \log (\rho_k/e^{-f}) \circ \nabla \psi_k \big) \;.
	\end{align*}
	Here, $\boldsymbol{\xi}_k$ resembles a (preconditioned) gradient step under a (local) mirror geometry.
	Equivalently, the Brenier potential admits an iterative refinement of the form of a (discretized) Monge-Amp\`{e}re PDE, 
	$$\frac{\psi_{k+1} - \psi_k}{\eta_k} = - f \circ \nabla \psi_k + g + \log\det( \nabla^2 \psi_k ) \;,$$
	which is the limit of the Sinkhorn algorithm under certain scalings \cite{berman2020sinkhorn, deb2023wasserstein}.  We note that explicitly evaluating the $\log{\mbox{det}}$ term in the above display can be avoided either by using the classification trick (see \cref{prop:log-density-ratio}) or the score-matching trick (see \cref{prop:score-matching}).
	
	Two modeling advantages of our idea are: (i) our iterative refinement acts directly on the space of Brenier maps, and thus the limit $(\nabla \psi_T)\# \nu \stackrel{T\rightarrow \infty}{\rightarrow} \mu$ provides an optimal one-step sampling of the target measure in the context of Brenier's theorem; (ii) at each iteration, one can learn the gradient $\boldsymbol{\xi}_k$ without worrying about enforcing notions of convexity/cyclic monotonicity. The convexity of $\psi_{k+1}$ will follow as a consequence of the convexity of $\psi_{k}$ and the small step-size $\eta_k$. These speak to the ease of sampling and learning, respectively. 
	
	As for theoretical contributions, we provide a refined no-regret analysis for the proposed generative modeling framework, without requiring the target distribution $\mu$ to be log-concave. Our regret analysis provides new technical tools to study generative models by introducing a new Evolution Variational Inequality (EVI) \cite{Ambrosio2008, guo2022online, salim2020wasserstein}---which connects geometry, transportation cost, and regret---tailored to the parabolic Monge-Amp\`{e}re PDE. Our novel regret analysis significantly expands the theory of generative models to accommodate non-log-concave measures.
	
	\subsection{Organization} 
	In \cref{sec:pmainmotapp}, we provide background and motivation for parabolic Monge-Amp\`{e}re PDE. In \cref{sec:Bregman}, we introduce and study the properties of Bregman divergence on the space of probability measures, as it serves as the key geometry with respect to which we shall study convergence. Specifically, \cref{sec:threeP} features a new three-point identity for Bregman divergences, while \cref{sec:BregKL} presents a new notion of convexity of the KL divergence with respect to a Wasserstein Bregman divergence (introduced in \cref{eg:mirror-G} below). In \cref{sec:regcon}, we present all our convergence and regret analyses: (i) \cref{sec:avgit} provides guarantees for average iterate convergence, (ii) \cref{sec:newEVI} introduces a new EVI for the KL and Bregman divergences, and (iii) \cref{sec:rblic} contains regret bounds and non-asymptotic rates for the last iterate. In \cref{sec:applications}, we discuss how our proposed discretized parabolic PDE can be employed to design novel algorithms for generative modeling (see \cref{Sec:genmod}) and variational inference (see \cref{sec:viform}). In \cref{sec:simulations}, we demonstrate the efficacy of our algorithms in sampling from a Gaussian mixture and provide a number of practical guidelines for stability, memory, and running time efficiency. The Appendix contains the proofs of all our main results and also includes an additional numerical experiment.
	
	\subsection{Notations}
	We will write $e^{-f}$ for the density of our target distribution $\mu$ supported on some $\mathcal{X}\subseteq \R^d$ and $e^{-g}$ for the density of the reference distribution $\nu$ supported on $\mathcal{Y}\subseteq\R^d$. Typically for simple $e^{-g}$ like the Gaussian, we will have $\mathcal{Y}=\R^d$. To help the reader distinguish between the domains, all integrals on $\mathcal{X}$ will be over the variable $x$ and all integrals on $\mathcal{Y}$ will be over the variable $y$. Only $\mu,\nu$ are reserved for probability measures, and the other Greek symbols $\rho$ and $\pi$ (with subscripts) will always refer to probability densities. All vector fields will be denoted using boldface notation, for instance, $\boldsymbol{\xi}, \boldsymbol{\tau}: \R^d \rightarrow \R^d$. 
	
	The push-forward $\boldsymbol{\tau}\#\mu$ for a vector field $\boldsymbol{\tau}$ denotes the distribution of $\boldsymbol{\tau}(X)$ when $X\sim\mu$. With slight notational abuse, we will sometimes directly apply the $\#$ notation on a probability density, say $\rho$, in which case $\boldsymbol{\tau}\#\rho$ will also denote a density function. $KL(\rho|\pi)$ and $W_2(\rho,\pi)$ denote the standard KL divergence and the Wasserstein distance between $\rho$ and $\pi$ respectively. For densities $\pi$ and $\rho$, we will use the standard definitions of Brenier potentials and Brenier maps between $\pi$ and $\rho$ from \cite[Theorem 2.12]{villani2009optimal}. Typically, these potentials will be represented with $\phi$ or $\psi$ (and the corresponding maps by $\nabla \phi$ and $\nabla \psi$), with appropriate indexing to make the probability measures involved transparent. $\ptac(\R^d)$ will denote the space of probability densities supported on some subset of $\R^d$ with finite second moments, and we will sometimes refer to it as the $2$-Wasserstein space. 
	
	Given a continuously differentiable function $\psi:\R^d\to\R$, we write its Fenchel dual as $\psi^*(x)=\sup_{y\in\mathcal{Y}} (\langle x,y\rangle-\psi(y))$. We will call a twice differentiable function $\psi:\R^d\to\R$ a $m$-strongly convex function if $\inf_y \lmn(\nabla^2 \psi(y))\ge m$. Similarly, we will call it $M$-smooth if $\sup_y \lmx(\nabla^2 \psi(y))\le M$. Here $\lmn$ and $\lmx$ denote the minimum and the maximum eigenvalues of a matrix. The following function spaces will be used : (i) $\mathcal{C}^2$ for twice continuously differentiable functions, (ii) $\mathcal{C}_c^{\infty}$ for infinitely differentiable functions with compact support, and (iii) $L^2(\rho)$ for square integrable functions with respect to some probability density $\rho$. Given a functional $\mathcal{F}:\mathcal{U}\to\R$ for some function space $\mathcal{U}$, $\frac{\delta \mathcal{F}}{\delta u}$ will denote the first variation/Gateaux derivative of $\mathcal{F}$ at $u$. Finally $\mathbb{S}_+^{d \times d}$ will denote the set of $d\times d$ symmetric positive semi-definite matrices; for instance, $\nabla^2 \psi(y) \in \mathbb{S}_+^{d \times d}, \forall y \in \cY$.

	\section{Parabolic Monge-Amp\`{e}re PDE}\label{sec:pmainmotapp}
	
	Recall that we wish to sample from a complicated \emph{target distribution} with density $e^{-f}\in\ptac(\R^d)$. We also have a \emph{reference distribution} with density $e^{-g}\in\ptac(\R^d)$, which is easy to sample from, for instance, a standard Gaussian. By Brenier's Theorem \cite{brenier1991polar,McCann1995}, there exists a Brenier potential $\psi:\R^d\to\R$ such that $\nabla \psi\#e^{-g}=e^{-f}$. The corresponding (static) Monge-Amp\`{e}re equation \cite{ampere1819memoire,monge1784memoire} thus reads:
	$$-f(\nabla\psi(y))+g(y)+\log\det(\nabla^2\psi(y))=0 \;.$$
	Learning $\nabla\psi$ is particularly useful for sampling as one can generate samples from $e^{-f}$ by first sampling from $Y\sim e^{-g}$ (say a Gaussian) and then applying the optimal transform $\nabla\psi(Y)$. However, 
	as the above is a non-linear second-order PDE, given $e^{-f}$ and $e^{-g}$, solving the equation for $\nabla\psi$ is highly computationally intensive. An alternate approach adopted in the PDE literature is to study instead the natural dynamic version, that is, the parabolic PDE 
	\begin{align}\label{eq:dualpma}
		\frac{\partial \psi_t}{\partial t}(y)=-f(\nabla\psi_t(y))+g(y)+\log{\mbox{det}(\nabla^2 \psi_t(y))} \;.
	\end{align}
	The existence of solutions, uniqueness, smoothness, and exponential convergence of $\nabla\psi_t\to \nabla \psi$ as $t\to\infty$ have been studied extensively in the PDE literature; see e.g. \cite{abedin2020exponential,berman2020sinkhorn,kim2012parabolic}. Therefore, $\{\nabla \psi_t\}_{t\ge 0}$ can be viewed as a continuum of iterative refinements of some initial map $\nabla\psi_0:\R^d\to\R^d$, eventually leading to the target Brenier map $\nabla \psi$. Furthermore, define $\rho_t:=\nabla\psi_t\# e^{-g}$. Then by sampling $Y\sim e^{-g}$, $\nabla\psi_t(Y)\sim\rho_t\approx e^{-f}$ (see \cite{deb2023wasserstein}). Therefore, $\nabla\psi_t(Y)$ can be viewed as an approximate one-step sample from the target density $e^{-f}$, for large $t$. Due to these favorable properties of the parabolic PDE \eqref{eq:dualpma}, it is natural to study an implementable time discretization of \eqref{eq:dualpma}. In this paper, we therefore study the following natural (forward) update rule:
	
	\begin{equation}\label{eq:discretedual}
		\frac{\psi_{k+1}(y) - \psi_k(y)}{\eta_k} = - f(\nabla \psi_k(y)) + g(y) + \log\det( \nabla^2 \psi_k (y)) = -f(\nabla \psi_k(y)) - \log \rho_k(\nabla \psi_k(y))
	\end{equation}
	
	where $\{\eta_k\}_{k\ge 0}$ denotes the set of step-sizes, $\psi_0:\R^d\to\R$ is some strongly convex function, and 
	\begin{equation}\label{eq:measupdate}
		\rho_k:=(\nabla \psi_k)\#e^{-g}.
	\end{equation}
	In this paper, we study the choice of step-sizes $\eta_k$, and its effect on the convergence of $\rho_k \rightarrow e^{-f}$. The implementation of \eqref{eq:discretedual} based on neural networks is possible using a classification technique as in \cite[Proposition 1]{goodfellow2014generative}, or a score matching technique as in \cite[Theorem 1]{Hyvarinen2005}; we will discuss these in detail in Section~\ref{sec:applications}. 
	
	Beyond the natural connections to the Monge-Amp\`{e}re equation, the continuous and discrete systems \eqref{eq:dualpma} and \eqref{eq:discretedual} have several compelling properties that make them useful for applications. We present a few of them below: (i) they can lead to faster flows in the simple Gaussian setting than perhaps the widely studied dynamical system, namely the canonical Fokker-Planck equation \cite{fokker1914mittlere,planck1917satz}; (ii) they are closely related to the dynamics of the celebrated Sinkhorn algorithm \cite{cuturi2013sinkhorn,franklin1989scaling};  (iii) they can be viewed as steepest descent of the $KL(\cdot|e^{-f})$ functional with respect to a certain weighted local $L^2$-metric; (iv) they can be used to construct novel generative learning algorithms that combine ease of sampling with ease of learning as mentioned in the introduction; and (v) they provide a new paradigm for variational inference. The first two points will be discussed in this section, while the latter three are deferred to \cref{sec:applications}.
	
	\subsection{Illustrative Example for Gaussians}\label{sec:uniGill}
	We provide a simple example on the rate of convergence of the aforementioned parabolic PDE when both $e^{-f}$ and $e^{-g}$ are densities of centered univariate Gaussians, say $N(0,1)$ and $N(0,\lambda^2)$, $\lambda<1$. Suppose that $\psi_0(y)=y^2/2$. The continuous time system \eqref{eq:dualpma}, after taking an additional space derivative then reduces to
	$$\frac{\partial \psi_t'(y)}{\partial t}=- \psi_t''(y) \psi_t'(y)+\frac{1}{\lambda^2}y+\frac{\psi_t'''(y)}{\psi_t''(y)} \;.$$
	
	As $\psi_0'$ is linear, it is easy to check that all $\psi_t'$s are linear. So, let $\psi_t'(y)=c_ty$ for some $c_t$. The above differential equation then becomes a Riccati equation $\dot{c_t}=-c_t^2+\frac{1}{\lambda^2}$ with $c_0=1$,
	which implies $c_t=(1/\lambda)\tanh((t/\lambda)+\tanh^{-1}(\lambda))$. Writing $\rho_t=(\psi_t')\#e^{-g}$ gives that $\rho_t$ is the density of the $N(0,\sigma_t^2)$ distribution where $\sigma_t=\tanh((t/\lambda)+\tanh^{-1}(\lambda))$. It is easy to check that for the standard Fokker-Planck system, the corresponding solution, say $\rho_F(t)$ is the density of $N(0,\sigma_{F,t}^2)$ where $\sigma_{F,t}^2=1-(1-\lambda^2)e^{-2t}$. As a result, 
	$$\frac{1-\sigma_{F,t}^2}{1-\sigma_t^2}=\frac{1-\lambda^2}{4}\left(e^{-t(1+\lambda^{-1})-\tanh^{-1}(\lambda)}+e^{t(\lambda^{-1}-1)+\tanh^{-1}(\lambda)}\right)^2\to \infty \quad \mbox{as}\,\, t\to\infty \;,$$
	because $\lambda<1$. Therefore, $\sigma_t^2$ converges to the target variance $1$ faster than the Fokker-Planck variance $\sigma_{F,t}^2$. For the discretization scheme \eqref{eq:discretedual}, it can be shown that the $\rho_k$'s (see \eqref{eq:measupdate}), under appropriate initialization, are all centered Gaussian distributions, and the variance converges to the target variance $1$ locally at an exponential rate. For brevity, we defer further details to \cref{sec:pfpmainmotapp}. 
	
	In the sequel, we provide in \cref{sec:regcon} global regret analyses of the discretization \eqref{eq:discretedual} under general assumptions on $e^{-f}$ and $e^{-g}$, extending far beyond Gaussianity and log-concavity.
	
	\subsection{Connection to Sinkhorn}
	\label{sec:pmacondis}
	
	The exploding literature on generative modeling over the recent years has witnessed state of the art methods that rely on entropy regularized optimal transport (EOT); see e.g.~\cite{cao2021don,de2021diffusion,genevay2018learning,wang2021deep}. Given a temperature parameter $\epsilon>0$, the EOT problem is given by 
	\begin{equation}\label{eq:EOT}
		\pi^{\epsilon}:=\argmin_{\pi \in \Pi(e^{-f},e^{-g})} \left\{\frac{1}{2}\int \lVert x-y\rVert^2\gamma(x,y)\dd (x,y) + \epsilon KL(\gamma|e^{-f}\otimes e^{-g})\right\} \;, 
	\end{equation}
	where $\Pi(e^{-f},e^{-g})$ denotes the class of probability densities on $\R^d\times\R^d$ with marginal densities $e^{-f}$ and $e^{-g}$. The optimal $\pi^{\epsilon}$ from \eqref{eq:EOT} is also referred to as the (static) Schr\"{o}dinger Bridge (see \cite{schrodinger1935present}) between $e^{-f}$ and $e^{-g}$. By \cite{fortet1940resolution,ruschendorf1993note}, it is known that there exist potentials $\phi^{\epsilon}$ and $\psi^{\epsilon}$ such that 
	$\pi^{\epsilon}(x,y)=\exp\left(\frac{1}{\epsilon}\langle x,y\rangle-\frac{1}{\epsilon}\phi^{\epsilon}(x)-\frac{1}{\epsilon}\psi^{\epsilon}(y)-f(x)-g(y)\right)$.
	As $\pi^{\epsilon}\in \Pi(e^{-f},e^{-g})$, both $\phi^{\epsilon}$ and $\psi^{\epsilon}$ will satisfy certain fixed point equations. In particular, the equation for $\psi^{\epsilon}$ is given by 
	
	$$\psi^{\epsilon}(y)=\mathcal{V}^{\epsilon}[\psi^{\epsilon}](y), \quad \mbox{where} \quad 
	\mathcal{V}^{\epsilon}[\psi^{\epsilon}](y):=\epsilon \log{\int \frac{\exp\left(\frac{1}{\epsilon}\langle x,y\rangle\right)}{\int \exp\left(\frac{1}{\epsilon}\langle x,y'\rangle-\frac{1}{\epsilon}\psi^{\epsilon}(y')-g(y')\right)\dd y'}e^{-f(x)}\dd x} \;.$$
	
	Perhaps the most popular algorithm for solving \eqref{eq:EOT} is the Sinkhorn algorithm (see \cite{ruschendorf1995convergence,cuturi2013sinkhorn}) which solves for $\phi^{\epsilon}$, $\psi^{\epsilon}$, $\pi^{\epsilon}$ using a natural iterative procedure to get $\{\phi_k^{\epsilon},\psi_k^{\epsilon},\pi_k^{\epsilon}\}_{k\ge 0}$. In particular, following the $\mathcal{V}^{\epsilon}$ notation above, the $\psi_k^{\epsilon}$s are updated as 
	\begin{equation}\label{eq:Sinkupdate}
		\psi_{k+1}^{\epsilon}=\mathcal{V}^{\epsilon}[\psi_k^{\epsilon}](y) \;.
	\end{equation}
	The parabolic PDE \eqref{eq:dualpma} can now be viewed as the \emph{scaling limit} of the Sinkhorn algorithm in the low temperature regime $\epsilon\to 0$ when \emph{the number of iterations $k$ scales like $t/\epsilon$ for some $t>0$}. The following proposition (also see \cite[Lemma 4.2]{berman2020sinkhorn} and \cite[Lemma 4.6]{deb2023wasserstein}) demonstrates why such a scaling limit is natural. We provide a proof in \cref{sec:pfpmainmotapp} that is self-contained.
	\begin{prop}\label{prop:Sinkapprox}
		Suppose $\psi:\R^d\to\R$ is a uniformly strongly convex $\mathcal{C}^2$ function. Let $\tilde{\psi}^{\epsilon}(y):=\mathcal{V}^{\epsilon}[\psi](y)$. Then, as $\epsilon\to 0$, we have:
		$$\lim_{\epsilon \rightarrow 0} \frac{\tilde{\psi}^{\epsilon}(y) - \psi(y)}{\epsilon}=- f(\nabla \psi(y)) + g(y) + \log\det( \nabla^2 \psi(y))  \;.$$
	\end{prop}
	Assuming that the Sinkhorn iterates $\{\psi_k^{\epsilon}\}$'s are $\mathcal{C}^2$ smooth and strongly convex, \cref{prop:Sinkapprox} coupled with the Sinkhorn update rule \eqref{eq:Sinkupdate} yields the following approximation 
	$$\frac{\psi_{k+1}^{\epsilon}(y)-\psi_k^{\epsilon}(y)}{\epsilon}\approx -f(\nabla \psi_k^{\epsilon}(y))+g(y)+\log\det(\nabla^2 \psi_k^{\epsilon}(y)),$$
	which approximately coincides with our discretization scheme \eqref{eq:discretedual} with constant step-sizes $\eta_k=\epsilon$. Moreover, the above display implies that $\psi_{k+1}^{\vep}$ and $\psi_k^{\vep}$ are $O(\vep)$ apart. As a result, by running the Sinkhorn algorithm for $k=t/\epsilon$ iterations, we expect to get a limiting curve, call it $\{\psi_t\}_{t\ge 0}$. Using similar heuristics as used in the discretization of the Cauchy problem \cite{merigotdiscretization,santambrogio2017euclidean}, we expect that the LHS $\epsilon^{-1}(\psi_{t/\epsilon+1}^{\epsilon}(y)-\psi_{t/\epsilon}^{\epsilon})$ will converge to the time derivative $\partial_t \psi_t(y)$ whereas the RHS $-f(\nabla \psi_{t/\epsilon}^{\epsilon}(y))+g(y)+\log\det(\nabla^2 \psi_{t/\epsilon}^{\epsilon}(y))$ should converge to the negative descent direction $-f(\nabla \psi_t(y))+g(y)+\log\det(\nabla^2\psi_t(y))$. This yields the parabolic PDE in \eqref{eq:dualpma}.  
	
	\section{Bregman Divergences over Probability Measures}
	\label{sec:Bregman}
	Our regret approach to studying the discretized parabolic Monge-Amp\`{e}re equation (and its application to generative modeling) relies on a novel three-point identity that delineates geometry, transportation cost, and regret. This key analytic tool depends on the concept of Bregman divergences on the $2$-Wasserstein space.  To motivate it, we recall the definition of Bregman divergence on the standard Euclidean space ---  
	For a diffeomorphism $\phi(\cdot)$ and its Fenchel conjugate $\phi^*(\cdot)$, define 
	\begin{equation}\label{eq:EucliDBreg}
		D_{\phi}(x|\nabla \phi^*(y)):= \phi(x)-\phi(\nabla \phi^{\star}(y))-\langle y,x-\nabla \phi^*(y)\rangle = \phi(x)+\phi^*(y)-\langle x,y\rangle \;.
	\end{equation}
	It is easy to check that if $\phi(\cdot)$ is strictly convex, then $D_{\phi}(x|\nabla \phi^*(y))\ge 0$ with equality if and only if $x=\nabla\phi^*(y)$. 
	The Bregman divergence is a fundamental notion in the celebrated field of Euclidean mirror descent (see \cite{Beck2003,Bubeck2021,Tzen2023}, and the references therein) and is known to yield nearly dimension-free rates in certain constrained optimization problems. The following definition extends the Bregman divergence (introduced first in \cite{li2021transport}; also see \cite{aubin2022mirror,bonet2024mirror,deb2023wasserstein,karimi2024sinkhorn}) from the finite-dimensional Euclidean space to the infinite-dimensional $2$-Wasserstein space.
	
	\begin{definition}[Bregman divergence over probability measures]\label{def:Breg}
		Given a function $\Gamma:\ptac(\R^d) \to\R$ which admits a well-defined first variation, we define the Bregman divergence functional for $\Gamma$ as follows:
		$$B_{\Gamma}(\rho_2|\rho_1):=\Gamma(\rho_2)-\Gamma(\rho_1)-\int \frac{\delta \Gamma}{\delta \rho}(\rho_1)(x)\,(\rho_2-\rho_1)(x)\,\dd x \;,$$
		for probability density functions $\rho_1,\rho_2\in\ptac(\R^d)$, provided the first variation is integrable under $\rho_1,\rho_2$. We will refer to $\Gamma$ as the mirror function. 
	\end{definition}
	
	The above definition of Bregman divergence on the $2$-Wasserstein space is reminiscent of the usual Euclidean Bregman divergence defined in \eqref{eq:EucliDBreg}. The main difference here is that the role of the Euclidean derivative is played by the first variation. For notational clarity, we use different letters $B$ and $D$ for divergences on the Wasserstein space and on the Euclidean space, respectively. Below, we provide some examples of Bregman divergence over probability measures.
	
	\begin{example}[Entropy as mirror]
		The most popular example of Bregman divergence is when the mirror is the entropy function $H(\rho):=\int \rho(x)\log{\rho(x)}\dd x$, in which case 
		$$B_{H}(\rho_2|\rho_1)=H(\rho_2)-H(\rho_1)-\int \frac{\delta H}{\delta \rho}(\rho_1)(x)\,(\rho_2-\rho_1)(x)\,\dd x=KL(\rho_2|\rho_1) \;,$$
		for probability density functions $\rho_1,\rho_2\in\ptac(\R^d)$, provided the first variation is integrable under $\rho_1,\rho_2$. 
	\end{example}
	
	\begin{example}[Wasserstein distance as mirror]
		\label{eg:mirror-G}
		Another example of Bregman divergence that will be useful in this paper is for the function $G(\rho):=(1/2)W_2^2(\rho,e^{-g})$, where $e^{-g}$ is the reference distribution. To simplify the Bregman divergence for this example, we write ${\phi_{\rho}}$ to be a Brenier potential from $\rho$ to $e^{-g}$. It is known (see \cite[Proposition 7.17]{Santambrogio2015} and \cite[Corollary 10.2.7]{Ambrosio2008}) that 
		$$\frac{\delta G}{\delta\rho}(\rho)(x)=\frac{1}{2}\lVert x\rVert^2 - \phi_{\rho}(x) \;.$$
		Therefore, the Bregman divergence of $G$ is given by 
		\begin{align*}
			B_G(\rho_2|\rho_1)&=G(\rho_2)-G(\rho_1)-\int \left(\frac{1}{2}\lVert x\rVert^2-\phi_{\rho_1}(x)\right)\,(\rho_2-\rho_1)(x) \dd x \;,
		\end{align*}
		for probability densities $\rho_1,\rho_2\in\ptac(\R^d)$.
	\end{example}
	
	\subsection{A New Three-Point Identity}\label{sec:threeP}
	
	The Bregman divergence function $B_G(\cdot|\cdot)$ satisfies non-negativity and can be expressed as the expectation of the standard Bregman divergence with respect to an appropriate probability measure. We defer the reader to \cref{lem:Bregprop} and \cref{lem:Bregprop2} for further details. To keep the discussion streamlined, we present the most important property of $B_G(\cdot|\cdot)$ here, namely a new three-point identity, that will be instrumental to our regret analysis. The proof is deferred to \cref{sec:Bregp}.

	\begin{lemma}[A new Bregman three-point identity]
		\label{lem:mirrid}
		Consider probability densities $\rho_1,\rho_2,\pi\in\ptac(\R^d)$, and recall the definitions of $G(\cdot) = (1/2) W_2^2(\cdot, e^{-g})$, its induced $B_G(\cdot | \cdot)$, and the Brenier potentials $\phi_{\rho_i}$ for $i=1,2$ as in Example~\ref{eg:mirror-G}. Define in addition $\pi_i:=(\nabla \phi_{\rho_i})\#\pi$ for $i=1,2$, and another Bregman divergence $B_{G_\pi} (\cdot | \cdot)$ induced by the function $G_{\pi}(\cdot):=(1/2)W_2^2(\cdot,\pi)$. Set $\psi_{\rho_i}=\phi_{\rho_i}^*$ for $i=1,2$. Then 
		$$\int (\psi_{\rho_2}-\psi_{\rho_1})(y)(\pi_1-e^{-g})(y)\dd y=B_G(\pi|\rho_1)-B_G(\pi|\rho_2)+B_{G_{\pi}}(\pi_1|\pi_2) \;.$$
	\end{lemma}
	
	Let us briefly examine why the above lemma will be useful in our regret bounds. We will use it with $\rho_2=\rho_{k+1}$ and $\rho_1=\rho_{k}$ (see \eqref{eq:measupdate}). The LHS will be the first order term in the expansion of the map $\rho\mapsto \mathrm{KL}(\rho|e^{-f})$ around $\rho_k$. For the RHS, the term $B_G(\pi|\rho_1)-B_G(\pi|\rho_2) \equiv B_G(\pi|\rho_k)-B_G(\pi|\rho_{k+1})$ will be a ``telescoping term" as we vary $k$. The term $B_{G_{\pi}}(\pi_1|\pi_2)$ is for the ``transportation cost" to go from $\rho_1$ to $\rho_2$, or equivalently $\rho_k$ to $\rho_{k+1}$. Curiously, note that the transportation cost does not use the same mirror map $G$ but instead requires a new mirror map $G_{\pi}$. Moreover, it is not directly a Bregman divergence between $\rho_1$ and $\rho_2$ but instead between the image of $\pi$ under $\nabla \phi_{\rho_1}$ and $\nabla\phi_{\rho_2}$. \cref{fig:transport} summarizes the maps and probability measures involved in the transportation cost term.
	
	As a concrete example, we can simplify the identity in \cref{lem:mirrid} when all the probability measures involved are Gaussians. Let $\pi\sim N(0,\sigma_{\pi}^2)$, $e^{-g}\sim N(0,\sigma_g^2)$, $\rho_1\sim N(0,\sigma_1^2)$ and $\rho_2\sim N(0,\sigma_2^2)$. Then $\psi'_{\rho_2}(y)=(\sigma_2/\sigma_g)y$ and $\psi'_{\rho_1}(y)=(\sigma_1/\sigma_g)y$. As a result, $\pi_1\sim N(0,(\sigma_g\sigma_{\pi}/\sigma_1)^2)$ and $\pi_2\sim N(0,(\sigma_g\sigma_{\pi}/\sigma_2)^2)$. It is easy to check that 
	$$\int (\psi_{\rho_2}-\psi_{\rho_1})(y)(\pi_1-e^{-g})(y)\dd y=\frac{1}{2}\sigma_g(\sigma_1-\sigma_2)\left(1-\frac{\sigma_{\pi}^2}{\sigma_1^2}\right),$$
	$$B_G(\pi|\rho_1)=\frac{1}{2}\frac{\sigma_g}{\sigma_1}(\sigma_{\pi}-\sigma_1)^2, \quad B_G(\pi|\rho_2)=\frac{1}{2}\frac{\sigma_g}{\sigma_2}(\sigma_{\pi}-\sigma_2)^2, \quad B_{G_{\pi}}(\pi_1|\pi_2)=\frac{1}{2} \sigma_g \sigma_\pi^2 \sigma_2  \left(\frac{1}{\sigma_1}-\frac{1}{\sigma_2}\right)^2.$$
	
	Direct computations can now be used to verify \cref{lem:mirrid}.
	
	\begin{figure}
		\centering
		\begin{tikzpicture}[>=Stealth, thick, font=\small]
			
			\tikzset{
				redbox/.style={
					draw=red, fill=red!10, thick, rectangle, rounded corners,
					minimum width=3.2cm, minimum height=4.2cm
				},
				greenbox/.style={
					draw=green!60!black, fill=green!15, thick, rectangle, rounded corners,
					minimum width=3.2cm, minimum height=4.2cm
				},
				arrowstyle/.style={
					->, thick, black
				}
			}
			
			\node[redbox] (targetbox) {};
			\node[above=0.1cm of targetbox.north] {\textbf{Target Domain $\mathcal{X}$}};
			\node at (targetbox) {
				\begin{tabular}{c}
					\(\rho_1\) \\[1.8em]
					\(\rho_2\) \\[1.8em]
					\(\pi\)
				\end{tabular}
			};
			
			\node[greenbox, right=6.5cm of targetbox] (refbox) {};
			\node[above=0.1cm of refbox.north] {\textbf{Reference Domain $\mathcal{Y}$}};
			\node at (refbox) {
				\begin{tabular}{c}
					\(e^{-g}\) \\[1.8em]
					\(\pi_2\) \\[1.8em]
					\(\pi_1\)
				\end{tabular}
			};
			
			\coordinate (rho1_src) at ([xshift=1.6cm, yshift=-1.05cm] targetbox.north);
			\coordinate (rho2_src) at ([xshift=1.6cm, yshift=-2.0cm] targetbox.north);
			\coordinate (pi_src)   at ([xshift=1.6cm, yshift=-3.1cm] targetbox.north);
			
			\coordinate (eg_tgt)   at ([xshift=-1.6cm, yshift=-1.05cm] refbox.north);
			\coordinate (pi1_tgt)  at ([xshift=-1.6cm, yshift=-2.2cm] refbox.north);
			\coordinate (pi2_tgt)  at ([xshift=-1.6cm, yshift=-3.1cm] refbox.north);
			
			\coordinate (rho1_in) at ([xshift=0.0cm, yshift=-1.20cm] targetbox.north);
			\coordinate (rho2_in) at ([xshift=0.0cm, yshift=-1.95cm] targetbox.north);
			
			\coordinate (pi1_in)  at ([xshift=-0.0cm, yshift=-2.25cm] refbox.north);
			\coordinate (pi2_in)  at ([xshift=-0.0cm, yshift=-3.05cm] refbox.north);
			
			\draw[arrowstyle] (rho1_src) -- node[above, yshift=2pt, pos = 0.6] {\(\nabla \phi_{\rho_1}\)} (eg_tgt);
			\draw[arrowstyle] (rho2_src) -- node[below, yshift=-2pt, pos = 0.6] {\(\nabla \phi_{\rho_2}\)} (eg_tgt);
			\draw[arrowstyle] (pi_src) -- node[above, sloped, yshift=2pt, pos=0.4] {\(\nabla \phi_{\rho_2}\)} (pi1_tgt);
			\draw[arrowstyle] (pi_src) -- node[below, sloped, yshift=-2pt, pos=0.4] {\(\nabla \phi_{\rho_1}\)} (pi2_tgt);
			
			\draw[->, thick, red]
			(rho1_in) -- node[left, text=black, xshift=2.5pt] {\tiny{\(B_G(\rho_1 | \rho_2)\)}} (rho2_in);
			
			\draw[<-, thick, green!70!black]
			(pi1_in) -- node[right, text=black, xshift=-2.5pt] {\tiny{\(B_{G_{\pi}}(\pi_1 | \pi_2)\)}} (pi2_in);
			
		\end{tikzpicture}
		\caption{\small The probability densities $\rho_1,\rho_2,\pi$ live on the target domain $\mathcal{X}$ while the densities $\pi_1,\pi_2,e^{-g}$ live on the reference domain $\mathcal{Y}$.  $\nabla\phi_{\rho_i}$ is the Brenier map from $\rho_i$ to $e^{-g}$, $i=1,2$. As $\pi_i=\nabla \phi_{\rho_i}\#\pi$, Brenier's Theorem \cite{brenier1991polar} implies that   $\nabla\phi_{\rho_i}$ is also the Brenier map from $\pi$ to $\pi_i$, for $i=1,2$. The telescoping term in \cref{lem:mirrid}, $B_G(\pi|\rho_1)-B_G(\pi|\rho_2)$ involves densities $\rho_1,\rho_2,\pi$ all of which live on the target space. One might naturally expect the transportation cost to be $B_G(\rho_1 | \rho_2)$. In sharp contrast, the corresponding term in \cref{lem:mirrid} is $B_{G_{\pi}}(\pi_1 | \pi_2)$ where $\pi_i$ lives on the reference domain and is the image of $\pi$ under the Brenier map $\nabla\phi_{\rho_i}$, $i=1,2$.}
		\label{fig:transport}
	\end{figure}
	
	\begin{remark}\label{rem:tpl}
		\rm
		In the mirror descent literature on Hilbert spaces, analogs of the above lemma are usually termed three-point inequalities or Bregman proximal inequalities; see \cite[Lemma 29]{bonet2024mirror}, \cite[Lemma 3.2]{chen1993convergence}, \cite[Lemma 1]{lan2011primal}; also see \cite{frigyik2008functional,lu2018relatively}. For instance, in the Euclidean setting, the Bregman descent update for optimizing a function $F:\R^d\to\R$ with respect to a mirror $\phi:\R^d\to\R$ is given by 
		$$\nabla \phi(x_{k+1})=\nabla \phi(x_k)-\eta_k \nabla F(x_k) \;.$$
		For any $z\in\R^d$, regret bounds for Euclidean mirror descent rely on the following three-point identity: 
		\begin{equation}\label{eq:Euclid3pt}
			\langle \nabla \phi(x_{k})-\nabla \phi(x_{k+1}),x_k-z\rangle=D_{\phi}(z|x_k)-D_{\phi}(z|x_{k+1})+D_{\phi}(x_k|x_{k+1}) \;,
		\end{equation}
		which relates the three points $z,x_k$ and $x_{k+1}$. \cref{lem:mirrid} can be viewed as an extension of \eqref{eq:Euclid3pt} to the $2$-Wasserstein space which does not possess a Hilbert structure. Note that \cref{lem:mirrid}  captures the relationship between five different measures $\pi,\pi_1,\pi_2,\rho_1,\rho_2$, instead of three. As $\pi_i$ is still a function of $\pi,\rho_i,e^{-g}$, we continue to use the term ``three-point lemma'' for \cref{lem:mirrid}; also see Figure~\ref{fig:transport}. Moreover, in contrast to the RHS of \eqref{eq:Euclid3pt} where all the divergences use the same mirror $\phi$, the RHS of \cref{lem:mirrid} requires the use of two different mirrors $G(\cdot)=(1/2)W_2^2(\cdot,e^{-g})$ and $G_{\pi}(\cdot)=(1/2)W_2^2(\cdot,\pi)$. To the best of our knowledge, the existing three-point equalities/inequalities for other forms of Bregman descent on the $2$-Wasserstein space, e.g.~\cite[Lemma 3]{aubin2022mirror}, \cite[Lemma 8]{han2025variational}, \cite[Theorem 1]{leger2021gradient}, also use the same mirror map on both the telescoping and the transportation cost terms. In that sense, \cref{lem:mirrid} is new and is particularly well suited to analyze the discretization \eqref{eq:discretedual}.
	\end{remark}

	\subsection{Bregman and KL: Relative Convexity}\label{sec:BregKL}
	Our next main result in this section introduces a new notion of convexity for the KL divergence with respect to the Bregman divergence $B_G(\cdot|\cdot)$. The proof is deferred to \cref{sec:Bregp}. 
	\begin{lemma}[Bregman and KL: relative convexity]\label{lem:unineq}
		Consider the reference distribution $e^{-g}$ where $g(\cdot)$ is $\lambda$-strongly convex. Further, let $\rho$ be a probability density such that the Brenier potential $\psi_{\rho}(\cdot)$ from $e^{-g}$ to $\rho$ satisfies the condition $\sup_y \lmx(\nabla^2 \psi_{\rho}(y))\le \beta$, for some $\beta>0$. Fix any absolutely continuous probability measure $\pi$ such that the Brenier potential $\psi_{\pi}$ from $e^{-g}$ to $\pi$ is strictly convex and twice differentiable. We also assume that for any $1\le i\le d$, it holds that the map 
		$y\mapsto e^{-g(y)}\left((\partial_i\psi_{\rho}^\star)(\nabla\psi_{\pi}(y))-y_i\right)$
		vanishes as $|y_i|\to\infty$. 
		Then the following inequality holds: 
		$$KL(\pi|\rho)\ge \frac{\lambda}{\beta} B_G(\pi|\rho) \;.$$
	\end{lemma}
	
	\begin{remark}
		\rm
		Recall $H(\rho)=\int \rho(x)\log{\rho(x)}\dd x$ for $\rho\in\ptac(\R^d)$. 
		The Bregman convexity in \cref{lem:unineq} can be rewritten as $B_H(\pi|\rho)\ge (\lambda/\beta)B_G(\pi|\rho)$. We believe it introduces a new notion of convexity for the KL  divergence. A crucial feature in our Lemma~\ref{lem:unineq} is that we do not require strong log-concavity on either $\pi$ or $\rho$; we merely assume its Brenier potentials $\psi_\rho, \psi_{\pi}$ to be strongly convex, a much milder assumption.
		
		The two most popular existing notions of convexity, namely geodesic convexity (see \cite[Chapter 9]{Ambrosio2008}, \cite{liero2023fine}) which says $KL(\pi|\rho)\ge (\gamma/2)W_2^2(\pi,\rho)$ and generalized geodesic convexity (see \cite[Chapter 11]{Ambrosio2008}, \cite{salim2020wasserstein}) which says $KL(\pi|\rho)\ge (\gamma/2)\int \lVert \nabla \phi_{\rho}^*(y)-\nabla\phi_{\pi}^*(y)\rVert^2 e^{-g(y)}\,\dd y$, both require $-\log{\rho}$ to be $\gamma$-strongly convex. However, since we apply \cref{lem:unineq} with $\rho=\rho_k$, such strong log-concavity is not available to us even if $\psi_k$ and $\psi_k^*$ are uniformly strongly convex.
		
		Finally, we note that the strong convexity assumption of $g(\cdot)$ in \cref{lem:unineq} can be replaced by the assumption $D_g(y_1|y_2) \ge \lambda D_{\phi_{\rho}^*}(y_2|y_1) $. This is evident from the proof of \cref{lem:unineq}. 
		Such assumptions directly on divergences instead of metrics are common in the mirror descent literature; see e.g.,~\cite[Definition 10.2.8]{chewi2023log}. However, for ease of presentation, we stick with the usual notion of strong convexity on $g$. 
	\end{remark}

	\section{Regret and Convergence Analysis}\label{sec:regcon}
	
	In this section, we will discuss regret bounds and convergence rates for $\rho_k$ (see \eqref{eq:measupdate}) as a function of the total number of iterations $T$. 
	
	\subsection{Average Iterate Convergence}\label{sec:avgit}
	We start off deriving a result that holds under very weak assumptions and allows for large, constant step-sizes. We  prove the KL distance between a weighted mixture of $\{\rho_k\}$'s and $e^{-f}$ converges to $0$, at a $O(T^{-1})$ rate. The proof is deferred to \cref{sec:pfmainres}.
	\begin{thm}\label{thm:dualcon}
		Suppose $\{\wps_k\}_{k=0}^T$ are generated according to the discretization \eqref{eq:discretedual}. Then
		$$\E_{Y \sim e^{-g}}[ \wps_{k+1}(Y) - \wps_{k}(Y) ] = -\eta_k KL\big(\rho_k| e^{-f} \big) \;.$$
		Next define $S_T:=\sum_{k=0}^{T-1} \eta_k$ and the mixture distribution  $\bar{\rho}_T:=\sum_{k=0}^{T-1} \frac{\eta_k}{S_T}\rho_k$. We then have: 
		$$KL(\bar{\rho}_T|e^{-f})\le S_T^{-1}\E_{Y\sim e^{-g}}[\wps_{0}(Y)-\wps_T(Y)] \;.$$
	\end{thm}
	
	\begin{remark} 
		\rm
		Suppose we use time-invariant step-sizes, i.e., $\eta_k=\eta$ for some $\eta>0$. Then \cref{thm:dualcon} implies 
		$$KL(\bar{\rho}_T|e^{-f})\le \frac{1}{\eta T}\E_{Y\sim e^{-g}}[\wps_0(Y)-\wps_T(Y)] \;,$$
		thereby yielding a $O(1/T)$ rate of convergence for $\bar{\rho}_T$ to $e^{-f}$, if $\E_{Y\sim e^{-g}}\psi_T(Y)$ is bounded in $T$. Note that \cref{thm:dualcon} does not impose any convexity restrictions on $f$ and $g$, nor does it require strong convexity/smoothness of the Brenier potentials $\psi_k$. This result, inspired by the Polyak-average scheme, also allows for large step-sizes.
	\end{remark} 
	
	While \cref{thm:dualcon} provides convergence of the average iterate $\bar{\rho}_T$ to $e^{-f}$, it is more natural to use the $\rho_k$ (for $k$ large) directly to approximate $e^{-f}$. This is called the last iterate convergence. In the sequel, we derive last iterate convergence under the additional regularity assumptions on the Brenier potentials $\{\psi_k\}_{k\geq 0}$. A main technical tool is a new EVI that will be used to conduct regret analyses.

	\subsection{New Evolution Variational Inequality}\label{sec:newEVI}
	
	We now adopt a regret analysis approach to study the theoretical guarantees of generative modeling with Monge-Amp\`{e}re PDE. Our method aligns with the tradition of online learning literature \cite{Cesa2006,Gordon1999,shalev2012online} and, in particular, is inspired by \cite{guo2022online}, where EVIs were utilized to study regret bounds on the Wasserstein space. To establish these regret bounds, the primary tool will be the Bregman divergence on the Wasserstein space introduced in \cref{sec:Bregman}; specifically, refer to \cref{def:Breg}. 
	
	Recall the Brenier potentials $\{\psi_k\}_{k\ge 0}$ from \eqref{eq:discretedual} and the corresponding probability densities $\{\rho_k\}_{k\ge 0}$ from \eqref{eq:measupdate}. Our first main result in this section presents a new one-step EVI that will be applied multiple times in deriving our final regret bounds. 
	
	\begin{lemma}\label{lem:basicineq}
		Consider $\{\wps_k\}_{k=0}^T$ generated from discretization \eqref{eq:discretedual}.
		Assume that $m_k:=\inf_{y}\lmn(\nabla^2 \psi_k(y))>0$, namely, the potentials $\psi_k(\cdot)$ are $m_k$-strongly convex. Define $\boldsymbol{\xi}_k (y):=\nabla_y ((f+\log\wrh_k)(\nabla \psi_k(y)))$. Given any probability density $\pi\in\ptac(\R^d)$, let $\pi_k$ denote the probability density function of $(\nabla \psi_k^*)\#\pi$. In this setting, the following inequality holds:
		$$KL(\wrh_k|e^{-f})-KL(\pi|e^{-f}) \le \frac{1}{\eta_k}(B_G(\pi|\wrh_k)-B_G(\pi|\wrh_{k+1}))+\frac{\eta_k}{2 m_{k+1}}\int \lVert \boldsymbol{\xi}_k(y)\rVert^2 \pi_k(y)\dd y-KL(\pi|\wrh_k) \;.$$
	\end{lemma}
	
	\begin{proof}
		By a direct computation coupled with \eqref{eq:discretedual}, we have: 
		\begin{align*}
			\;\;\;\;KL(\wrh_k|e^{-f}) - KL(\pi|e^{-f}) &=\int (f+\log{\wrh_k})(x)(\wrh_k-\pi)(x)\dd x-KL(\pi|\wrh_k) \\ &=-\frac{1}{\eta_k}\int (\psi_{k+1}-\psi_k)(y)e^{-g(y)}\dd y+\frac{1}{\eta_k}\int (\psi_{k+1}-\psi_k)(y)\,\pi_k(y)\dd y -KL(\pi|\wrh_k) \;.
		\end{align*}
		By using \cref{lem:mirrid} with $\pi$, $\rho_1=\rho_k$, and $\rho_2=\rho_{k+1}$, we note that:
		\begin{equation}\label{eq:noneuclid3pt}
			\int (\psi_{k+1}-\psi_k)(y)(\pi_k-e^{-g})(y)\dd y=B_G(\pi|\rho_k)-B_G(\pi|\rho_{k+1})+B_{G_{\pi}}(\pi_k|\pi_{k+1}) \;,
		\end{equation}
		where $G_{\pi}(\cdot)=(1/2)W_2^2(\cdot,\pi)$. As a result, 
		$$KL(\wrh_k|e^{-f}) - KL(\pi|e^{-f})=\frac{1}{\eta_k}\left[B_G(\pi|\wrh_k)-B_G(\pi|\wrh_{k+1})+B_{G_{\pi}}(\pi_k|\pi_{k+1})\right]-KL(\pi|\wrh_k) \;.$$
		By \cref{lem:Bregprop2}, we further have
		\begin{align*}
			B_{G_{\pi}}(\pi_k|\pi_{k+1})&\le \frac{1}{2}\sup_x\lmx(\nabla^2 \psi^*_{k+1}(x))\int \lVert \nabla \psi_{k+1}(y)-\nabla \psi_k(y)\rVert^2\,\pi_k(y)\dd y \\ &=\frac{\eta_k^2}{2}\left(\inf_y \lmn(\nabla^2 \psi_{k+1}(y))\right)^{-1}\int \lVert \boldsymbol{\xi}_k(y)\rVert^2\pi_k(y) \dd y \\ &\le \frac{\eta_k^2}{2 m_{k+1}}\int \lVert \boldsymbol{\xi}_k(y)\rVert^2\pi_k(y) \dd y \;,
		\end{align*}
		where the equality in the second display uses \cref{obs:elemprop} and \eqref{eq:discretedual}. By combining the above displays, we have the following: 
		$$KL(\wrh_k|e^{-f})-KL(\pi|e^{-f}) \le \frac{1}{\eta_k}(B_G(\pi|\wrh_k)-B_G(\pi|\wrh_{k+1}))+\frac{\eta_k}{2 m_{k+1}}\int \lVert \boldsymbol{\xi}_k(y)\rVert^2 \pi_k(y)\dd y-KL(\pi|\wrh_k) \;.$$
		This completes the proof.
	\end{proof}
	
	We note that EVIs are popular in the PDE literature and have been used in statistics to address questions on optimization and sampling; see \cite{Dieleveut2020,guo2022online,salim2020wasserstein}. To the best of our knowledge, none of these results cover EVIs with Bregman divergence over the space of probability measures.

	\subsection{Regret Bounds and Last Iterate Convergence}\label{sec:rblic}
	We now apply \cref{lem:basicineq} to obtain different regret bounds. The proofs are deferred to \cref{sec:pfmainres}. We present these bounds in progressive order by imposing slightly stronger regularity conditions on the reference distribution $e^{-g}$ and the Brenier potentials $\{\psi_k\}_{k\geq 0}$. With more regularity conditions, the regret bounds will be stronger.
	
	Our first result provides a $O(\sqrt{T})$-regret bound that requires only a uniform lower bound on the convexity of the Brenier potentials $\{\psi_k\}_{k\geq 0}$.
	
	\begin{thm}\label{thm:regbd1}
		Recall the notation from \cref{lem:basicineq}. Suppose  $m_k\ge m$ and $\\ \int\lVert \boldsymbol{\xi_k}(y)\rVert^2\pi_k(y)\dd y \le A$ for all $k=0,1,\ldots ,T+1$, and some $m,A>0$. Then by choosing $\eta_k=T^{-1/2}$ for all $k=0,1,\ldots ,T$, we get: 
		$$\sum_{k=0}^T \big(KL(\wrh_k|e^{-f})-KL(\pi|e^{-f})\big)\le \sqrt{T}\bigg(B_G(\pi|\wrh_0)+\frac{A}{2m}\bigg).$$
	\end{thm}
	
	While \cref{thm:regbd1} shows that $O(\sqrt{T})$-regret is attainable with time-invariant step-sizes, it is natural to ask if the regret can be improved under stronger conditions. To wit, we note that \cref{thm:regbd1} only requires (i) uniform strong convexity of the potentials $\{\psi_k\}$ across iterations $k=0,1,2,\ldots ,T$, and (ii) uniform boundedness of the vector fields $\{\boldsymbol{\xi}_k\}$ in the $L^2(\pi_k)$ norm. So far, we have not imposed any smoothness conditions on the $\{\psi_k\}_{k\geq 0}$. In the following result, we show that under additional smoothness conditions on $\{\psi_k\}_{k\geq 0}$ and a strong log-concavity assumption on the reference $e^{-g}$, logarithmic regret bound is attainable under a carefully chosen sequence of time-varying step-sizes.
	
	\begin{thm}\label{thm:regbd2}
		Recall the notation from \cref{lem:basicineq}. Given a probability density $\pi\in\ptac(\R^d)$, let $\psi_{\pi}$ denote a Brenier potential from $e^{-g}$ to $\pi$. Also suppose $\sup_y\lmx(\nabla^2 \psi_k(y))\le M_k$, for some $M_k>0$ (i.e., the potentials $\psi_k$ are $M_k$-smooth) and the map 
		$y\mapsto e^{-g(y)}((\partial_i \psi_k^*)(\nabla \psi_{\pi}(y))-y_i)$
		vanishes as $|y_i|\to \infty$, for all $k=0,1,\ldots ,T$. Finally, let $e^{-g}$ be a strongly log-concave distribution where $g(\cdot)$ is $\lambda$-strongly convex. 
		
		\begin{itemize}
			\item[(i)] Define the following:
			$$S_{k+1}:=\lambda\sum_{i=0}^k \frac{1}{M_i}\quad \mbox{and} \quad \eta_k:=S_{k+1}^{-1}.$$
			Then we have
			$$\sum_{k=0}^T (KL(\rho_k|e^{-f})-KL(\pi|e^{-f})) \le \frac{1}{2}\sum_{k=0}^T \frac{1}{m_{k+1} S_{k+1}}\int \lVert \boldsymbol{\xi}_k(y)\rVert^2 \pi_k(y)\dd y \;.$$
			
			\item[(ii)] Now suppose that 
			$M_k\le M$, $m_k\ge m$, and $\int \lVert \boldsymbol{\xi}_k(y)\rVert^2\pi_k(y) \dd y \le A$ for all $k=0,1,\ldots ,T$, and some $m,M\in (0,\infty)$. Then by choosing $\eta_k^{-1}:=(k+1)\lambda/M$, we have: 
			$$\sum_{k=0}^T \big(KL(\wrh_k|e^{-f})-KL(\pi|e^{-f})\big) \le \frac{1}{2}\frac{A M}{\lambda m}(1+\log{(T+1)}) \;.$$
		\end{itemize}
	\end{thm}
	
	We emphasize that our $O(\log T)$-regret bound does not assume that the target distribution $e^{-f}$ is strongly log-concave or that it satisfies log-Sobolev or Poincar\'{e} type inequalities. Instead, the time-independent constant in our bound depends on the ratio of the uniform smoothness parameter $M$ and the uniform strong convexity parameter $m$ on Brenier's potential $\psi$. This ratio $M/m$ can be viewed as an analog of the condition number (see \cite{Ma2021,zhang2023improved}) in the absence of any functional inequality assumptions on the target $e^{-f}$. Again, here we merely require strong convexity/smoothness of the Brenier potential $\psi$, which is known to be convex as a premise even for non-log-concave target distributions $e^{-f}$. Therefore, our assumptions are much weaker than log-concavity of $e^{-f}$. While it may be possible to improve the condition number dependence to $\sqrt{M/m}$ by proposing an accelerated version of \eqref{eq:discretedual}, we leave that for future research.

	\begin{remark}
		Note that the step-sizes chosen in Theorems \ref{thm:regbd1} and \ref{thm:regbd2} are fundamentally different. While the former uses time-invariant step-sizes, the latter involves time-decaying ones. Therefore, it is instructive to understand the average regret bounds in terms of the \emph{effective time} $S_T=\sum_{k=0}^{T-1} \eta_k$ (defined earlier in \cref{thm:dualcon}). For \cref{thm:regbd1}, $S_T=\sqrt{T}$ and the average regret 
		$$\frac{1}{T}\sum_{k=0}^T \big(KL(\wrh_k|e^{-f})-KL(\pi|e^{-f})\big) = O(T^{-1/2}) = O(S_T^{-1}).$$
		So the average regret in this case decays polynomially. On the other hand, for \cref{thm:regbd2}, $S_T\approx \log T$ and the average regret 
		$$\frac{1}{T}\sum_{k=0}^T \big(KL(\wrh_k|e^{-f})-KL(\pi|e^{-f})\big) = O(T^{-1}\log{T}) = O(S_T e^{-S_T}).$$
		Therefore, in contrast to \cref{thm:regbd1},  the average regret in this case decays nearly at an exponential rate in terms of the effective time. 
	\end{remark}
	
	\begin{remark}
		\rm
		In \cite[Theorem 2]{han2025variational}, the authors obtain a $O(\sqrt{T})$-regret bound (similar to our \cref{thm:regbd1}) for a different mirror-descent-like algorithm, where the mirror map is also a certain KL divergence functional. Unlike us, they impose a uniform logarithmic Sobolev inequality type assumption on the KL divergence functional, which makes their approach very different from ours. The same mirror-descent algorithm was also studied in \cite{aubin2022mirror}, where the authors 
		assume a stronger version of \cref{lem:unineq} (both an upper and lower bound). However, the authors do not establish sufficient conditions as we do in \cref{lem:unineq}. These papers use fixed step-sizes and do not recover logarithmic regret bounds like our \cref{thm:regbd2}. 
	\end{remark}
	
	\noindent The final result in this section establishes the rate of convergence of the last iteration $\rho_T$ to $e^{-f}$ in terms of the Bregman divergence $B_G$, under the same assumptions that lead to logarithmic regret in \cref{thm:regbd2}. The choice of step-sizes $\{\eta_k\}$ is, however, different in the two cases. 
	
	\begin{thm}\label{thm:ptcon}
		Consider the same assumptions as in \cref{thm:regbd2}, part (ii). 
		
		\noindent (i) The update from $\rho_k$ to $\rho_{k+1}$ satisfies the following asymptotic contraction: 
		$$\limsup_{\eta_k\to 0} \frac{1}{\eta_k}\frac{B_G(e^{-f}|\rho_{k+1})-B_G(e^{-f}|\rho_k)}{B_G(e^{-f}|\rho_k)}\le -\frac{\lambda}{M}.$$
		
		\noindent (ii) By choosing $\eta_k = (C M B_G(e^{-f}|\rho_0)\log{T})/(\lambda T), \forall k \leq T$ with $C>0$, the following holds for all large enough $T>0$:
		$$B_G(e^{-f}|\rho_T)\le \left(\frac{1}{T^{C B_G(e^{-f}|\rho_0)}} +\frac{C M^2 A}{2\lambda^2 m}\frac{\log{T}}{T}\right)B_G(e^{-f}|\rho_0).$$
	\end{thm}
	
	The first part of \cref{thm:ptcon} demonstrates an ``asymptotic contraction'' of the sequence of probability measures $\{\rho_k\}$; that is, if $\eta_k\to 0$, then 
	$$B_G(e^{-f}|\rho_{k+1})\le \left(1- \eta_k \cdot \lambda / M + o(\eta_k) \right)B_G(e^{-f}|\rho_k).$$
	This implies that the local contraction rate of $\rho_k$ is governed by the strong convexity parameter of $g$ and the smoothness of the potentials $\{\psi_k\}_{k=0}^T$. Naturally, if the strong convexity parameter of $g$, specifically $\lambda$, is small, or if the smoothness parameter $M$ of the $\{\psi_k\}$s is large, the rate of asymptotic contraction slows down.
	
	The second part concerns the last iterate convergence of rate $O(T^{-1}\log T)$. It presents a non-asymptotic convergence rate of $\rho_T$ to $e^{-f}$ under the Bregman divergence $B_G(e^{-f}|\rho_T)$. When $B_G(e^{-f}|\rho_0)>0$, it indicates a nearly $O(T^{-1})$ rate up to logarithmic factors, as the first term in the bound can be made arbitrarily small by selecting $C$ large enough. Note that this bound also captures the fact that if $\rho_0$ and $e^{-f}$ are close, then so are $\rho_T$ and $e^{-f}$.
	
	\section{Applications}\label{sec:applications}
	\subsection{Generative Modeling and Neural-PDE}
	\label{Sec:genmod}
	
	In this section, we demonstrate how the discretized parabolic Monge-Amp\`{e}re PDE \eqref{eq:discretedual} can be used to design new generative modeling algorithms, that is, learning to sample from a target measure $\mu$ with density $e^{-f} := \dd \mu/\dd x$. It turns out our discretized parabolic Monge-Amp\`{e}re PDE integrates well with techniques established in current generative modeling literature: (i) learning density-ratio through logistic regression as in GANs \cite{goodfellow2020generative}; (ii) learning the score function via score matching as in DPMs \cite{Hyvarinen2005, song2019generative}. Additionally, our discretized parabolic Monge-Amp\`{e}re PDE can be implemented as a neural PDE using a residual neural network architecture \cite{he2016deep} with standard auto-differentiation libraries.
	
	Recall the discretized Monge-Amp\`{e}re PDE \eqref{eq:discretedual}; one crucial step for implementation is to learn 
	\begin{itemize}
		\item[(i)] the log density-ratio function $h_k(x):=\log(\rho_k(x)/e^{-f(x)})$, or
		\item[(ii)] its derivative, namely the score functions $\boldsymbol{m}_k(x):= \nabla \log(\rho_k(x)/e^{-f(x)})$,
	\end{itemize}
	given data $X \sim e^{-f}$ and $\tilde{X} \sim \rho_k$ drawn i.i.d. from the target distribution and the simulated distribution respectively.
	
	The following two propositions directly address how to estimate each term, with the first taken from \cite[Proposition 1]{goodfellow2014generative}, and the second from \cite[Theorem 1]{Hyvarinen2005}. We shall design new generative modeling algorithms with neural networks based on these two propositions, inspired by the discretized parabolic Monge-Amp\`{e}re PDE.
	
	\medskip 
	
	\noindent \textbf{Neural-PDE via Logistic Regression.}
	
	\begin{prop}[Log density-ratio via logistic regression]
		\label{prop:log-density-ratio}
		Let $\rho, \pi \in \ptac(\R^d)$ and consider the following data-generating process $(X, L)\sim \gamma$: first, sample labels $L = 0$ or $1$ with equal probability, and then generate data given the label as $X|L=0 \sim \pi$ and $X|L=1 \sim \rho$. With the data $(X, L) \sim \gamma $, define the logistic loss functional
		\begin{align*}
			\mathcal{L}(h) := \E_{(X, L) \sim \gamma}\left[ L \log (1+e^{-h(X)}) +   (1-L) \log (1+e^{h(X)})\right] \;.
		\end{align*}
		Assume $(\rho+\pi)(x) > 0, \forall x \in \R^d$, then the unique minimizer $h^\ast: \R^{d} \rightarrow \R$ of $\mathcal{L}(h)$ satisfies $h^\ast(x) = \log(\rho(x)/\pi(x))$ .
	\end{prop}
	Proposition~\ref{prop:log-density-ratio} provides a procedure to estimate the log-density-ratio: given i.i.d. data $X_1, \cdots, X_n \sim e^{-f}$ and $\tilde{X}_1, \cdots, \tilde{X}_n \sim \rho_k = (\nabla \psi_k)\# e^{-g}$, augment to  labeled data $\{(X_i, L_i = 0)\}_{i=1}^n\cup\{( \tilde{X}_i, \tilde{L}_i = 1)\}_{i=1}^n$ and denote its empirical distribution as $\widehat{\gamma}_k$; specify a neural network $h(\cdot ; \theta): \R^d \rightarrow \R$ parametrized by $\theta$ and solve
	\begin{align}
		\label{eqn:neural-pde-logistic-reg}
		\widehat{h}_k (\cdot) := h( \cdot ; \hat{\theta})~, \text{where}~ \hat{\theta} \in  \argmin_{\theta} \E_{(X, L) \sim \widehat{\gamma}_k} \left[ L \log (1+e^{-h(X;\theta)}) +   (1-L) \log (1+e^{h(X;\theta)})\right] \;.
	\end{align}
	The minimizer $\widehat{h}_k$ approximates $h_k$, the true log density-ratio function. 
	
	In practice, one can implement $\widehat{h}_k (\cdot)$ with neural networks. This motivates a new generative modeling algorithm by iteratively refining the Brenier potential $\psi_k$ with a new neural network $\widehat{h}_k$, coupled with auto-differentiation, composition, and residual connections; See Algorithm~\ref{alg:neural-pde-logistic-reg} below for details.
	
	\medskip
	\noindent \textbf{Neural-PDE via Score Matching.}
	
	\begin{prop}[Score matching via Fisher divergence]
		\label{prop:score-matching}
		Consider $\rho \in \ptac(\R^d)$ and the corresponding data score function $\nabla \log \rho$.
		Let $x \in \mathbb{R}^d$ and $\boldsymbol{\sigma}_i(x)$ denote the $i$-th component of a vector-valued function.
		Consider the Fisher divergence functional $\mathcal{J}(\boldsymbol{\sigma}) :=  \frac{1}{2} \E_{X \sim \rho} [ \|  \nabla \log \rho(X) - \boldsymbol{\sigma}(X) \|^2 ]$.
		Assume $\boldsymbol{\sigma}$ is differentiable and that, $\forall i$, $\rho(x) \boldsymbol{\sigma}_i(x) \rightarrow 0$ as $|x_i| \rightarrow \infty$. Then
		$$\mathcal{J}(\boldsymbol{\sigma}) = \E_{X \sim \rho} \left[  \sum_{i=1}^d   \partial_i  \boldsymbol{\sigma}_i(X)  + \frac{1}{2} \boldsymbol{\sigma}_i(X)^2  \right] + \text{const \;.}$$
	\end{prop}
	
	Proposition~\ref{prop:score-matching} also sets forth the following procedure to estimate the score functions: given i.i.d. $\tilde{X}_1, \cdots, \tilde{X}_n \sim \rho_k = (\nabla \psi_k)\# e^{-g}$, denote its empirical distribution $\widehat{\rho}_k$, specify a neural network $\boldsymbol{\sigma}(\cdot; \omega) : \R^d \rightarrow \R^d$ parametrized by $\omega$ and solve
	\begin{align}
		\label{eqn:neural-pde-score-matching}
		\widehat{\boldsymbol{\sigma}}_k (\cdot) := \boldsymbol{\sigma}( \cdot ; \hat{\omega})~, \text{where}~ \hat{\omega} \in  \argmin_{\omega} \E_{X \sim \widehat{\rho}_k} \left[ \sum_{i=1}^d   \partial_i  \boldsymbol{\sigma}_i( X;\omega)  + \frac{1}{2} \boldsymbol{\sigma}_i(X;\omega)^2 \right] \;.
	\end{align}
	
	\begin{center}
		\begin{minipage}{0.7\textwidth}
			\begin{algorithm}[H]
				\caption{Monge-Amp\`{e}re Neural-PDE via Logistic Regression}
				\label{alg:neural-pde-logistic-reg}
				\SetKwInOut{Input}{Input}
				\SetKwInOut{Output}{Output}
				\Input{$T$, total number of steps; $\{\eta_k\}_{k \leq T}$, the step-sizes; initialize a neural network function $\psi_0:\R^d \rightarrow \R$, and set $k=0$.}
				\Output{A neural network function $\psi_T: \R^d \rightarrow \R$, and samples from $(\nabla \psi_T) \# e^{-g}$.}
				\While{$k < T$}{
					\textbf{Sampling step}: Given i.i.d. data $X_1, \cdots, X_n \sim e^{-f}$, sample $\tilde{X}_1, \cdots, \tilde{X}_n \sim \rho_k = (\nabla \psi_k)\# e^{-g}$, augment to $\{(X_i, L_i = 0)\}_{i=1}^n\cup\{( \tilde{X}_i, \tilde{L}_i = 1)\}_{i=1}^n$, and form its empirical distribution as $(X, L) \sim \widehat{\gamma}_k$\;
					
					\textbf{Learning step}: Estimate a neural network discriminator function $\widehat{h}_k$ as in \eqref{eqn:neural-pde-logistic-reg} \;
					
					\textbf{Neural Network update}: Define a new neural network with the residual architecture shown on the right
					$\psi_{k+1} = - \eta_k \widehat{h}_k \circ \nabla \psi_k + \psi_k$ \;
					
					Set $k \leftarrow k+1$ \;
				}
			\end{algorithm}
		\end{minipage}
		\begin{minipage}{.3\textwidth}
			\begin{tikzpicture}[
				roundnode/.style={rectangle, draw=green!60, fill=green!5, very thick, minimum size=7mm},
				squarednode/.style={rectangle, draw=red!60, fill=red!5, very thick, minimum size=5mm},
				]
				\node[squarednode]      (maintopic)                              {$\widehat{h}_k$};
				\node[roundnode]        (uppercircle)       [above=of maintopic] {$\psi_k$};
				\node[squarednode]      (leftsquare)       [left=of maintopic] {$\nabla$};
				\node[squarednode]      (rightsquare)       [right=of maintopic] {id};
				\node[roundnode]        (lowercircle)       [below=of maintopic] {$\psi_{k+1} = - \eta_k \widehat{h}_k \circ \nabla \psi_k + \psi_k$};
				\draw[->] (uppercircle.south) -- (leftsquare.north);
				\draw[->] (uppercircle.south) -- (rightsquare.north);
				\draw[->] (maintopic.south) -- (lowercircle.north);
				\draw[->] (leftsquare.east) -- (maintopic.west);
				\draw[->] (rightsquare.south) -- (lowercircle.north);
			\end{tikzpicture}
		\end{minipage}
	\end{center}
	
	The minimizer $\widehat{\boldsymbol{\sigma}}_k $ approximates the score function $\nabla \log \rho_k$; similarly, one can solve for $\widehat{\boldsymbol{\sigma}}$ that approximates the data score function $\nabla 
	\log e^{-f}$, and define
	\begin{align}
		\label{eqn:estimate-grad}
		\widehat{\boldsymbol{m}}_k := \widehat{\boldsymbol{\sigma}}_k - \widehat{\boldsymbol{\sigma}} \;.
	\end{align}
	This motivates a new generative modeling algorithm to refine the Brenier map $\boldsymbol{n}_k := \nabla \psi_k$ by iteratively adding neural network maps $\widehat{\boldsymbol{m}}_k$ using a residual connection, coupled with auto-differentiation and composition; we outline in Algorithm~\ref{alg:neural-pde-score-matching} below.

	\begin{remark}[Parabolic Monge-Amp\`{e}re discretizations in the literature]
		A discretization similar to \eqref{eq:discretedual} was studied in \cite{sulman2021domain,sulman2011efficient} with a time-invariant step-size sequence. However, the authors do not establish any convergence bounds. To the best of our knowledge, the integration of neural learning techniques from generative networks and score matching, as described in this section, is also novel within the parabolic PDE literature. This approach is practically useful because it avoids the otherwise cumbersome Hessian computation in \eqref{eq:discretedual}. Moreover, the use of parabolic PDEs in existing literature is primarily for adaptive mesh generation (see \cite{sulman2011optimal,budd2009moving,budd2009adaptivity}) to approximate solutions of other differential equations. In that sense, the connection between parabolic Monge-Amp\`{e}re and generative modeling is also new.
	\end{remark}

	\subsection{Variational Formulation of Parabolic PDE}\label{sec:viform}
	
	We first present a variational objective that leads to the discretization \eqref{eq:discretedual}. We will then use this variational objective to provide a new perspective on variational inference and also propose a new closed form iterative scheme for Gaussian variational inference.

	\noindent \textbf{Steepest Descent Interpretation.} In Euclidean space, the forward gradient descent can be viewed as the solution of a minimizing movement scheme involving an objective function and a regularizer (the squared Euclidean distance). In this section, we show that the forward discretization \eqref{eq:discretedual} can also be viewed as the solution of a variational problem that involves (i) an objective function, which will be the KL divergence to $e^{-f}$ and (ii) a regularizer which will be a divergence on the space of probability measures (to be defined below). This is in sharp contrast to usual Wasserstein gradient flows where the regularizer is the usual $2$-Wasserstein distance. To recast \eqref{eq:discretedual} as a variational problem, we will define the divergence on the tangent set  $\mathrm{Tan}_{e^{-g}}:=\overline{\{\nabla \psi: \psi\in \mathcal{C}_c^{\infty}\}}\subseteq L^2(e^{-g})$ where the closure is taken with respect to the $L^2(e^{-g})$ norm. 
	
	\begin{center}
		\begin{minipage}{0.7\textwidth}
			\begin{algorithm}[H]
				\caption{Monge-Amp\`{e}re Neural-PDE via Score Matching}
				\label{alg:neural-pde-score-matching}
				\SetKwInOut{Input}{Input}
				\SetKwInOut{Output}{Output}
				\Input{$T$, total number of steps; $\{\eta_k\}_{k \leq T}$, the step-sizes; initialize a vector-valued neural network $\boldsymbol{n}_0:\R^d \rightarrow \R^d$, and set $k=0$.}
				\Output{A vector-valued neural network function $\boldsymbol{n}_T : \R^d \rightarrow \R^d$, and samples from $(\boldsymbol{n}_T) \# e^{-g}$.}
				\While{$k < T$}{
					\textbf{Sampling step}: Obtain i.i.d. samples from $\rho_k = (\boldsymbol{n}_k) \# e^{-g}$, form the empirical measure $\widehat{\rho}_k$\;
					
					\textbf{Learning step}: Estimate a neural network score $\widehat{\boldsymbol{\sigma}}_k$ as in \eqref{eqn:neural-pde-score-matching} and define the corresponding $\widehat{\boldsymbol{m}}_k$ as in \eqref{eqn:estimate-grad} \;
					
					\textbf{Neural Network update}: Define a new neural network with the residual architecture shown on the right
					$\boldsymbol{n}_{k+1} = -\eta_k \nabla \boldsymbol{n}_k \cdot  (\widehat{\boldsymbol{m}}_k \circ \boldsymbol{n}_k)  + \boldsymbol{n}_k \;, $ here $\nabla \boldsymbol{n}_k: \R^d \rightarrow \mathbb{S}_+^{d \times d}$, and $\widehat{\boldsymbol{m}}_k \circ \boldsymbol{n}_k : \R^d \rightarrow \R^d$ and `$\cdot$' denotes the matrix-vector product  \;
					
					Set $k \leftarrow k+1$ \;
				}
			\end{algorithm}
		\end{minipage}
		\begin{minipage}{.3\textwidth}
			\begin{tikzpicture}[
				roundnode/.style={rectangle, draw=green!60, fill=green!5, very thick, minimum size=7mm},
				squarednode/.style={rectangle, draw=red!60, fill=red!5, very thick, minimum size=5mm},
				]
				\node[squarednode]      (maintopic)                              {$\widehat{\boldsymbol{m}}_k$};
				\node[roundnode]        (uppercircle)       [above=of maintopic] {$\boldsymbol{n}_k$};
				\node[squarednode]      (leftsquare)       [left=of maintopic] {$\nabla$};
				\node[squarednode]      (rightsquare)       [right=of maintopic] {id};
				\node[roundnode]        (lowercircle)       [below=of maintopic] {$\boldsymbol{n}_{k+1} = -\eta_k \nabla \boldsymbol{n}_k \cdot  (\widehat{\boldsymbol{m}}_k \circ \boldsymbol{n}_k)  + \boldsymbol{n}_k$};
				\draw[->] (uppercircle.south) -- (maintopic.north);
				\draw[->] (uppercircle.south) -- (leftsquare.north);
				\draw[->] (uppercircle.south) -- (rightsquare.north);
				\draw[->] (maintopic.south) -- (lowercircle.north);
				\draw[->] (leftsquare.south) -- (lowercircle.north);
				\draw[->] (rightsquare.south) -- (lowercircle.north);
			\end{tikzpicture}
		\end{minipage}
	\end{center}
	
	\begin{definition}[Linearized Wasserstein divergence]
		Let $\nabla \psi_1,\nabla \psi_2\in \mathrm{Tan}_{e^{-g}}$ such that $\nabla\psi_1$ is twice differentiable and strictly convex. Then the linearized Wasserstein divergence 
		$$LD_{W}(\nabla \psi_2;\nabla \psi_1) := \int (\nabla \psi_2(y)-\nabla \psi_1(y))^{\top}(\nabla^2 \psi_1(y))^{-1}(\nabla\psi_2(y)-\nabla\psi_1(y)) e^{-g(y)}\dd y.$$
	\end{definition}
	We used the term linearized because without the inverse Hessian term, $LD_{W}$ is exactly the popular linearized optimal transport distance in the literature; see \cite{cai2020linearized,moosmuller2023linear,wang2013linear}. On the other hand, we call it a divergence because if  $\nabla\psi_2$ and $\nabla \psi_1$ are ``close", then $LD_W(\nabla\psi_2;\nabla\psi_1)$ is the second-order approximation for the following expected Bregman divergence 
	$\E_{Y\sim e^{-g}} [\psi_2^*(\nabla\psi_2(Y))-\psi_2^*(\nabla\psi_1(Y))-\langle \nabla\psi_2(Y)-\nabla\psi_1(Y),\nabla\psi_2^*(\nabla\psi_1(Y))\rangle]$ (see \cref{sec:Bregman} for more details on Bregman divergences). 
	As a final step to defining the variational problem for \eqref{eq:discretedual}, we reparametrize the KL divergence as a functional on $\mathrm{Tan}_{e^{-g}}$, i.e., define 
	$\tilde{F}(\nabla \psi):=KL((\nabla \psi)\#e^{-g}|e^{-f}).$
	
	\begin{prop}\label{prop:varrep}
		Suppose $\nabla \psi_1\in \mathrm{Tan}_{e^{-g}}$ is strictly convex and twice differentiable. Also, assume that both $f$ and $g$ are twice differentiable and the map 
		$y\mapsto \sum_{j=1}^d \int \frac{\partial^2}{\partial y_i\partial y_j}\psi^*(\nabla\psi(y))\frac{\partial}{\partial y_j}\Theta(y)e^{-g(y)}\dd y$
		vanishes as $|y_i|\to\infty$ for all $i=1,2,\ldots ,d$, and all $\Theta\in\mathcal{C}_c^{\infty}$. 
		Let 
		\begin{align}\label{eq:targetoptim}
			\nabla\tilde{\psi}\in \argmin_{\nabla\psi\in L^2(e^{-g})} \left[\tilde{F}(\nabla \psi_1)+\Bigg\langle \frac{\delta}{\delta (\nabla \psi)}\tilde{F}(\nabla \psi)\big|_{\nabla \psi=\nabla \psi_1}, \nabla \psi -\nabla \psi_1\Bigg\rangle_{e^{-g}}+\frac{1}{2\tau}LD_W(\nabla \psi;\nabla \psi_1)\right].
		\end{align}
		Then any optimizer $\nabla \tilde{\psi}$ of the above variational problem satisfies the following stationary condition:
		$$\nabla \tilde{\psi}(y) - \nabla \psi_1(y) = -\tau\nabla \left(f(\nabla \psi_1(y))-g(y)-\log\det(\nabla^2\psi_1(y))\right)$$
	\end{prop}
	We defer the proof of \cref{prop:varrep} to \cref{sec:pfappl}. Note that the above variational problem is reminiscent of preconditioned gradient descent \cite{bonet2024mirror,li2017preconditioned} and Newton iterations \cite{battiti1992first,cartis2010complexity}. However, to the best of our knowledge, the above variational problem is new in the literature, which makes our parabolic PDE-based approach different from contemporary mirror descent schemes on the $2$-Wasserstein space, e.g., \cite{aubin2022mirror,han2025variational,karimi2024sinkhorn,leger2021gradient}.

	\noindent \textbf{Bayesian Variational Inference.} Our discretized parabolic PDE \eqref{eq:discretedual} and the variational interpretation in \cref{prop:varrep} can be leveraged to provide an alternate perspective on variational inference. A common computational challenge in Bayesian Statistics is to compute integrals with respect to a complicated high-dimensional probability distribution, say our target $e^{-f}$. Variational inference \cite{blei2017variational} has emerged as a computationally viable approximation. At the core of variational inference is to find 
	$$\hat{\rho}_{\mathcal{V}}:=\argmin_{\rho\in \mathcal{V}} KL(\rho|e^{-f}),$$
	where $\mathcal{V}$ denotes a candidate variational family of distributions. These families are chosen so that integrals against measures in $\mathcal{V}$ are easy to compute. 
	Popular examples of variational families include:
	
	\begin{itemize}
		\item[(i)] Gaussian variational inference \cite{diao2023forward,Katsevich2024} - Here $\mathcal{V}$ is the family of $d$-dimensional Gaussian distributions, i.e., $\mathcal{V}:=\{N(m,\Sigma):\, m\in\R^d,\ \Sigma\in \mathbb{S}_+^{d\times d} \mbox{ is symmetric positive definite}\}.$
		
		\item[(ii)] Mean-field variational inference \cite{jiang2024algorithms,wainwright2008graphical} - Here $\mathcal{V}$ is the family of $d$-dimensional product measures, i.e., $\mathcal{V}=\{\pi_1\otimes \ldots \otimes \pi_d:\, \pi_i \mbox{ is a probability measure on }\R,\, 1\le i\le d\}$.
	\end{itemize}
	
	An alternate way of looking at the above variational inference problem is to formulate it in terms of transport maps from a simple reference distribution $e^{-g}$, say $N(0,\lambda^2 Id)$ for some $\lambda>0$. So, instead of parameterizing with a family of distributions, we now parameterize with a family of transport maps. We define 
	\begin{equation}\label{eq:vicall}
		\nabla{\hat\psi}_{\boldsymbol{\mathcal{S}}}:=\argmin_{\nabla\psi\in\mathcal{S}} KL(\nabla\psi\#e^{-g} |e^{-f}),
	\end{equation}
	where $\boldsymbol{\mathcal{S}}\subseteq L^2(e^{-g})$ is now a \emph{variational family of functions}. Some natural examples would include: 
	
	\begin{itemize}
		\item[(i)] Gaussian variational inference - Here $$\boldsymbol{\mathcal{S}}:=\{Ay+m: \, m\in\R^d,\, A\in \mathbb{S}_+^{d\times d} \mbox{ is symmetric positive definite}\}.$$ 
		\item[(ii)] Mean-field variational inference - Here $$\boldsymbol{\mathcal{S}}:=\{(S_1(y_1),\ldots ,S_d(y_d)):\, S_i:\R\to\R \mbox{ is monotone increasing} \}.$$
	\end{itemize}
	The above formulation in \eqref{eq:vicall} should immediately remind us of the variational form introduced in \cref{prop:varrep} as $KL(\nabla\psi\#e^{-g} |e^{-f})$ is exactly $\tilde{F}(\nabla\psi)$ from \cref{prop:varrep}. Therefore, we can iteratively solve \eqref{eq:targetoptim} over a potentially nonparametric class of vector fields $\boldsymbol{\mathcal{S}}\subseteq L^2(e^{-g})$, i.e., 
	\begin{align}\label{eq:targetoptimvi}
		\nabla\psi_{k+1}\in \argmin_{\nabla\psi\in \boldsymbol{\mathcal{S}}} \left[\tilde{F}(\nabla \psi_k)+\Bigg\langle \frac{\delta}{\delta (\nabla \psi)}\tilde{F}(\nabla \psi_k), \nabla \psi -\nabla \psi_k\Bigg\rangle_{e^{-g}}+\frac{1}{2\tau}LD_W(\nabla \psi;\nabla \psi_k)\right].
	\end{align}
	Here the first variation $\frac{\delta}{\delta (\nabla \psi)}\tilde{F}(\nabla \psi_k)$ is the variation restricted to the class $\boldsymbol{\mathcal{S}}$, i.e., for any $\nabla \Theta\in\boldsymbol{\mathcal{S}}$, $\langle \frac{\delta}{\delta (\nabla \psi)}\tilde{F}(\nabla \psi_k),\nabla\Theta\rangle_{e^{-g}}=\lim_{\epsilon\to 0} \frac{\tilde{F}(\nabla\psi_k+\epsilon \nabla\Theta)-\tilde{F}(\nabla\psi_k)}{\epsilon}$. For parametric classes of $\boldsymbol{\mathcal{S}}$ such as in Gaussian variational inference, one could directly solve \eqref{eq:targetoptimvi} over the natural parameters $\mu$ and $A$. In the univariate setting, the new Gaussian variational inference updates are summarized in \cref{alg:VI}.
	
	\begin{center}
		\begin{minipage}{0.9\textwidth}
			\begin{algorithm}[H]
				\caption{Gaussian Variational Inference}
				\label{alg:VI}
				\SetKwInOut{Input}{Input}
				\SetKwInOut{Output}{Output}
				\Input{$T$, total number of steps; $\{\eta_k\}_{k \leq T}$, the step-sizes; $e^{-g}$, density of $N(0,\lambda^2)$, $\lambda>0$, initializer $(m_0,\sigma_0)$}
				\Output{Final mean and standard deviation $(m_T,\sigma_T)$.}
				\While{$k < T$}{
					\textbf{Update variational parameters}: \[
					( m_{k+1} , \sigma_{k+1} ) \gets G_{\eta_k}(m_k,\sigma_k) \qquad \mbox{where} 
					\]
					\[
					G_{\eta}(t,s) :=\begin{pmatrix}t-\frac{\eta s}{\lambda}\E_{Y\sim N(0,\lambda^2)} f'((s/\lambda)Y+t)) \\  s-\frac{\eta}{\lambda^2}(s^2 \E_{Y\sim N(0,\lambda^2)} f''((s/\lambda)Y+t)-1). \end{pmatrix}
					\]
					Set $k \gets k+1$
				}
			\end{algorithm}
		\end{minipage}
	\end{center}
	To heuristically see the validity of \cref{alg:VI}, consider the case $\eta_k\equiv \eta$ in which case the updates can be written as the fixed point iteration $(m_{k+1},\sigma_{k+1})=G_{\eta}(m_k,\sigma_k)$ where $G_{\eta}(\cdot,\cdot)$ is defined in \cref{alg:VI}. Therefore, the limit $(m_k,\sigma_k)$ as $k\to\infty$, say $(m,\sigma)$ if it exists, would satisfy $(m,\sigma)=G_{\eta}(m,\sigma)$. This simplifies to 
	$$\E_{X\sim N(m,\sigma^2)} f'(X)=0, \quad \mbox{and} \quad \E_{X\sim N(m,\sigma^2)} f''(X)=\sigma^{-2},$$
	which are exactly the stationary conditions for $(m,\sigma^2)$ arising as a solution of the Gaussian variational inference problem; see \cite[Equation 1.9]{Katsevich2024}. In \cref{sec:numvar}, we provide a numerical experiment that illustrates the performance of \cref{alg:VI}.

	\section{Numerical Experiments}\label{sec:simulations}
	
	In this Section, we provide a numerical example to demonstrate the efficacy of our proposed algorithms toward sampling from a non-log-concave target. To wit, we will use a Gaussian location mixture as our target distribution 
	\begin{equation}\label{eq:targetmix}
		e^{-f(x)}=\frac{1}{2}\phi(x-2)+\frac{1}{2}\phi(x+2),
	\end{equation}
	where $\phi(\cdot)$ is the standard normal density. We initialize $\psi_0(y)=y^2/2$ and choose the reference $e^{-g(y)}$ to be the standard normal density. Note that the optimal transport from the standard normal to the mixture \eqref{eq:targetmix} is given by 
	\begin{equation}\label{eq:actualOT}
		\psi_{\star}'(y)=F^{-1}\circ \Phi(y),
	\end{equation}
	where $\Phi(\cdot)$ is the standard normal cumulative distribution function and $F(x)=(1/2)(\Phi(x-2)+\Phi(x+2))$. 
	The discrete updates from \eqref{eq:discretedual} can now be written as 
	\begin{align}\label{eq:iterate_mix}
		\psi_{k+1}(y) &=\psi_k(y)+\eta_k \Delta_k(y) \;,\\
		\Delta_k(y) &:=\log\left(\frac{1}{2}\phi(\psi_k'(y)-2)+\frac{1}{2}\phi(\psi_k'(y)+2)\right)+\frac{1}{2}y^2+\log(\psi_k''(y)) \;, \nonumber
	\end{align}
	under an appropriately chosen sequence of step-sizes which will be described below. We will describe our experiments in four steps that highlight the practical challenges in discretization and our proposed strategies for addressing them. 
	
	\begin{figure}
		\centering
		\includegraphics[width=0.49\linewidth]{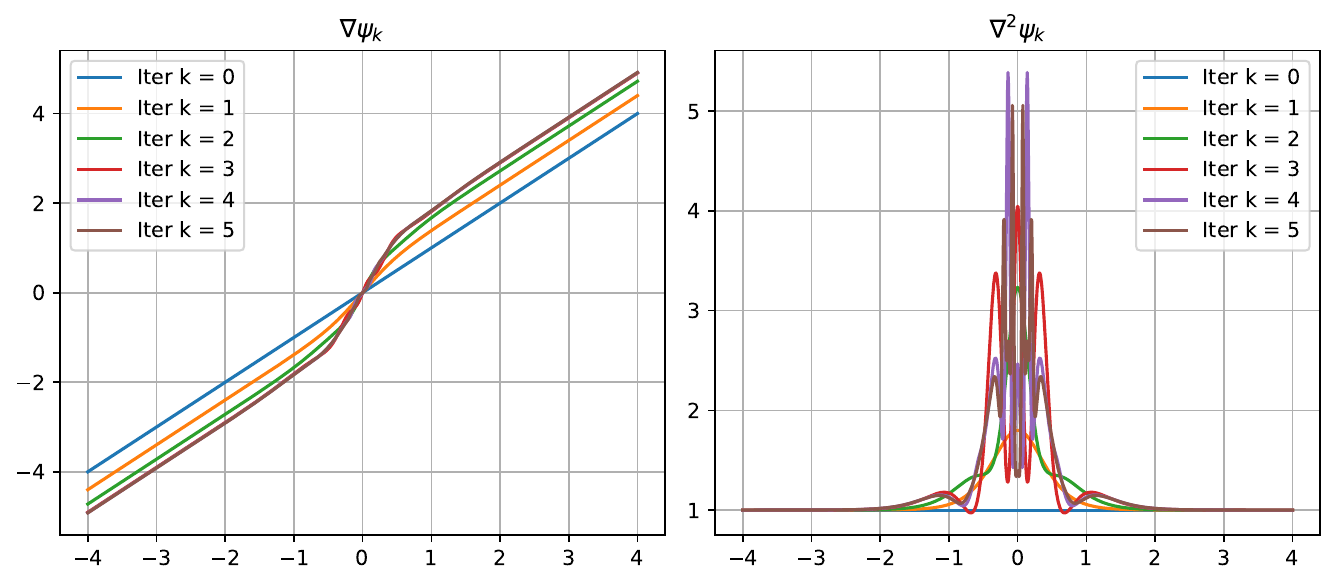}
		\includegraphics[width=0.49\linewidth]{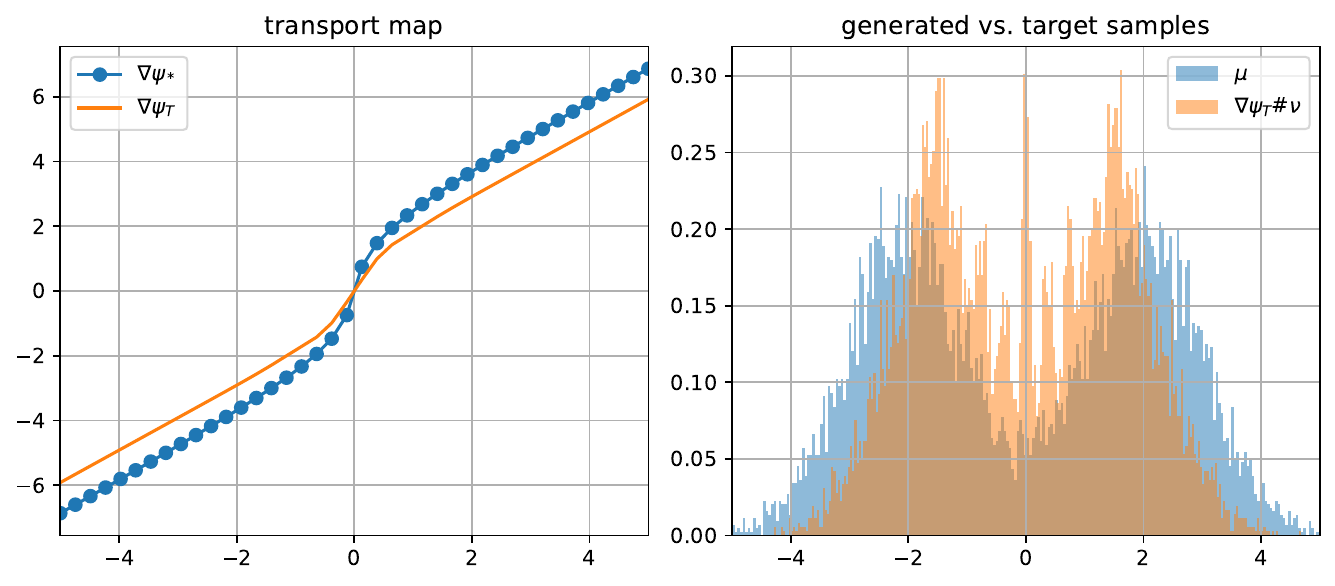}
		\caption{Oracle updates --- The first and second plots track the evolution of the gradient and the Hessian, respectively. The third plot compares the learned transport map at the final iterate with the target optimal transport map from $N(0,1)$ to the Gaussian mixture \eqref{eq:targetmix}. The fourth plot compares the histograms between $10000$ samples from the mixture in \eqref{eq:targetmix} and $10000$ samples from the push-forward measure $\nabla\psi_T\#N(0,1)$.}
		\label{fig:oracle}
	\end{figure}
	
	\begin{figure}
		\centering
		\includegraphics[width=0.49\linewidth]{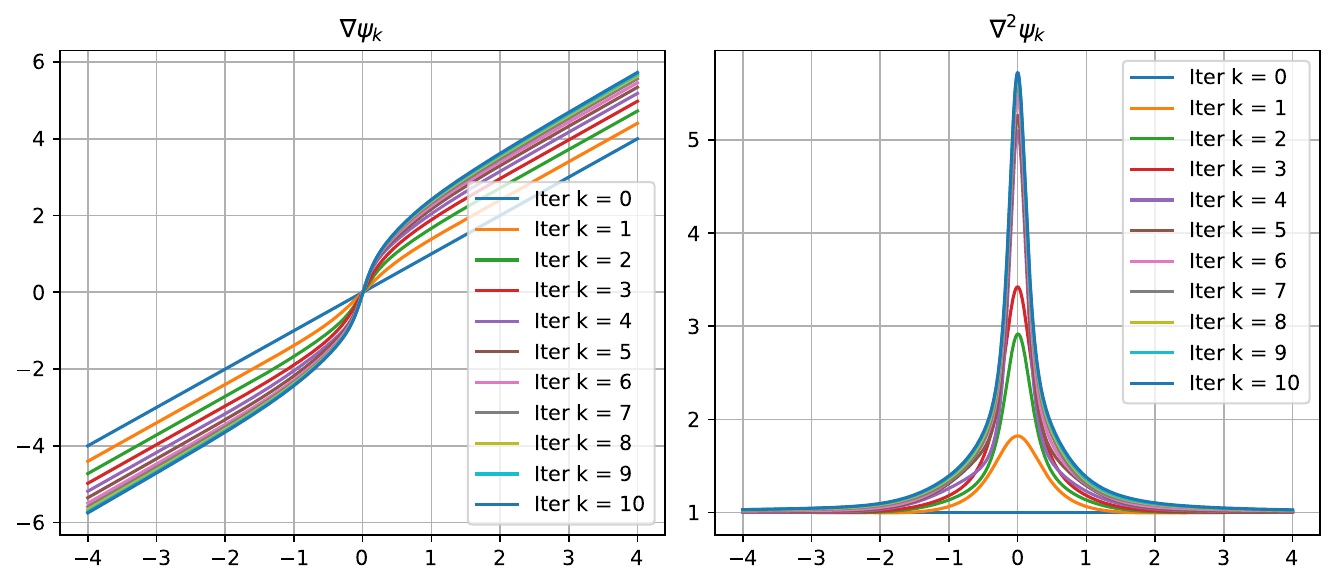}
		\includegraphics[width=0.49\linewidth]{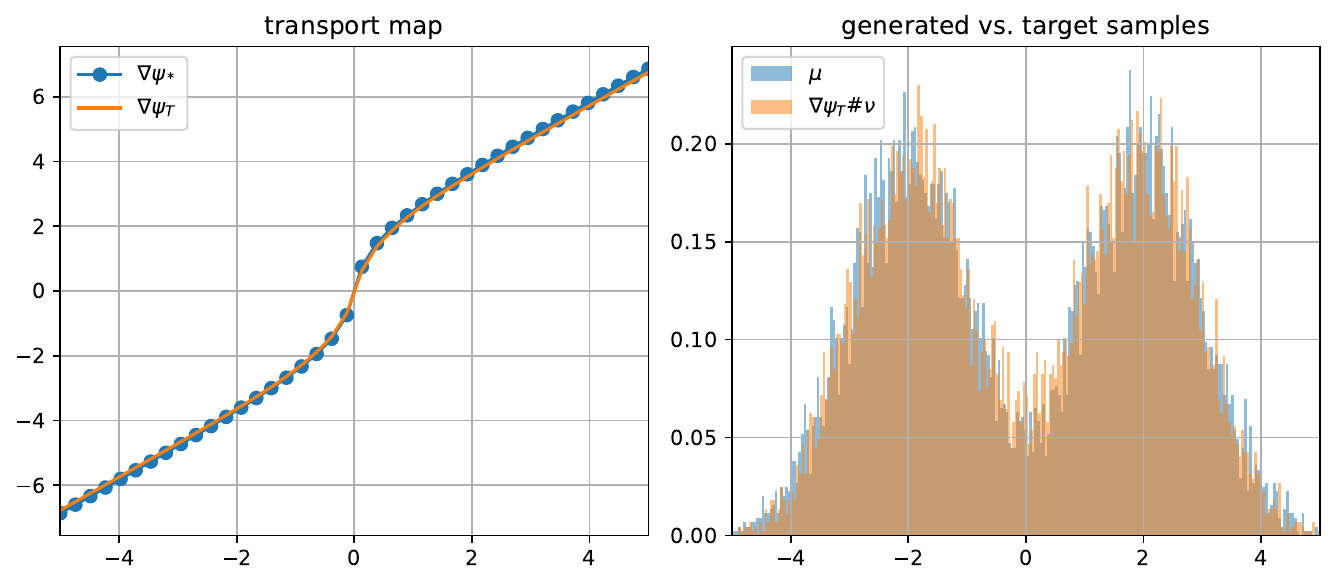}
		\caption{Same as \cref{fig:oracle} for oracle updates with distillation.}
		\label{fig:oracle_distill}
	\end{figure}

	\paragraph*{Oracle computation of \eqref{eq:discretedual} with batch auto-differentiation} In the first approach, we assume access to the functional form of the target $e^{-f}$ and we update the potentials analytically from \eqref{eq:iterate_mix}. All gradients and higher-order derivatives were computed using the automatic 
	differentiation module available in \texttt{PyTorch} (through the \texttt{torch.autograd} function). The step-sizes are chosen adaptively as follows: 
	\begin{itemize}
		\item We partition the range $[-3,3]$ into $1000$ equally spaced parts. Suppose $\{a_1,a_2,\ldots ,a_{1000}\}$ denotes the boundary points of the partition.
		\item Suppose that we have constructed up to $\psi_k$, we then choose 
		\begin{equation}\label{eq:adaptstep}
			\eta_k=\frac{1}{2}\max\left\{\max\left\{-\frac{\psi_k''(a_i)}{\Delta_k''(a_i)}: \Delta_k''(a_i)<0, 1\le i\le 1000\right\},0.4\right\}.
		\end{equation}
	\end{itemize}
	The above choice ensures that the double derivative of $\psi_{k+1}$ is strictly positive at the grid points. The performance of the oracle procedure is illustrated in \cref{fig:oracle}. There are two major challenges which arise in the above implementation:
	
	(a) The higher-order derivatives are numerically unstable. This instability accumulates in the hessians of the potentials as can be seen in the second panel of \cref{fig:oracle}. 
	
	(b) Computing the gradients and hessians via backpropagation requires storing the full computational graph which is memory intensive. In our experiments, we found it hard to go beyond $T = 5$ or $6$ iterations. While the learned transport map after $5$ iterations captures a bimodal structure, the third and fourth panels of \cref{fig:oracle} show that it is not reasonably close to $\psi_{\star}'$ (see \eqref{eq:actualOT}).
	
	\paragraph*{Oracle computation of  \eqref{eq:discretedual} with distillation} To address the numerical instabilities and memory allocation issues with the oracle computation, we use a \emph{distillation} step at every iteration. The distillation decouples the computational graph and regularizes the instabilities. It proceeds as follows: with the $k$-th iteration potential $\psi_k$, we calculate the update direction $\Delta_k(y)$ according to \eqref{eq:iterate_mix}. Then we train a simple \emph{student} network by minimizing
	\begin{align}\label{eq:distill}
		\Delta_{k}^{\mbox{stud}}\in \argmin_{ \mathrm{nn}(\cdot ; \theta) } \frac{1}{M}\sum_{j=1}^M (\Delta_k'(Z_j)-\mathrm{nn}'(Z_j ; \theta))^2,
	\end{align}
	where the above minimization is over $\mathrm{nn}(\cdot ; \theta)$, a three-layer neural network with two hidden layers each of width $32$ with weights $\theta$ and $Z_1,\ldots ,Z_M\overset{i.i.d.}{\sim} \mathrm{Unif}[-3,3]$, with $M=500$. The above distillation step initializes the student network from a cold start and uses the \texttt{Adam} optimizer for training. To guarantee smoothness, we choose the softplus activation function ($x\mapsto \log(1+e^x)$) for the student network. We then update $\psi_{k+1} = \psi_{k} + \eta_k \Delta_{k}^{\mbox{stud}}$ with adaptive step-sizes from \eqref{eq:adaptstep}. This process is repeated recursively. It ensures that the gradients and Hessians of $\psi_k$ are stable. It also avoids the need to store a large computational graph (growing exponentially with number of iterations), hence overcoming the memory allocation issues. This allows us to run more iterations $T=10$ with much less computational time. The performance of this procedure is presented in \cref{fig:oracle_distill}. In contrast to \cref{fig:oracle}, the gradients and Hessians in \cref{fig:oracle_distill} are much more stable. Within $10$ iterations, we observe very close agreement between $\psi'_k$ and $\psi'_{\star}$ in the third and fourth panels of \cref{fig:oracle_distill}.

	\paragraph*{\cref{alg:neural-pde-score-matching} with distillation} Both of the above approaches rely on the knowledge of the functional form of the target $e^{-f}$. In practice, we only have access to samples, say $X_1,\ldots ,X_n\overset{i.i.d.}{\sim} e^{-f}$. In such a scenario, let us first discuss the performance of \cref{alg:neural-pde-score-matching} with $n=10000$. In \eqref{eqn:neural-pde-score-matching}, we use a two hidden-layer neural network of width $32$, and the softplus activation function, to optimize the score-matching objective. We use the same adaptive step-sizes as in \eqref{eq:adaptstep} and distillation approach similar to \eqref{eq:distill}. \cref{fig:scorematching} shows that the iterates are much more stable than the analytic oracle updates. After $10$ iterations, the learned transport map $\psi'_k$ is close to the optimal map $\psi'_{\star}$ in the interval $[-3,3]$, beyond which the approximation deteriorates slightly (particularly for $|y|>3$). A potential reason for this is that the score matching approach tries to learn the score for the target \eqref{eq:targetmix} and the current iterate distribution $\rho_k=(\psi_k')\#N(0,1)$ separately, without leveraging the fact that the two scores are likely to be closer as the number of iterates increases. This limitation is overcome by our next approach.
	
	\begin{figure}
		\centering
		\includegraphics[width=0.49\linewidth]{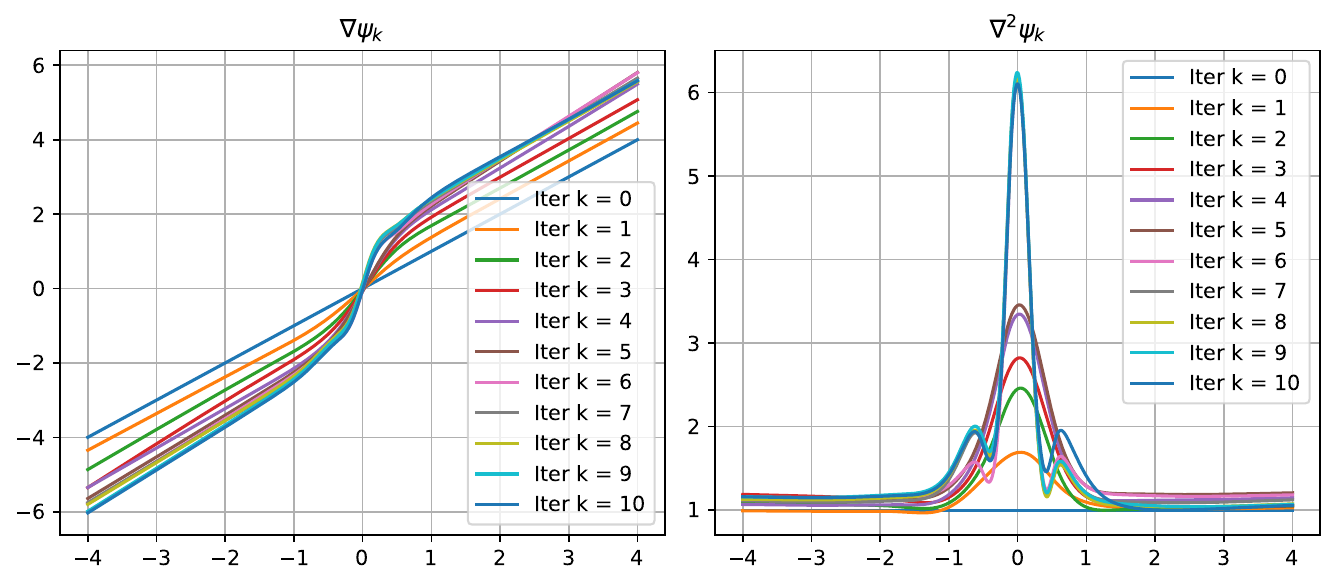}
		\includegraphics[width=0.49\linewidth]{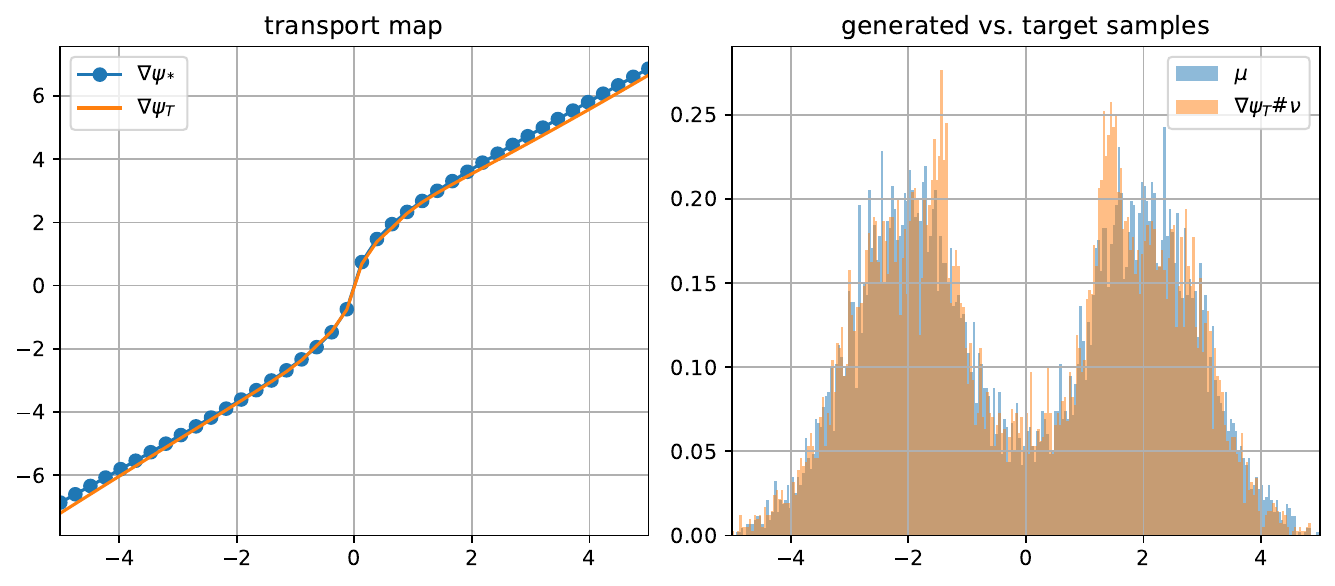}
		\caption{Same as \cref{fig:oracle} for score matching based \cref{alg:neural-pde-score-matching} with distillation.}
		\label{fig:scorematching}
	\end{figure}
	
	\paragraph*{\cref{alg:neural-pde-logistic-reg} with distillation} Let us now discuss the performance of \cref{alg:neural-pde-logistic-reg}. In \eqref{eqn:neural-pde-logistic-reg}, we again use a two-hidden-layer neural network of width $32$, and the softplus activation function, to optimize the logistic loss. We use the same adaptive step-sizes as in \eqref{eq:adaptstep} and distillation approach as in \eqref{eq:distill}. This time the distillation learns an estimate of the log density-ratio via a simple student network, say $\Delta_k^{\mbox{stud}}$.  \cref{fig:logisticloss} shows that the iterates (their gradients as well as hessians) are stable. After $10$ iterations, the learned transport map $\psi_k'$ nicely approximates the true optimal map $\psi_{\star}'$, better than in the previous score matching based approach, particularly outside $[-3,3]$. We believe this is because \cref{alg:neural-pde-logistic-reg} learns the logarithm of the density ratio between the current distribution $\rho_k=(\psi_k')\#N(0,1)$ and the target \eqref{eq:targetmix} directly (unlike in \cref{alg:neural-pde-score-matching} where the score functions are learned separately). Since this log density-ratio gets closer to $0$ as the number of iterations grows, it can be approximated accurately using a simple neural network.
	
	\begin{figure}
		\centering
		\includegraphics[width=0.49\linewidth]{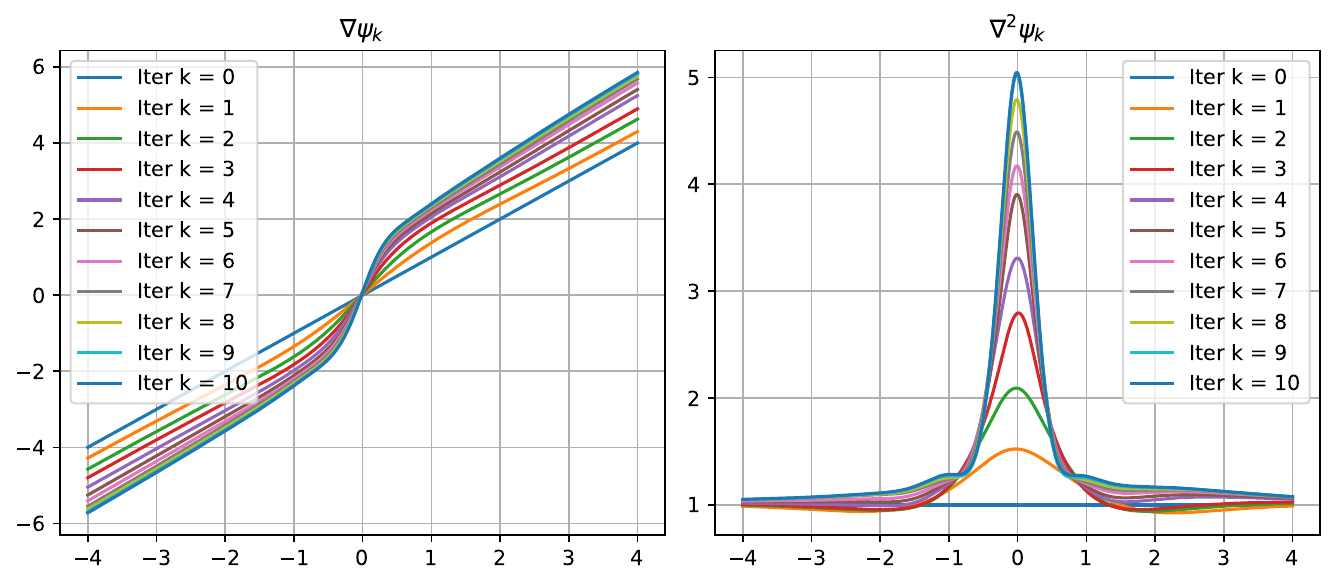}
		\includegraphics[width=0.49\linewidth]{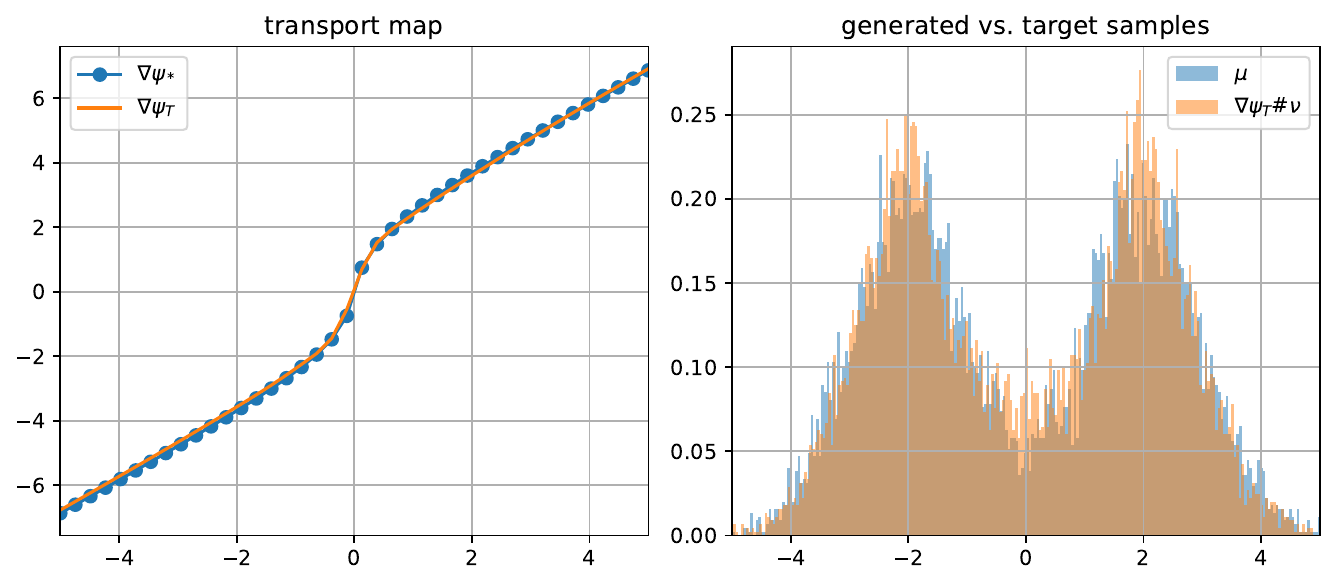}
		\caption{Same as \cref{fig:oracle} for logistic regression based \cref{alg:neural-pde-logistic-reg} with distillation.}
		\label{fig:logisticloss}
	\end{figure}
	
	\begin{remark}[Effect of distillation]
		Let us consider our approach based on logistic regression. Note that the distillation step attempts to approximate the log density-ratio $f+\log\rho_k$ with a simpler student network, say $\Delta_k^{\mbox{stud}}$. This approximation error can be transferred easily to our current regret bounds. In particular, if 
		$$\int (f+\log\rho_k-\Delta_k^{\mbox{stud}})(x)(\rho_k-e^{-f})(x)\, d x = O(\eta_k),$$
		where $\eta_k$ denotes the step-size, then all our regret bounds and last iterate convergence bounds continue to hold after adjusting for constants. In \cref{sec:numvar}, we discuss another method to avoid long backpropagation chains that do not require distillation.
	\end{remark}
	
	\begin{remark}[On termination criteria]
		A practical question is how to determine when to terminate the proposed algorithms. One natural strategy is to assess whether the push forward samples $\nabla\psi_k(Y_i)$, where $Y_i \sim \mathcal{N}(0,I)$ (for instance), are statistically indistinguishable from samples drawn from the target distribution. This can be done by performing a two-sample hypothesis test (for example, an energy \cite{szekely2013energy} or kernel-based test \cite{gretton2012kernel}) at a prescribed significance level~$\alpha$. Alternatively, one may compute a discrepancy measure such as the maximum mean discrepancy (see \cite[Equation 3]{gretton2012kernel}) between the generated and target samples, and monitor its decay across iterations. The iteration can then be stopped once the maximum mean discrepancy falls below a user-defined tolerance, providing a simple and data-driven convergence criterion. 
	\end{remark}

	\section*{Discussion and Future work}
	In this paper, we have introduced a new generative modeling framework which draws a novel connection between a parabolic PDE (\eqref{eq:dualpma}) and probabilistic sampling. We have proposed a time discretization for \eqref{eq:dualpma} and established different regimes of convergence rates to the target distribution, depending on smoothness assumptions and appropriate step-size choices. Our theoretical results include bounds on average iterate (see \cref{thm:dualcon}), average regret (see Theorems \ref{thm:regbd1} and \ref{thm:regbd2}) and last iterate (see \cref{thm:ptcon}). In addition we show how the discretization integrates nicely with neural network based learning techniques from generative networks and score-based diffusion models (see \cref{Sec:genmod}). Finally we illustrate how the proposal naturally yields a new family of variational inference algorithms. Overall, our proposal combines the ease of one-pass sampling typically associated with GANs (generative adversarial networks) alongside the ease of distribution learning associated with DPMs (diffusion-based probabilistic models). 
	
	While our paper is an early theoretical step towards bridging the gap between parabolic PDE discretization and sampling, a number of interesting theoretical and practical questions remain open. A few of them are discussed below.
	
	\begin{enumerate}
		\item \textbf{Sample complexity bounds}: Given data $X_1,\ldots ,X_n\overset{i.i.d.}{\sim} e^{-f}$, suppose the initial distribution $\rho_0$ is chosen to be the empirical measure $n^{-1}\sum_{i=1}^n \delta_{X_i}$. A natural question would be to establish a sample complexity bound, i.e., given an $\epsilon>0$, one must find $n\equiv n(\epsilon,e^{-f})$, $T\equiv T(\epsilon,e^{-f})$, and appropriate step-sizes $\{\eta_k\}_{k=0}^T$ such that $\mbox{KL}(\rho_T,e^{-f})\le \epsilon$. 
		
		\item \textbf{Structural estimation in high dimension}: Empirically it is often observed that structural assumptions on $e^{-f}$ can lead to faster convergence of generative models, particularly when the dimension $d$ is large. It would be theoretically challenging and practically useful to establish such results for our proposed parabolic Monge-Amp\`{e}re discretization. Two natural structures one can impose are: (a) Low intrinsic dimension --- Here the target distribution is supported on some low-dimensional manifold. While this means the target does not admit a Lebesgue density, we note that \cref{alg:neural-pde-logistic-reg} from \cref{Sec:genmod} does not require the existence of a density, and (b) Markov random field: Here the target has a density $e^{-f}$ but the (conditional) dependence structure is governed by an underlying sparse graph. A canonical example would be a standard Markov chain. 
		
		\item \textbf{Optimal transport map estimation}: There has been extensive research in recent years towards constructing computationally tractable estimators of optimal transport maps. While the Sinkhorn algorithm shows tremendous potential, establishing statistical convergence rates for a Sinkhorn-based estimator, under smoothness/structural conditions, remains open. Since the parabolic PDE targets the optimal transport map from $e^{-g}$ to $e^{-f}$, our proposed discretizations yield an estimator (say $\nabla \hat{\psi}_T$) of this optimal transport map. It would be interesting to explore whether one can establish sample complexity bounds under the popular metric $\lVert \nabla \hat{\psi}_T - \nabla \psi_{\infty}\rVert_{L^2(e^{-g})}$, and show that it offers fast convergence under structural/smoothness assumptions on $e^{-f}$.
		
		\item \textbf{Empirical validation on high-resolution image data}: Diffusion based generative models are popularly used to generate high-resolution images. It would be interesting to explore empirically how well our proposed algorithms perform in such image generation experiments. There are interesting practical challenges that arise while working with images, such as (a) potential violation of smoothness assumptions required for Theorems \ref{thm:regbd1}, \ref{thm:regbd2}, and \ref{thm:ptcon}, (b) designing structured neural networks that exploit potential low-dimensional nature of images, and (c) tuning the step-size sequence adaptively.
	\end{enumerate}
	
	\section*{Acknowledgment}
	Liang acknowledges the generous support from the NSF Career Award (DMS-2042473) and the Wallman Society of Fellows at the University of Chicago. The authors thank Takuya Koriyama for implementing a distillation module, which we use in the numerical experiment section.
	
	\bibliographystyle{abbrvnat}
	\bibliography{References}

\begin{thebibliography}{83}
\providecommand{\natexlab}[1]{#1}
\providecommand{\url}[1]{\texttt{#1}}
\expandafter\ifx\csname urlstyle\endcsname\relax
  \providecommand{\doi}[1]{doi: #1}\else
  \providecommand{\doi}{doi: \begingroup \urlstyle{rm}\Url}\fi

\bibitem[Abedin and Kitagawa(2020)]{abedin2020exponential}
F.~Abedin and J.~Kitagawa.
\newblock Exponential convergence of parabolic optimal transport on bounded
  domains.
\newblock \emph{Analysis \& PDE}, 13\penalty0 (7):\penalty0 2183--2204, 2020.

\bibitem[Ambrosio et~al.(2008)Ambrosio, Gigli, and Savar\'e]{Ambrosio2008}
L.~Ambrosio, N.~Gigli, and G.~Savar\'e.
\newblock \emph{Gradient flows in metric spaces and in the space of probability
  measures}.
\newblock Lectures in Mathematics ETH Z\"urich. Birkh\"auser Verlag, Basel,
  second edition, 2008.
\newblock ISBN 978-3-7643-8721-1.

\bibitem[Amp{\`e}re(1819)]{ampere1819memoire}
A.-M. Amp{\`e}re.
\newblock \emph{M{\'e}moire contenant l'application de la th{\'e}orie
  expos{\'e}e dans le XVII. e Cahier du Journal de l'{\'E}cole polytechnique,
  {\`a} l'int{\'e}gration des {\'e}quations aux diff{\'e}rentielles partielles
  du premier et du second ordre}.
\newblock De l'Imprimerie royale, 1819.

\bibitem[Arjovsky et~al.(2017)Arjovsky, Chintala, and
  Bottou]{arjovsky2017wasserstein}
M.~Arjovsky, S.~Chintala, and L.~Bottou.
\newblock Wasserstein generative adversarial networks.
\newblock In \emph{International conference on machine learning}, pages
  214--223. PMLR, 2017.

\bibitem[Aubin-Frankowski et~al.(2022)Aubin-Frankowski, Korba, and
  L{\'e}ger]{aubin2022mirror}
P.-C. Aubin-Frankowski, A.~Korba, and F.~L{\'e}ger.
\newblock Mirror descent with relative smoothness in measure spaces, with
  application to sinkhorn and em.
\newblock \emph{Advances in Neural Information Processing Systems},
  35:\penalty0 17263--17275, 2022.

\bibitem[Battiti(1992)]{battiti1992first}
R.~Battiti.
\newblock First-and second-order methods for learning: between steepest descent
  and newton's method.
\newblock \emph{Neural computation}, 4\penalty0 (2):\penalty0 141--166, 1992.

\bibitem[Beck and Teboulle(2003)]{Beck2003}
A.~Beck and M.~Teboulle.
\newblock Mirror descent and nonlinear projected subgradient methods for convex
  optimization.
\newblock \emph{Oper. Res. Lett.}, 31\penalty0 (3):\penalty0 167--175, 2003.
\newblock ISSN 0167-6377,1872-7468.

\bibitem[Berman(2020)]{berman2020sinkhorn}
R.~J. Berman.
\newblock The sinkhorn algorithm, parabolic optimal transport and geometric
  monge--amp{\`e}re equations.
\newblock \emph{Numerische Mathematik}, 145\penalty0 (4):\penalty0 771--836,
  2020.

\bibitem[Blei et~al.(2017)Blei, Kucukelbir, and McAuliffe]{blei2017variational}
D.~M. Blei, A.~Kucukelbir, and J.~D. McAuliffe.
\newblock Variational inference: A review for statisticians.
\newblock \emph{Journal of the American statistical Association}, 112\penalty0
  (518):\penalty0 859--877, 2017.

\bibitem[Bonet et~al.(2024)Bonet, Uscidda, David, Aubin-Frankowski, and
  Korba]{bonet2024mirror}
C.~Bonet, T.~Uscidda, A.~David, P.-C. Aubin-Frankowski, and A.~Korba.
\newblock Mirror and preconditioned gradient descent in wasserstein space.
\newblock \emph{arXiv preprint arXiv:2406.08938}, 2024.

\bibitem[Boyd and Vandenberghe(2004)]{boyd2004convex}
S.~P. Boyd and L.~Vandenberghe.
\newblock \emph{Convex optimization}.
\newblock Cambridge university press, 2004.

\bibitem[Brenier(1991)]{brenier1991polar}
Y.~Brenier.
\newblock Polar factorization and monotone rearrangement of vector-valued
  functions.
\newblock \emph{Communications on pure and applied mathematics}, 44\penalty0
  (4):\penalty0 375--417, 1991.

\bibitem[Bubeck et~al.(2021)Bubeck, Cohen, Lee, and Lee]{Bubeck2021}
S.~Bubeck, M.~B. Cohen, J.~R. Lee, and Y.~T. Lee.
\newblock Metrical task systems on trees via mirror descent and unfair gluing.
\newblock \emph{SIAM J. Comput.}, 50\penalty0 (3):\penalty0 909--923, 2021.
\newblock ISSN 0097-5397,1095-7111.

\bibitem[Budd and Williams(2009)]{budd2009moving}
C.~J. Budd and J.~Williams.
\newblock Moving mesh generation using the parabolic monge--amp{\`e}re
  equation.
\newblock \emph{SIAM Journal on Scientific Computing}, 31\penalty0
  (5):\penalty0 3438--3465, 2009.

\bibitem[Budd et~al.(2009)Budd, Huang, and Russell]{budd2009adaptivity}
C.~J. Budd, W.~Huang, and R.~D. Russell.
\newblock Adaptivity with moving grids.
\newblock \emph{Acta Numerica}, 18:\penalty0 111--241, 2009.

\bibitem[Cai et~al.(2020)Cai, Cheng, Craig, and Craig]{cai2020linearized}
T.~Cai, J.~Cheng, N.~Craig, and K.~Craig.
\newblock Linearized optimal transport for collider events.
\newblock \emph{Physical Review D}, 102\penalty0 (11):\penalty0 116019, 2020.

\bibitem[Cao et~al.(2021)Cao, Bie, Vahdat, Fidler, and Kreis]{cao2021don}
T.~Cao, A.~Bie, A.~Vahdat, S.~Fidler, and K.~Kreis.
\newblock Don’t generate me: Training differentially private generative
  models with sinkhorn divergence.
\newblock \emph{Advances in Neural Information Processing Systems},
  34:\penalty0 12480--12492, 2021.

\bibitem[Cartis et~al.(2010)Cartis, Gould, and Toint]{cartis2010complexity}
C.~Cartis, N.~I. Gould, and P.~L. Toint.
\newblock On the complexity of steepest descent, newton's and regularized
  newton's methods for nonconvex unconstrained optimization problems.
\newblock \emph{Siam journal on optimization}, 20\penalty0 (6):\penalty0
  2833--2852, 2010.

\bibitem[Cesa-Bianchi and Lugosi(2006)]{Cesa2006}
N.~Cesa-Bianchi and G.~Lugosi.
\newblock \emph{Prediction, learning, and games}.
\newblock Cambridge University Press, Cambridge, 2006.
\newblock ISBN 978-0-521-84108-5; 0-521-84108-9.

\bibitem[Chen and Teboulle(1993)]{chen1993convergence}
G.~Chen and M.~Teboulle.
\newblock Convergence analysis of a proximal-like minimization algorithm using
  bregman functions.
\newblock \emph{SIAM Journal on Optimization}, 3\penalty0 (3):\penalty0
  538--543, 1993.

\bibitem[Chen et~al.(2022)Chen, Chewi, Li, Li, Salim, and
  Zhang]{chen2022sampling}
S.~Chen, S.~Chewi, J.~Li, Y.~Li, A.~Salim, and A.~R. Zhang.
\newblock Sampling is as easy as learning the score: theory for diffusion
  models with minimal data assumptions.
\newblock \emph{arXiv preprint arXiv:2209.11215}, 2022.

\bibitem[Chen et~al.(2024)Chen, Chewi, Lee, Li, Lu, and
  Salim]{chen2024probability}
S.~Chen, S.~Chewi, H.~Lee, Y.~Li, J.~Lu, and A.~Salim.
\newblock The probability flow ode is provably fast.
\newblock \emph{Advances in Neural Information Processing Systems}, 36, 2024.

\bibitem[Chewi(2023)]{chewi2023log}
S.~Chewi.
\newblock Log-concave sampling.
\newblock \emph{Book draft}, 2023.

\bibitem[Cuturi(2013)]{cuturi2013sinkhorn}
M.~Cuturi.
\newblock Sinkhorn distances: Lightspeed computation of optimal transport.
\newblock \emph{Advances in neural information processing systems}, 26, 2013.

\bibitem[De~Bortoli et~al.(2021)De~Bortoli, Thornton, Heng, and
  Doucet]{de2021diffusion}
V.~De~Bortoli, J.~Thornton, J.~Heng, and A.~Doucet.
\newblock Diffusion schr{\"o}dinger bridge with applications to score-based
  generative modeling.
\newblock \emph{Advances in Neural Information Processing Systems},
  34:\penalty0 17695--17709, 2021.

\bibitem[Deb et~al.(2023)Deb, Kim, Pal, and Schiebinger]{deb2023wasserstein}
N.~Deb, Y.-H. Kim, S.~Pal, and G.~Schiebinger.
\newblock Wasserstein mirror gradient flow as the limit of the sinkhorn
  algorithm.
\newblock \emph{arXiv preprint arXiv:2307.16421}, 2023.

\bibitem[Diao et~al.(2023)Diao, Balasubramanian, Chewi, and
  Salim]{diao2023forward}
M.~Z. Diao, K.~Balasubramanian, S.~Chewi, and A.~Salim.
\newblock Forward-backward gaussian variational inference via jko in the
  bures-wasserstein space.
\newblock In \emph{International Conference on Machine Learning}, pages
  7960--7991. PMLR, 2023.

\bibitem[Dieuleveut et~al.(2020)Dieuleveut, Durmus, and Bach]{Dieleveut2020}
A.~Dieuleveut, A.~Durmus, and F.~Bach.
\newblock Bridging the gap between constant step size stochastic gradient
  descent and {M}arkov chains.
\newblock \emph{Ann. Statist.}, 48\penalty0 (3):\penalty0 1348--1382, 2020.
\newblock ISSN 0090-5364,2168-8966.

\bibitem[Fokker(1914)]{fokker1914mittlere}
A.~D. Fokker.
\newblock Die mittlere energie rotierender elektrischer dipole im
  strahlungsfeld.
\newblock \emph{Annalen der Physik}, 348\penalty0 (5):\penalty0 810--820, 1914.

\bibitem[Fortet(1940)]{fortet1940resolution}
R.~Fortet.
\newblock R{\'e}solution d'un syst{\`e}me d'{\'e}quations de m.
  schr{\"o}dinger.
\newblock \emph{Journal de math{\'e}matiques pures et appliqu{\'e}es},
  19\penalty0 (1-4):\penalty0 83--105, 1940.

\bibitem[Franklin and Lorenz(1989)]{franklin1989scaling}
J.~Franklin and J.~Lorenz.
\newblock On the scaling of multidimensional matrices.
\newblock \emph{Linear Algebra and its applications}, 114:\penalty0 717--735,
  1989.

\bibitem[Frigyik et~al.(2008)Frigyik, Srivastava, and
  Gupta]{frigyik2008functional}
B.~A. Frigyik, S.~Srivastava, and M.~R. Gupta.
\newblock Functional bregman divergence and bayesian estimation of
  distributions.
\newblock \emph{IEEE Transactions on Information Theory}, 54\penalty0
  (11):\penalty0 5130--5139, 2008.

\bibitem[Genevay et~al.(2018)Genevay, Peyr{\'e}, and
  Cuturi]{genevay2018learning}
A.~Genevay, G.~Peyr{\'e}, and M.~Cuturi.
\newblock Learning generative models with sinkhorn divergences.
\newblock In \emph{International Conference on Artificial Intelligence and
  Statistics}, pages 1608--1617. PMLR, 2018.

\bibitem[Goodfellow et~al.(2020)Goodfellow, Pouget-Abadie, Mirza, Xu,
  Warde-Farley, Ozair, Courville, and Bengio]{goodfellow2020generative}
I.~Goodfellow, J.~Pouget-Abadie, M.~Mirza, B.~Xu, D.~Warde-Farley, S.~Ozair,
  A.~Courville, and Y.~Bengio.
\newblock Generative adversarial networks.
\newblock \emph{Communications of the ACM}, 63\penalty0 (11):\penalty0
  139--144, 2020.

\bibitem[Goodfellow et~al.(2014)Goodfellow, Pouget-Abadie, Mirza, Xu,
  Warde-Farley, Ozair, Courville, and Bengio]{goodfellow2014generative}
I.~J. Goodfellow, J.~Pouget-Abadie, M.~Mirza, B.~Xu, D.~Warde-Farley, S.~Ozair,
  A.~Courville, and Y.~Bengio.
\newblock Generative adversarial nets.
\newblock \emph{Advances in neural information processing systems}, 27, 2014.

\bibitem[Gordon(1999)]{Gordon1999}
G.~J. Gordon.
\newblock Regret bounds for prediction problems.
\newblock In \emph{Proceedings of the {T}welfth {A}nnual {C}onference on
  {C}omputational {L}earning {T}heory ({S}anta {C}ruz, {CA}, 1999)}, pages
  29--40. ACM, New York, 1999.
\newblock ISBN 1-58113-167-4.

\bibitem[Gretton et~al.(2012)Gretton, Borgwardt, Rasch, Sch{\"o}lkopf, and
  Smola]{gretton2012kernel}
A.~Gretton, K.~M. Borgwardt, M.~J. Rasch, B.~Sch{\"o}lkopf, and A.~Smola.
\newblock A kernel two-sample test.
\newblock \emph{The journal of machine learning research}, 13\penalty0
  (1):\penalty0 723--773, 2012.

\bibitem[Guo et~al.(2022)Guo, Hur, Liang, and Ryan]{guo2022online}
W.~Guo, Y.~Hur, T.~Liang, and C.~Ryan.
\newblock Online learning to transport via the minimal selection principle.
\newblock In \emph{Conference on Learning Theory}, pages 4085--4109. PMLR,
  2022.

\bibitem[Han et~al.(2025)Han, Kim, YOO, and Zhang]{han2025variational}
D.-S. Han, J.~Kim, H.~B. YOO, and B.-T. Zhang.
\newblock Variational mirror descent for robust learning in schr\"odinger
  bridge.
\newblock 2025.

\bibitem[He et~al.(2016)He, Zhang, Ren, and Sun]{he2016deep}
K.~He, X.~Zhang, S.~Ren, and J.~Sun.
\newblock Deep residual learning for image recognition.
\newblock In \emph{Proceedings of the IEEE conference on computer vision and
  pattern recognition}, pages 770--778, 2016.

\bibitem[Hur et~al.(2024)Hur, Guo, and Liang]{hur2024reversible}
Y.~Hur, W.~Guo, and T.~Liang.
\newblock Reversible gromov--monge sampler for simulation-based inference.
\newblock \emph{SIAM Journal on Mathematics of Data Science}, 6\penalty0
  (2):\penalty0 283--310, 2024.

\bibitem[Hyv\"arinen(2005)]{Hyvarinen2005}
A.~Hyv\"arinen.
\newblock Estimation of non-normalized statistical models by score matching.
\newblock \emph{J. Mach. Learn. Res.}, 6:\penalty0 695--709, 2005.
\newblock ISSN 1532-4435,1533-7928.

\bibitem[Jiang et~al.(2024)Jiang, Chewi, and Pooladian]{jiang2024algorithms}
Y.~Jiang, S.~Chewi, and A.-A. Pooladian.
\newblock Algorithms for mean-field variational inference via polyhedral
  optimization in the wasserstein space.
\newblock In \emph{The Thirty Seventh Annual Conference on Learning Theory},
  pages 2720--2721. PMLR, 2024.

\bibitem[Karimi et~al.(2024)Karimi, Hsieh, and Krause]{karimi2024sinkhorn}
M.~R. Karimi, Y.-P. Hsieh, and A.~Krause.
\newblock Sinkhorn flow as mirror flow: A continuous-time framework for
  generalizing the sinkhorn algorithm.
\newblock In \emph{International Conference on Artificial Intelligence and
  Statistics}, pages 4186--4194. PMLR, 2024.

\bibitem[Katsevich and Rigollet(2024)]{Katsevich2024}
A.~Katsevich and P.~Rigollet.
\newblock On the approximation accuracy of {G}aussian variational inference.
\newblock \emph{Ann. Statist.}, 52\penalty0 (4):\penalty0 1384--1409, 2024.
\newblock ISSN 0090-5364,2168-8966.

\bibitem[Kim et~al.(2012)Kim, Streets, and Warren]{kim2012parabolic}
Y.-H. Kim, J.~Streets, and M.~Warren.
\newblock Parabolic optimal transport equations on manifolds.
\newblock \emph{International Mathematics Research Notices}, 2012\penalty0
  (19):\penalty0 4325--4350, 2012.

\bibitem[Knott and Smith(1984)]{knott1984optimal}
M.~Knott and C.~S. Smith.
\newblock On the optimal mapping of distributions.
\newblock \emph{Journal of Optimization Theory and Applications}, 43:\penalty0
  39--49, 1984.

\bibitem[Lan et~al.(2011)Lan, Lu, and Monteiro]{lan2011primal}
G.~Lan, Z.~Lu, and R.~D. Monteiro.
\newblock Primal-dual first-order methods with iteration-complexity for cone
  programming.
\newblock \emph{Mathematical Programming}, 126\penalty0 (1):\penalty0 1--29,
  2011.

\bibitem[L{\'e}ger(2021)]{leger2021gradient}
F.~L{\'e}ger.
\newblock A gradient descent perspective on sinkhorn.
\newblock \emph{Applied Mathematics \& Optimization}, 84\penalty0 (2):\penalty0
  1843--1855, 2021.

\bibitem[Li(2021)]{li2021transport}
W.~Li.
\newblock Transport information bregman divergences.
\newblock \emph{Information Geometry}, 4\penalty0 (2):\penalty0 435--470, 2021.

\bibitem[Li(2017)]{li2017preconditioned}
X.-L. Li.
\newblock Preconditioned stochastic gradient descent.
\newblock \emph{IEEE transactions on neural networks and learning systems},
  29\penalty0 (5):\penalty0 1454--1466, 2017.

\bibitem[Liang(2021)]{liang2021well}
T.~Liang.
\newblock How well generative adversarial networks learn distributions.
\newblock \emph{Journal of Machine Learning Research}, 22\penalty0
  (228):\penalty0 1--41, 2021.

\bibitem[Liang et~al.(2026)Liang, Dharmakeerthi, and
  Koriyama]{liang2024denoising}
T.~Liang, K.~Dharmakeerthi, and T.~Koriyama.
\newblock Denoising diffusions with optimal transport: Localization, curvature,
  and multi-scale complexity.
\newblock \emph{Transactions on Machine Learning Research}, 2026.
\newblock ISSN 2835-8856.

\bibitem[Liero et~al.(2023)Liero, Mielke, and Savar{\'e}]{liero2023fine}
M.~Liero, A.~Mielke, and G.~Savar{\'e}.
\newblock Fine properties of geodesics and geodesic $\lambda$-convexity for the
  hellinger--kantorovich distance.
\newblock \emph{Archive for Rational Mechanics and Analysis}, 247\penalty0
  (6):\penalty0 112, 2023.

\bibitem[Lim et~al.(2023)Lim, Yoon, Byun, Kang, Kim, Lee, and
  Choi]{song2020score}
S.~Lim, E.~Yoon, T.~Byun, T.~Kang, S.~Kim, K.~Lee, and S.~Choi.
\newblock Score-based generative modeling through stochastic evolution
  equations in hilbert spaces.
\newblock 2023.

\bibitem[Lu et~al.(2018)Lu, Freund, and Nesterov]{lu2018relatively}
H.~Lu, R.~M. Freund, and Y.~Nesterov.
\newblock Relatively smooth convex optimization by first-order methods, and
  applications.
\newblock \emph{SIAM Journal on Optimization}, 28\penalty0 (1):\penalty0
  333--354, 2018.

\bibitem[L{\"u}beck et~al.(2022)L{\"u}beck, Bunne, Gut, del Castillo, Pelkmans,
  and Alvarez-Melis]{lubeck2022neural}
F.~L{\"u}beck, C.~Bunne, G.~Gut, J.~S. del Castillo, L.~Pelkmans, and
  D.~Alvarez-Melis.
\newblock Neural unbalanced optimal transport via cycle-consistent
  semi-couplings.
\newblock \emph{arXiv preprint arXiv:2209.15621}, 2022.

\bibitem[Ma et~al.(2021)Ma, Chatterji, Cheng, Flammarion, Bartlett, and
  Jordan]{Ma2021}
Y.-A. Ma, N.~S. Chatterji, X.~Cheng, N.~Flammarion, P.~L. Bartlett, and M.~I.
  Jordan.
\newblock {Is there an analog of Nesterov acceleration for gradient-based
  MCMC?}
\newblock \emph{Bernoulli}, 27\penalty0 (3):\penalty0 1942 -- 1992, 2021.

\bibitem[McCann(1995)]{McCann1995}
R.~J. McCann.
\newblock Existence and uniqueness of monotone measure-preserving maps.
\newblock \emph{Duke Math. J.}, 80\penalty0 (2):\penalty0 309--323, 1995.
\newblock ISSN 0012-7094,1547-7398.

\bibitem[M{\'e}rigot(2016)]{merigotdiscretization}
Q.~M{\'e}rigot.
\newblock Discretization of euler’s equations using optimal transport: Cauchy
  and boundary value problems.
\newblock \emph{S{\'e}minaire Laurent Schwartz—EDP et applications}, pages
  1--12, 2016.

\bibitem[Monge(1784)]{monge1784memoire}
G.~Monge.
\newblock \emph{M{\'e}moire sur le calcul int{\'e}gral des {\'e}quations aux
  diff{\'e}rences partielles}.
\newblock Imprimerie royale, 1784.

\bibitem[Moosm{\"u}ller and Cloninger(2023)]{moosmuller2023linear}
C.~Moosm{\"u}ller and A.~Cloninger.
\newblock Linear optimal transport embedding: provable wasserstein
  classification for certain rigid transformations and perturbations.
\newblock \emph{Information and Inference: A Journal of the IMA}, 12\penalty0
  (1):\penalty0 363--389, 2023.

\bibitem[Planck(1917)]{planck1917satz}
V.~Planck.
\newblock {\"U}ber einen satz der statistischen dynamik und seine erweiterung
  in der quantentheorie.
\newblock \emph{Sitzungberichte der}, 1917.

\bibitem[R{\"u}schendorf(1995)]{ruschendorf1995convergence}
L.~R{\"u}schendorf.
\newblock Convergence of the iterative proportional fitting procedure.
\newblock \emph{The Annals of Statistics}, pages 1160--1174, 1995.

\bibitem[R{\"u}schendorf and Thomsen(1993)]{ruschendorf1993note}
L.~R{\"u}schendorf and W.~Thomsen.
\newblock Note on the schr{\"o}dinger equation and i-projections.
\newblock \emph{Statistics \& probability letters}, 17\penalty0 (5):\penalty0
  369--375, 1993.

\bibitem[Salim et~al.(2020)Salim, Korba, and Luise]{salim2020wasserstein}
A.~Salim, A.~Korba, and G.~Luise.
\newblock The wasserstein proximal gradient algorithm.
\newblock \emph{Advances in Neural Information Processing Systems},
  33:\penalty0 12356--12366, 2020.

\bibitem[Santambrogio(2015)]{Santambrogio2015}
F.~Santambrogio.
\newblock \emph{Optimal transport for applied mathematicians}, volume~87 of
  \emph{Progress in Nonlinear Differential Equations and their Applications}.
\newblock Birkh\"auser/Springer, Cham, 2015.
\newblock ISBN 978-3-319-20827-5; 978-3-319-20828-2.
\newblock Calculus of variations, PDEs, and modeling.

\bibitem[Santambrogio(2017)]{santambrogio2017euclidean}
F.~Santambrogio.
\newblock Euclidean, metric, and {W}asserstein gradient flows: an overview.
\newblock \emph{Bulletin of Mathematical Sciences}, 7:\penalty0 87--154, 2017.

\bibitem[Schr{\"o}dinger(1935)]{schrodinger1935present}
E.~Schr{\"o}dinger.
\newblock The present status of quantum mechanics.
\newblock \emph{Die Naturwissenschaften}, 23\penalty0 (48):\penalty0 1--26,
  1935.

\bibitem[Shalev-Shwartz(2012)]{shalev2012online}
S.~Shalev-Shwartz.
\newblock Online learning and online convex optimization.
\newblock \emph{Foundations and Trends{\textregistered} in Machine Learning},
  4\penalty0 (2):\penalty0 107--194, 2012.

\bibitem[Shun and McCullagh(1995)]{shun1995laplace}
Z.~Shun and P.~McCullagh.
\newblock Laplace approximation of high dimensional integrals.
\newblock \emph{Journal of the Royal Statistical Society Series B: Statistical
  Methodology}, 57\penalty0 (4):\penalty0 749--760, 1995.

\bibitem[Song and Ermon(2019)]{song2019generative}
Y.~Song and S.~Ermon.
\newblock Generative modeling by estimating gradients of the data distribution.
\newblock \emph{Advances in Neural Information Processing Systems}, 32, 2019.

\bibitem[Song et~al.(2021)Song, Durkan, Murray, and Ermon]{song2021maximum}
Y.~Song, C.~Durkan, I.~Murray, and S.~Ermon.
\newblock Maximum likelihood training of score-based diffusion models.
\newblock \emph{Advances in Neural Information Processing Systems},
  34:\penalty0 1415--1428, 2021.

\bibitem[Sulman et~al.(2011{\natexlab{a}})Sulman, Williams, and
  Russell]{sulman2011optimal}
M.~Sulman, J.~Williams, and R.~D. Russell.
\newblock Optimal mass transport for higher dimensional adaptive grid
  generation.
\newblock \emph{Journal of computational physics}, 230\penalty0 (9):\penalty0
  3302--3330, 2011{\natexlab{a}}.

\bibitem[Sulman et~al.(2021)Sulman, Nguyen, Haynes, and
  Huang]{sulman2021domain}
M.~H. Sulman, T.~B. Nguyen, R.~D. Haynes, and W.~Huang.
\newblock Domain decomposition parabolic monge--amp{\`e}re approach for fast
  generation of adaptive moving meshes.
\newblock \emph{Computers \& Mathematics with Applications}, 84:\penalty0
  97--111, 2021.

\bibitem[Sulman et~al.(2011{\natexlab{b}})Sulman, Williams, and
  Russell]{sulman2011efficient}
M.~M. Sulman, J.~Williams, and R.~D. Russell.
\newblock An efficient approach for the numerical solution of the
  monge--amp{\`e}re equation.
\newblock \emph{Applied Numerical Mathematics}, 61\penalty0 (3):\penalty0
  298--307, 2011{\natexlab{b}}.

\bibitem[Sz{\'e}kely and Rizzo(2013)]{szekely2013energy}
G.~J. Sz{\'e}kely and M.~L. Rizzo.
\newblock Energy statistics: A class of statistics based on distances.
\newblock \emph{Journal of statistical planning and inference}, 143\penalty0
  (8):\penalty0 1249--1272, 2013.

\bibitem[Tzen et~al.(2023)Tzen, Raj, Raginsky, and Bach]{Tzen2023}
B.~Tzen, A.~Raj, M.~Raginsky, and F.~Bach.
\newblock Variational principles for mirror descent and mirror {L}angevin
  dynamics.
\newblock \emph{IEEE Control Syst. Lett.}, 7:\penalty0 1542--1547, 2023.
\newblock ISSN 2475-1456.

\bibitem[Villani(2009)]{villani2009optimal}
C.~Villani.
\newblock \emph{Optimal transport: old and new}, volume 338.
\newblock Springer, 2009.

\bibitem[Wainwright and Jordan(2008)]{wainwright2008graphical}
M.~J. Wainwright and M.~I. Jordan.
\newblock Graphical models, exponential families, and variational inference.
\newblock \emph{Foundations and Trends{\textregistered} in Machine Learning},
  1\penalty0 (1--2):\penalty0 1--305, 2008.

\bibitem[Wang et~al.(2021)Wang, Jiao, Xu, Wang, and Yang]{wang2021deep}
G.~Wang, Y.~Jiao, Q.~Xu, Y.~Wang, and C.~Yang.
\newblock Deep generative learning via schr{\"o}dinger bridge.
\newblock In \emph{International conference on machine learning}, pages
  10794--10804. PMLR, 2021.

\bibitem[Wang et~al.(2013)Wang, Slep{\v{c}}ev, Basu, Ozolek, and
  Rohde]{wang2013linear}
W.~Wang, D.~Slep{\v{c}}ev, S.~Basu, J.~A. Ozolek, and G.~K. Rohde.
\newblock A linear optimal transportation framework for quantifying and
  visualizing variations in sets of images.
\newblock \emph{International journal of computer vision}, 101:\penalty0
  254--269, 2013.

\bibitem[Zhang et~al.(2023)Zhang, Chewi, Li, Balasubramanian, and
  Erdogdu]{zhang2023improved}
S.~Zhang, S.~Chewi, M.~Li, K.~Balasubramanian, and M.~A. Erdogdu.
\newblock Improved discretization analysis for underdamped langevin monte
  carlo.
\newblock In \emph{The Thirty Sixth Annual Conference on Learning Theory},
  pages 36--71. PMLR, 2023.

\end{thebibliography}

	\section*{Appendix} 
	\begin{appendix}
		The Appendix is organized as follows: 
		\begin{itemize}
			\item \cref{sec:pfpmainmotapp} proves the results from \cref{sec:pmainmotapp}.
			\item \cref{sec:Bregp} proves the results from \cref{sec:Bregman}.
			\item \cref{sec:pfmainres} proves the results from \cref{sec:regcon}.
			\item \cref{sec:pfappl} proves the results from \cref{sec:applications}.
			\item \cref{sec:numvar} contains additional discussions on implementation including (a) a numerical experiment that illustrates the performance of the Gaussian variational inference procedure, namely \cref{alg:VI}, and (b) a distillation-free computationally tractable approach that avoids long backpropagation chains. 
		\end{itemize}
		
		\section{Proofs for Section~\ref{sec:pmainmotapp}}\label{sec:pfpmainmotapp}
		
		\noindent First we present the rate of convergence of the discretization \eqref{eq:discretedual}, for the univariate Gaussian case from \cref{sec:uniGill}. The main purpose here is to investigate the impact of the parameter $\lambda$ in the rate of convergence ``locally". 
		
		\begin{prop}\label{prop:uniGauss}
			Suppose that $e^{-f}$ and $e^{-g}$ are densities of centered univariate Gaussians, say $N(0,1)$ and $N(0,\lambda^2)$. Choose constant step size $\eta_k\equiv \eta$ for all $k\ge 0$ and some $\eta\in (0,\lambda)$. Let $\upsilon \in (|1-2\eta/\lambda|,1)$ and $c_0\in (\lambda^{-1}-\delta,\lambda^{-1}+\delta)$ with $\delta:=\min\{2\eta-\lambda(1-\upsilon),\lambda(1+\upsilon)-2\eta\}/2\eta\lambda>0$. Suppose \eqref{eq:discretedual} is initialized with $\psi_0(y)=c_0y^2/2$. Then the following conclusions hold:
			\begin{enumerate}
				\item For each $k\ge 1$, $\psi_k'(y)=c_k y$ where each $c_k\in ((1-\upsilon)/2\eta,(1+\upsilon)/2\eta)$. 
				\item $|c_k-\lambda^{-1}|\le \upsilon^k |c_0-\lambda^{-1}|$.
				\item Recall that $\rho_k=(\psi_k')\#e^{-g}$ (see \eqref{eq:measupdate}), then each $\rho_k$ is the density of $N(0,\sigma_k^2)$ where $|\sigma_k-1|\le \upsilon^k |\sigma_0-1|$.
			\end{enumerate}
		\end{prop}
		
		\begin{proof}[Proof of \cref{prop:uniGauss}]
			We first show the existence of $\upsilon,\delta$ satisfying the conditions of the proposition. As $\eta\in (0,\lambda)$, we have $|1-2\eta/\lambda|<1$ and consequently, there exists $\upsilon\in (|1-2\eta/\lambda|,1)$. This also implies that both $2\eta-\lambda(1-\upsilon)>0$ and $\lambda(1+\upsilon)-2\eta>0$, which in turn  implies $\delta=\min\{2\eta-\lambda(1-\upsilon),\lambda(1+\upsilon)-2\eta\}/2\eta\lambda>0$. 
			
			The rest of the proof proceeds by induction. Suppose the conclusions in \cref{prop:uniGauss}, parts 1,2, and 3 hold up to some $k$. As the conclusions hold for $k=0$, it suffices to show that the conclusions hold at $k+1$. By \eqref{eq:discretedual}, we have: 
			$$\psi_{k+1}'(y)=c_k y-\eta\left(c_k^2y-\frac{y}{\lambda^2}\right)=\left(c_k-\eta\left(c_k^2-\frac{1}{\lambda^2}\right)\right)y.$$
			This shows $\psi'_{k+1}(\cdot)$ is also linear. Writing $\psi_{k+1}'(y)=c_{k+1}y$, we then get the evolution equation $c_{k+1}=c_k-\eta(c_k^2-\lambda^{-2})$. This equation can alternatively be expressed as $c_{k+1}=\theta(c_k)$ where the function $\theta:\R\to\R$ is given by $\theta(z)=z-\eta(z^2-\lambda^{-2})$. We will use the following two properties of $\theta$ multiple times in the rest of the proof: 
			\begin{equation}\label{eq:thetaprop}
				\theta(\lambda^{-1})=\lambda^{-1} \quad \mbox{and} \quad |\theta'(z)|\le \upsilon \,\, \mbox{for all } z\in ((1-\upsilon)/2\eta,(1+\upsilon)/2\eta).
			\end{equation}
			Therefore by \eqref{eq:thetaprop}, we have: 
			$|c_{k+1}-\lambda^{-1}|=|\theta(c_k)-\theta(\lambda^{-1})|=|\theta'(\vartheta_k)||c_k-\lambda^{-1}|$ 
			where $\vartheta_k$ lies on the line joining $c_k$ and $\lambda^{-1}$. By induction hypothesis $c_k\in ((1-\upsilon)/2\eta,(1+\upsilon)/2\eta)$. It is easy to check that $\upsilon\in (|1-2\eta/\lambda|,1)$ implies $\lambda^{-1}\in ((1-\upsilon)/2\eta,(1+\upsilon)/2\eta)$. Therefore $\vartheta_k\in ((1-\upsilon)/2\eta,(1+\upsilon)/2\eta)$ and consequently by \eqref{eq:thetaprop}, $|\theta'(\vartheta_k)|\le \upsilon$. As a result, 
			$$|c_{k+1}-\lambda^{-1}|\le \upsilon |c_k-\lambda^{-1}|\le \upsilon\times \upsilon^{k}|c_0-\lambda^{-1}|=\upsilon^{k+1}|c_0-\lambda^{-1}|.$$
			In the above display, we have used the induction hypothesis $|c_k-\lambda^{-1}|\le \upsilon^k |c_0-\lambda^{-1}|$. This establishes part 2 of \cref{prop:uniGauss}. Moreover, the above display also yields $|c_{k+1}-\lambda^{-1}|\le |c_0-\lambda^{-1}|\le \delta$ where $\delta>0$ is defined in the proposition as $\delta=\min\{2\eta-\lambda(1-\upsilon),\lambda(1+\upsilon)-2\eta\}/2\eta\lambda$. So $c_{k+1}\in (\lambda^{-1}-\delta,\lambda^{-1}+\delta)\subseteq ((1-\upsilon)/2\eta,(1+\upsilon)/2\eta)$. This establishes part 1 of the proposition. Finally, as $\psi_k'(\cdot)$s are all linear and $\psi_k'(0)=0$ for all $k$, we immediately get that $\rho_k$ is the density of $N(0,\sigma_k^2)$ where $\sigma_k=c_k\lambda$. Therefore, 
			$$|\sigma_{k+1}-1|=|(c_{k+1}-\lambda^{-1})\lambda|\le \upsilon |\sigma_k-1| \le \upsilon\times \upsilon^k |\sigma_0-1|=\upsilon^{k+1}|\sigma_0-1|.$$
			This establishes part 3 of the proposition.
		\end{proof}
		
		\noindent \cref{prop:uniGauss} proves a locally exponential rate of convergence to the target distribution $N(0,1)$. In other words, once the algorithm reaches a sufficiently ``local ball" around the target, the convergence is exponentially fast. To fix ideas, let us fix the step-size $\eta$. Then by choosing $\lambda \approx 2\eta$, we have $1-2\eta/\lambda \approx 0$. This means we can choose $\upsilon\approx 0$ resulting in fast convergence. However the trade-off is that $\upsilon\approx 0$ and $\lambda\approx 2\eta$ implies $\delta\approx 0$, i.e., the local convergence radius needs to be very close to the target in order to witness such fast rates. 
		
		Next we will discuss the proof of \cref{prop:Sinkapprox}. We begin with an elementary observation for strictly convex functions from $\R^d\to\R$. We refer the reader to \cite[Lemma 2.3]{berman2020sinkhorn} for a proof.
		\begin{prop}\label{obs:elemprop}
			Suppose $\psi:\R^d\to\R$ is a strictly convex function and $\psi^*$ is its Fenchel dual. Then $\nabla\psi^*(\nabla\psi(y))=y$, $\nabla^2\psi(y)=(\nabla^2 \psi^*(\nabla\psi(y))^{-1}$, and $\psi^*(\nabla\psi(y))+\psi(y)=\langle y,\nabla\psi(y)\rangle$, for all $y\in\R^d$.
		\end{prop}
		
		\begin{proof}[Proof of \cref{prop:Sinkapprox}]
			Let $\phi=\psi^*$ denote the Fenchel dual of $\psi$. Observe the following identity:
			\begin{align}\label{eq:pmacondis1}
				\exp(\frac{\langle x, y \rangle - \psi(y)}{\epsilon}) = \exp( - \frac{D_{\psi}(y | \nabla \phi(x))}{\epsilon}) \exp( \frac{\phi(x)}{\epsilon} )
			\end{align}
			As a result, we get: 
			\begin{align*}
				\exp(\frac{\tilde{\psi}^{\epsilon}(y) - \psi(y)}{\epsilon}) \nonumber
				&= \int \frac{\exp(\frac{\langle x, y \rangle - \psi(y)}{\epsilon})}{ \int \exp(\frac{\langle x, y' \rangle - \psi(y')}{\epsilon}) \exp( - g(y'))  \dd y'} \exp( - f(x)) \dd x \nonumber \\
				&=\int \frac{\exp( - \frac{D_{\psi}(y | \nabla \phi(x))}{\epsilon}-f(x))}{ \int \exp( - \frac{D_{\psi}(y' | \nabla \phi(x))}{\epsilon}) \exp( - g(y'))  \dd y'}  \dd x \nonumber \\ &=\int \frac{\exp( - \frac{D_{\phi}(x | \nabla \psi(y))}{\epsilon}-f(x))}{ \int \exp( - \frac{D_{\psi}(y' | \nabla \phi(x))}{\epsilon}) \exp( - g(y'))  \dd y'}\,\dd x, 
			\end{align*}
			where the second to last equality uses \eqref{eq:pmacondis1}, whereas the last equality uses
			fact that $D_{\psi}(y|\nabla \phi(x))=D_{\phi}(x|\nabla\psi(y))$.  
			Next, by using the Laplace approximation (see \cite{shun1995laplace}), we note that 
			\begin{align*}
				\lim_{\epsilon\to 0}(2\pi\epsilon)^{-d/2}\int\exp( - \frac{D_{\psi}(y' | \nabla \phi(x))}{\epsilon}) \exp(-g(y')) \,\dd y' = \exp(-g(\nabla \phi(x)))  | \det ( \nabla^2 \psi (\nabla \phi(x) )) |^{-1/2}
			\end{align*}
			uniformly in $x$. By combining the above displays with another application of the Laplace approximation, we have
	
			\begin{align*}
				\int \frac{\exp( - \frac{D_{\phi} ( x | \nabla \psi(y) ) }{\epsilon}) \exp( - f(x)) }{ \int \exp( - \frac{D_{\psi}(y' | \nabla \phi(x))}{\epsilon}) \exp( - g(y'))\dd y'}\dd x &
				 = \left(\int \frac{\exp( - \frac{D_{\phi} ( x | \nabla \psi(y) )}{\epsilon}) \exp( - f(x)) }{ \exp(-g(\nabla \phi(x))) \cdot (2\pi \epsilon)^{d/2} | \det ( \nabla^2 \psi (\nabla \phi(x) )) |^{-1/2} }\dd x\right)(1 + o(1))
				\\
				& = \exp\left(-f(\nabla \psi(y)) + g(y) -  \frac{1}{2}\log\det\left((\nabla^2 \psi(y))^{-1} \nabla^2\phi( \nabla\psi(y ))\right)\right)(1+ o(1))
			\end{align*}		
			as $\epsilon\to 0$, where the last line uses the fact $\nabla \psi = ( \nabla \phi)^{-1}$. The conclusion now follows from \cref{obs:elemprop}.
		\end{proof}
		
		\section{Proofs for Section~\ref{sec:Bregman}}\label{sec:Bregp}
		In this section, we will prove the two main results, namely \cref{lem:mirrid} and \cref{lem:unineq}, followed by two auxiliary lemmas on some fundamental properties of the Bregman divergence $B_G(\cdot|\cdot)$ (see \cref{eg:mirror-G}), that will help prove the main results. First we state the two supporting lemmas. Recall that given a probability density $\rho\in\ptac(\R^d)$, we write $\phi_{\rho}$ to denote a Brenier potential from $\rho$ to the reference distribution $e^{-g}$. We also recall $G(\cdot)=(1/2)W_2^2(\cdot,e^{-g})$.
		\begin{lemma}\label{lem:Bregprop}
			Let $\rho_1,\rho_2$ be two absolutely continuous probability measures. Then $B_G(\rho_2|\rho_1)\ge 0$ and $B_G(\rho_2|\rho_1)=0$ if and only if $\rho_1=\rho_2$. Further, the map $\rho_2\mapsto B_G(\rho_2|\rho_1)$ is convex for all $\rho_1$.
		\end{lemma}
		
		\begin{lemma}\label{lem:Bregprop2}
			Given probability densities $\rho_1,\rho_2\in\ptac(\R^d)$, let $\phi_{\rho_1}$, $\phi_{\rho_2}$ be Brenier potentials from $\rho_1,\rho_2$ to $e^{-g}$. Also define $\psi_{\rho_1}:=\phi_{\rho_1}^*$ and $\psi_{\rho_2}:=\phi_{\rho_2}^*$.  Then we have  
			$$B_G(\rho_2|\rho_1)=\E_{Y\sim e^{-g}}D_{\phi_{\rho_1}}(\nabla \psi_{\rho_2}(Y)|\nabla \psi_{\rho_1}(Y))=\E_{X\sim \rho_2}D_{\psi_{\rho_1}}(\nabla \phi_{\rho_2}(X)|\nabla \phi_{\rho_1}(X)).$$
			As a result, we have: 
			$$\frac{1}{2}\bar{d}^2 \sup_y \lmx(\nabla^2 \psi_{\rho_1}(y))\ge B_G(\rho_2|\rho_1)\ge \frac{1}{2}\bar{d}^2\inf_y \lmn(\nabla^2 \psi_{\rho_1}(y)),$$
			where $\bar{d}^2:=\E_{X\sim\rho_2}\lVert \nabla \phi_{\rho_2}(X)-\nabla\phi_{\rho_1}(X)\rVert^2$. 
			Similarly, 
			$$\frac{1}{2}\underline{d}^2 \sup_x \lmx(\nabla^2 \phi_{\rho_1}(x)) \ge B_G(\rho_2|\rho_1)\ge\frac{1}{2}\underline{d}^2\inf_x \lmn(\nabla^2 \phi_{\rho_1}(x)),$$
			where $\underline{d}^2:=\E_{Y\sim e^{-g}}\lVert \nabla \psi_{\rho_2}(Y)-\nabla\psi_{\rho_1}(Y)\rVert^2$. 
		\end{lemma}
		
		\subsection{Proof of Main Results}
		
		\begin{proof}[Proof of \cref{lem:mirrid}] 
			By \eqref{eq:baseineq}, we have: 
			\begin{align*}
				B_G(\pi|\rho_1)=\E_{X\sim \pi}(\psi_{\rho_1}^*-\psi_{\pi}^*)(X)+\E_{Y\sim e^{-g}}(\psi_{\rho_1}-\psi_{\pi})(Y), \quad 			
				B_G(\pi|\rho_2)=\E_{X\sim \pi}(\psi_{\rho_2}^*-\psi_{\pi}^*)(X)+\E_{Y\sim e^{-g}}(\psi_{\rho_2}-\psi_{\pi})(Y),
			\end{align*}
			and 
			\begin{align*}
				B_{G_{\pi}}(\pi_1|\pi_2)=\E_{X\sim \pi}(\psi_{\rho_2}^*-\psi_{\rho_1}^*)(X)+\E_{Y\sim \pi_1}(\psi_{\rho_2}-\psi_{\rho_1})(Y).
			\end{align*}
			The conclusion follows from direct computation.
		\end{proof}
		
		\begin{proof}[Proof of \cref{lem:unineq}]
			Let $\psi_{\pi}^*$ be the Brenier potential from $\pi$ to $e^{-g}$. By the change of variable formula, we have: 
			$$\log\pi = -g(\nabla \psi_{\pi}^*)+\log\det(\nabla^2 \psi_{\pi}^*), \quad \mbox{and}\quad \log\rho = -g(\nabla \psi_{\rho}^*)+\log\det(\nabla^2 \psi_{\rho}^*).$$
			By subtracting the above equations, we have: 
			$$\log\frac{\pi}{\rho}=g(\nabla\psi_{\rho}^*)-g(\nabla\psi_{\pi}^*)-\log\frac{\det(\nabla^2 \psi_{\rho}^*)}{\det(\nabla^2 \psi_{\pi}^*)}.$$
			Let us write $Y=\nabla\psi_{\pi}^*(X)$, where $X\sim \pi$. Then $Y\sim e^{-g}$. By integrating the above equation with respect to $\pi$, we get: 
			\begin{align*}
				\;\;\;\;KL(\pi|\rho)&=\E_{X\sim \pi} \left[g(\nabla\psi_{\rho}^*(X))-g(\nabla\psi_{\pi}^*(X))-\log\frac{\det(\nabla^2 \psi_{\rho}^*(X))}{\det(\nabla^2 \psi_{\pi}^*(X))}\right] \\ &=\E_{ Y\sim e^{-g}}\left[g(\nabla\psi_{\rho}^*\circ \nabla\psi_{\pi}(Y))-g(Y)-\log\frac{\det(\nabla^2 \psi_{\rho}^*(\nabla\psi_{\pi}(Y)))}{\det(\nabla^2 \psi_{\pi}^*(\nabla\psi_{\pi}(Y)))}\right] \\ &=\E_{Y\sim e^{-g}} D_g(\nabla\psi_{\rho}^*\circ \nabla\psi_{\pi}(Y)|Y) + \int \left(\langle \nabla g(y),\nabla\psi_{\rho}^*\circ\nabla\psi_{\pi}(y)-y\rangle-\log\frac{\det(\nabla^2 \psi_{\rho}^*(\nabla\psi_{\pi}(y)))}{\det(\nabla^2 \psi_{\pi}^*(\nabla\psi_{\pi}(y)))}\right) e^{-g(y)}\,\dd y \\ &\ge \frac{\lambda}{2}\int \lVert \nabla\psi_{\rho}^*\circ \nabla\psi_{\pi}(y)-y\rVert^2 e^{-g(y)}\,\dd y + \int \left(\langle \nabla g(y),\nabla\psi_{\rho}^*\circ\nabla\psi_{\pi}(y)-y\rangle-\log\frac{\det(\nabla^2 \psi_{\rho}^*(\nabla\psi_{\pi}(y)))}{\det(\nabla^2 \psi_{\pi}^*(\nabla\psi_{\pi}(y)))}\right)e^{-g(y)}\,\dd y.
			\end{align*}
			Therefore it suffices to show that 
			\begin{align}\label{eq:toshow1}
				\int \lVert \nabla\psi_{\rho}^*\circ \nabla\psi_{\pi}(y)-y\rVert^2 e^{-g(y)}\,\dd y \ge \frac{2}{\beta} B_G(\pi|\rho),
			\end{align}
			and
			\begin{align}\label{eq:toshow2}
				&\int \left(\langle \nabla g(y),\nabla\psi_{\rho}^*\circ\nabla\psi_{\pi}(y)-y\rangle-\log\frac{\det(\nabla^2 \psi_{\rho}^*(\nabla\psi_{\pi}(y)))}{\det(\nabla^2 \psi_{\pi}^*(\nabla\psi_{\pi}(y)))}\right)e^{-g(y)}\,\dd y\ge 0.
			\end{align}
			
			\emph{Proof of \eqref{eq:toshow1}.} By the change of variable $Y=\nabla\psi_{\pi}^*(X)$ for $X\sim\pi$, we have: 
			\begin{align*}
				\int \int \lVert \nabla\psi_{\rho}^*\circ \nabla\psi_{\pi}(y)-y\rVert^2 e^{-g(y)}\,\dd y & =\int \lVert\nabla\psi_{\rho}^*(x)-\nabla\psi_{\pi}^*(x)\rVert^2\,\pi(x)\dd x\\ &\ge \frac{2}{\beta}\frac{\sup_y \lmx(\nabla^2 \psi_{\rho}(y))}{2} \int \lVert\nabla\psi_{\rho}^*(x)-\nabla\psi_{\pi}^*(x)\rVert^2\,\pi(x)\dd x \\ & \ge \frac{2}{\beta}B_G(\pi|\rho).
			\end{align*}
			Here the second to last inequality uses the assumption that $\sup_y \lmx(\nabla^2 \psi_{\rho}(y))\le \beta$, and the last inequality follows from \cref{lem:Bregprop2}.

			\emph{Proof of \eqref{eq:toshow2}.} By \cref{obs:elemprop}, note that $(\nabla^2\psi_{\pi}^*(\nabla\psi_{\pi}(y)))^{-1}=\nabla^2\psi_{\pi}(y)$. For $y\in\R^d$, let $\lambda_1(y)\ge \ldots \lambda_d(y)$ denote the eigenvalues of the matrix $\nabla^2 \psi_{\rho}^*(\nabla\psi_{\pi}(y))\nabla^2 \psi_{\pi}(y)$. By applying integration by parts and the multivariate chain rule, we have: 
			\begin{align*}
				\int \langle \nabla g(y),\nabla \psi_{\rho}^*\circ \nabla\psi_{\pi}(y)-y\rangle e^{-g(y)}\,\dd y  &=\sum_{i=1}^d \int \partial_i(e^{-g(y)})\left(y_i-(\partial_i\psi_{\rho}^*)(\nabla\psi_{\pi}(y))\right)\,\dd y \\ &=\sum_{i=1}^d \int \left(\frac{\partial}{\partial y_i}\left((\partial_i\psi_{\rho}^*)(\nabla\psi_{\pi}(y))\right)-1\right) e^{-g(y)}\,\dd y \\ &=\sum_{i=1}^d \int\left(\sum_{j=1}^d \frac{\partial^2}{\partial y_i\partial y_j}\psi_{\rho}^*(\nabla\psi_{\pi}(y))\frac{\partial^2}{\partial y_i\partial y_j}\psi_{\pi}(y))-1\right)e^{-g(y)}\,\dd y\\ &=\int \left(\mbox{Trace}(\nabla^2 \psi_{\rho}^*(\nabla \psi_{\pi}(y))\nabla^2 \psi_{\pi}(y))-d\right) e^{-g(y)}\,\dd y =\sum_{i=1}^d \int (\lambda_{i}(y)-1)e^{-g(y)}\,\dd y.
			\end{align*}
			On the other hand, we also have 
			$$\int \log\det\frac{\nabla^2\psi_{\rho}^*(\nabla\psi_{\pi}(y))}{\nabla^2\psi_{\pi}^*(\nabla\psi_{\pi}(y))}e^{-g(y)}\,\dd y=\sum_{i=1}^d \int \big(\log{\lambda_i(y)}\big)e^{-g(y)}\,\dd y.$$
			Therefore, 
			\begin{align*}
				&\;\;\;\;\int \left(\langle \nabla g(y),\nabla\psi_{\rho}^*\circ\nabla\psi_{\pi}(y)-y\rangle-\log\frac{\det(\nabla^2 \psi_{\rho}^*(\nabla\psi_{\pi}(y)))}{\det(\nabla^2 \psi_{\pi}^*(\nabla\psi_{\pi}(y)))}\right)e^{-g(y)}\,\dd y =\sum_{i=1}^d \int(\lambda_i(y)-1-\log{\lambda_i(y)})e^{-g(y)}\,\dd y.
			\end{align*}
			The conclusion now follows from the fact that $\log{z}\le z-1$ for all $z>0$. 
		\end{proof}
		
		\subsection{Proof of Auxiliary Lemmas}
		\begin{proof}[Proof of \cref{lem:Bregprop}]
			By Brenier's Theorem \cite{brenier1991polar,knott1984optimal}, we have for any probability density $\rho\in\ptac(\R^d)$, we have: 
			$$G(\rho)=\sup_{\phi\in L^1(\rho)}\left[\int \left(\frac{1}{2}\lVert x\rVert^2-\phi(x)\right)\,\rho(x)\dd x+\int \left(\frac{1}{2}\lVert y\rVert^2-\phi^*(y)\right)e^{-g(y)}\,\dd y\right].$$
			By the optimality of $\phi_{\rho_2}$ and $\psi_{\rho_1}=\phi_{\rho_1}^*$ (the Fenchel conjugate of $\phi_{\rho_1}$), we therefore have: 
			\begin{align*}
				G(\rho_2) & \ge \int \left(\frac{1}{2}\lVert x\rVert^2-\phi_{\rho_1}(x)\right)\rho_2(x)\dd x+\int \left(\frac{1}{2}\lVert y\rVert^2-\psi_{\rho_1}(y)\right)e^{-g(y)}\dd y  =G(\rho_1)+\int\left(\frac{1}{2}\lVert x\rVert^2-\phi_{\rho_1}(x)\right)(\rho_2-\rho_1)(x)\dd x.
			\end{align*}
			By plugging in the above inequality into the definition of $B_G(\rho_2|\rho_1)$ (see \cref{eg:mirror-G}), we get $B_G(\rho_2|\rho_1)\ge 0$. Here equality holds if and only if $\phi_{\rho_2}$ and $\phi_{\rho_1}$ are equal a.e. up to translations. This implies $\rho_1=\rho_2$ by the uniqueness of optimal transport maps \cite{McCann1995}. The convexity of $\rho_2\mapsto B_G(\rho_2|\rho_1)$ follows from the convexity of the map $\rho_2\mapsto W_2^2(\rho_2,e^{-g})$.
		\end{proof}
		
		\begin{proof}[Proof of \cref{lem:Bregprop2}]
			From direct simplification, it follows that 
			\begin{align}\label{eq:baseineq}
				B_G(\rho_2|\rho_1) &=\int \left(\frac{1}{2}\lVert x\rVert^2-\phi_{\rho_2}(x)\right)\,\rho_2(x)\,dx+\int \left(\frac{1}{2}\lVert y\rVert^2-\psi_{\rho_2}(y)\right)e^{-g(y)}\,dy \nonumber \\ &\quad - \int \left(\frac{1}{2}\lVert x\rVert^2-\phi_{\rho_1}(x)\right)\rho_{1}(x)\dd x-\int \left(\frac{1}{2}\lVert y\rVert^2-\psi_{\rho_1}(y)\right)e^{-g(y)}\,dy \nonumber \\ &\quad -\int \left(\frac{1}{2}\lVert x\rVert^2-\phi_{\rho_1}(x)\right)(\rho_2-\rho_1)(x)\dd x \nonumber \\ &=\E_{X\sim\rho_2}(\phi_{\rho_1}-\phi_{\rho_2})(X)+\E_{Y\sim e^{-g}}(\psi_{\rho_1}-\psi_{\rho_2})(Y).
			\end{align}
			Recall the definition of Bregman divergence on the Euclidean space (see \eqref{eq:EucliDBreg}). As $X\overset{d}{=}\nabla \psi_{\rho_2}(Y)$ and $\phi_{\rho_2}(\nabla \psi_{\rho_2}(y))+\psi_{\rho_2}(y)=\langle y,\nabla \psi_{\rho_2}(y)\rangle$ (by \cref{obs:elemprop}), we can further simplify $B_G(\rho_2|\rho_1)$ as 
			\begin{align*}
				\E_{Y\sim e^{-g}}\big(\phi_{\rho_1}(\nabla \psi_{\rho_2}(Y))+\psi_{\rho_1}(Y)-\phi_{\rho_2}(\nabla\psi_{\rho_2}(Y))-\psi_{\rho_2}(Y)\big) &=\E_{Y\sim e^{-g}} (\phi_{\rho_1}(\nabla \psi_{\rho_2}(Y))+\psi_{\rho_1}(Y)-\langle Y,\nabla \psi_{\rho_2}(Y)\rangle) \\ &=\E_{Y\sim e^{-g}} D_{\phi_{\rho_1}}(\nabla \psi_{\rho_2}(Y)|\nabla \psi_{\rho_1}(Y)).
			\end{align*}
			By a similar computation, replacing $Y=\nabla \phi_{\rho_2}(X)$, $X\sim \rho_2$ in \eqref{eq:baseineq}, we get: 
			$B_G(\rho_2|\rho_1)=\E_{X\sim \rho_2}D_{\psi_{\rho_1}}(\nabla \phi_{\rho_2}(X)|\nabla \phi_{\rho_1}(X))$.	The inequalities now follow from the following standard property of Bregman divergences for strictly convex $\phi$, namely, 
			$$\frac{1}{2}\sup_z \lmx(\nabla^2\phi(z))\lVert x-\nabla \phi^*(y)\rVert^2 \ge D_{\phi}(x|\nabla \phi^*(y))\ge \frac{1}{2}\inf_z \lmn(\nabla^2\phi(z))\lVert x-\nabla \phi^*(y)\rVert^2.$$
		\end{proof}

		\section{Proofs for Section~\ref{sec:regcon}}\label{sec:pfmainres}
		In this section, we prove \cref{thm:dualcon}--\cref{thm:ptcon}. Our key tool will be the one-step EVI \cref{lem:basicineq}.
		\begin{proof}[Proof of \cref{thm:dualcon}]
			Recall from \eqref{eq:measupdate} that $\hr_k:=(\nabla \wps_k)\#e^{-g}$. Using the change of variable formula, we have $\hr_k(\nabla \wps_k(y)) \det(\nabla^2 \wps_k(y)) = e^{-g(y)}$. Note the following equation from the discretization \eqref{eq:discretedual}.
			\begin{align*}
				\wps_{k+1}(y) - \wps_k(y) = -\eta_k\big(f(\nabla \wps_k(y)) + \log \hr_k(\nabla \wps_k(y)) \big)
			\end{align*}
			
			Now integrate both sides with respect to $Y  \sim e^{-g}$ to get:
			\begin{align*}
				\E_{Y \sim e^{-g}}[ \wps_{k+1}(Y) - \wps_k(Y) ] &= -\eta_k \E_{Y \sim e^{-g}}[ f(\nabla \wps_k(Y)) + \log \hr_k(\nabla \wps_k(Y))] \\
				& = - \eta_k \E_{Y \sim e^{-g}}\big[ \log \frac{\hr_k(\nabla \wps_k(Y))}{e^{-f(\nabla \wps_k(Y))}} \big] \\
				& = - \eta_k \E_{X \sim \hr_k}\big[ \log \frac{\hr_k(X)}{e^{-f(X)}} \big] = -\eta_kKL\big( \hr_k | e^{-f} \big)   
			\end{align*}
			where the last line uses the fact that if $X = \nabla \wps_k(Y)$, $Y\sim e^{-g}$, then $X\sim \hr_k$.
			By summing over $k$, we get: 
			\begin{align}
				\frac{1}{S_T} \sum_{k=0}^{T-1} \eta_kKL(\hr_k|e^{-f}) = \frac{1}{S_T} \E_{Y \sim e^{-g}} [ \wps_{0}(Y) - \wps_{T}(Y) ].
			\end{align}
			The conclusion now follows from Jensen's inequality.
		\end{proof}
		
		\begin{proof}[Proof of \cref{thm:regbd1}]
			Recall the parameter specifications in the statement of \cref{thm:regbd1}. Using \cref{lem:basicineq} and the non-negativity of the Bregman divergence (see \cref{lem:Bregprop}) and the KL divergence, we then get: 
			\begin{align*}
				\sum_{k=0}^T (KL(\wrh_k|e^{-f})-KL(\pi|e^{-f})) &\le \sqrt{T}\sum_{k=0}^T (B_G(\pi|\wrh_k)-B_G(\pi|\wrh_{k+1}))+\frac{A T}{2m\sqrt{T}} - \sum_{k=0}^T KL(\pi|\wrh_k) \le \sqrt{T}\left(B_G(\pi|\wrh_0)+\frac{A}{2m}\right).
			\end{align*}
		\end{proof}
		
		\begin{proof}[Proof of \cref{thm:regbd2}]
			Note that by \cref{lem:unineq}, we have:
			$KL(\pi|\rho_k)\ge \frac{\lambda}{M_k} B_G(\pi|\rho_k)$. 
			
			\emph{Part (i).} Set $S_0=0$. Observe that 
			$S_{k+1}-\lambda/M_k=1/\eta_k-\lambda/M_k=S_k$. 
			By using \cref{lem:basicineq}, coupled with the above observation and the corresponding parameter specifications in part (i), we get: 
			\begin{align*}
				\sum_{k=0}^T (KL(\wrh_k|e^{-f})-KL(\pi|e^{-f}))  &\le \sum_{k=0}^T S_{k+1}(B_G(\pi|\wrh_k)-B_G(\pi|\wrh_{k+1}))+\frac{1}{2}\sum_{k=0}^T \frac{m_{k+1}^{-1}}{S_{k+1}}\int \lVert \boldsymbol{\xi}_k(y)\rVert^2\pi_k(y)\dd y-\frac{\lambda}{M_k} B_G(\pi|\wrh_k) \\ &=\sum_{k=0}^T (S_k B_G(\pi|\wrh_k) - S_{k+1} B_G(\pi|\wrh_k))+\frac{1}{2}\sum_{k=0}^T \frac{m_{k+1}^{-1}}{S_{k+1}}\int \lVert \boldsymbol{\xi}_k\rVert^2\,d\pi_k \\ &\le \frac{1}{2}\sum_{k=0}^T \frac{m_{k+1}^{-1}}{S_{k+1}}\int \lVert \boldsymbol{\xi}_k(y)\rVert^2\pi_k(y)\dd y.
			\end{align*}
			The last line uses $S_0=0$ and the non-negativity of $B_G$ (see \cref{lem:Bregprop}).
			
			\noindent \emph{Part (ii).} As $\sum_{k=0}^T (k+1)^{-1}\le 1+\log{(T+1)}$, the conclusion follows from part (i) after replacing $M_k$, $m_k$, and $\int \lVert \boldsymbol{\xi}_k\rVert^2\,d\pi_k$ with $M$, $m$, and $A$ respectively.
		\end{proof}
		
		\begin{proof}[Proof of \cref{thm:ptcon}]
			\noindent \emph{Part (i).} Set $\pi=e^{-f}$ and note that $KL(\rho_k|e^{-f})\ge 0$. Using \cref{lem:basicineq} and $KL(e^{-f}|\rho_k)\ge \frac{\lambda}{M} B_G(\pi|\rho_k)$ from \cref{lem:unineq}, we get 
			\begin{align*}
				0 &\le \left(\frac{1}{\eta_k}-\frac{\lambda}{M}\right)B_G(e^{-f}|\rho_k)-\frac{1}{\eta_k}B_G(e^{-f}|\rho_{k+1})+\frac{1}{2m}\eta_k A \\ \implies B_G(e^{-f}|\rho_{k+1})&\le \left(1-\eta_k\frac{\lambda}{M}\right)B_G(e^{-f}|\rho_k)+\frac{1}{2m}\eta_k^2 A \\ \implies \frac{1}{\eta_k}\frac{B_G(e^{-f}|\rho_{k+1})-B_G(e^{-f}|\rho_k)}{B_G(e^{-f}|\rho_k)} &\le -\frac{\lambda}{M} + \frac{1}{2m}\eta_k A.
			\end{align*}
			The result follows by taking $\limsup$ as $\eta_k\to 0$ on both sides.
			
			\noindent \emph{Part (ii).} Define $x_k:=B_G(e^{-f}|\rho_k)$ for $k=0,1,2,\ldots ,T+1$. As $\eta_k=(C M x_0\log{T})/(\lambda  T)$, by part (i), we have:
			$$x_{k+1}\le \left(1-\frac{C x_0\log{T}}{T}\right)x_k+ \frac{C^2 M^2 A x_0^2}{2\lambda^2 m}\left(\frac{\log{T}}{T}\right)^2.$$
			By a standard induction argument, we get: 
			\begin{align*}
				x_{T}& \le \left(1-\frac{C x_0 \log{T}}{T}\right)^T x_0 + \frac{C^2 M^2 A x_0^2}{2\lambda^2 m}\left(\frac{\log{T}}{T}\right)^2 \frac{T}{C x_0\log{T}}\\ &\le \left(1-\frac{C x_0\log{T}}{T}\right)^T x_0+ \frac{C M^2 A x_0}{2\lambda^2 m}\frac{\log{T}}{T} \le \left(\frac{1}{T^{C x_0}}+\frac{C M^2 A}{2\lambda^2 m}\frac{\log{T}}{T}\right)x_0.
			\end{align*}
			The last line follows from the inequality $(1-x)^a\le e^{-ax}$ for $x\in (0,1)$ and $a>0$. This completes the proof. 
		\end{proof}
		
		\section{Proofs from Section~\ref{sec:applications}}\label{sec:pfappl}
		
		\begin{proof}[Proof of \cref{prop:varrep}]
			We note that 
			$$\tilde{F}(\nabla \psi)=\int \big(f(\nabla \psi(y))-g(y)-\log\det(\nabla^2\psi(y))\big)e^{-g(y)} \dd y.$$
			We compute the first variation of $\tilde{F}$ at a point $\nabla \psi$ in the direction, say $\nabla \Theta\in \mathrm{Tan}_{e^{-g}}$ where $\Theta\in \mathcal{C}_c^{\infty}$. In other words, we are interested in finding $A_{\nabla\psi}:\R^d\to\R^d$ such that 
			$$\lim_{\epsilon\to 0} \frac{\tilde{F}(\nabla \psi+\epsilon \nabla\Theta)-\tilde{F}(\nabla \psi) - \langle A_{\nabla \psi},\nabla \Theta\rangle_{e^{-g}}}{\epsilon} = 0.$$
			In the sequel, we will use the fact that given any positive definite matrix $(\partial/\partial B)\log\det(B)=B^{-1}$ (see \cite[Section A.4.1]{boyd2004convex}). We note that 
			\begin{align*}
				\tilde{F}(\nabla \psi+\epsilon \nabla\Theta)-\tilde{F}(\nabla \psi) &=\epsilon \int \langle (\nabla f)(\nabla \psi(y)),\nabla \Theta(y)\rangle e^{-g(y)}\,\dd y - \epsilon \int \big\langle (\nabla^2 \psi(y))^{-1},\nabla^2 \Theta(y)\big\rangle e^{-g(y)}\,\dd y + o(\epsilon) \\ &=\epsilon \int \langle (\nabla f)(\nabla \psi(y)),\nabla \Theta(y)\rangle e^{-g(y)}\,\dd y - \epsilon \int \big\langle \nabla^2 \psi^*(\nabla \psi(y)),\nabla^2 \Theta(y)\big\rangle e^{-g(y)}\,\dd y + o(\epsilon), 
			\end{align*}
			where the last equality uses $\nabla^2 \psi^*(\nabla \psi(y))=(\nabla^2\psi(y))^{-1}$ (see \cref{obs:elemprop}). Moreover  by using integration by parts and the chain rule, we get: 
			\begin{align*}
				&\;\;\;\;\int \big\langle \nabla^2 \psi^*(\nabla \psi(y)),\nabla^2 \Theta(y)\big\rangle e^{-g(y)}\,\dd y \\ &=\sum_{i=1}^d \int \sum_{j=1}^d \frac{\partial^2}{\partial y_i \partial y_j}\psi^*(\nabla\psi(y)) \frac{\partial^2}{\partial y_i\partial y_j}\Theta(y) e^{-g(y)}\dd y\\ &= \int \Bigg(\sum_{i,j=1}^d  \frac{\partial}{\partial y_i}g(y)\frac{\partial^2}{\partial y_i \partial y_j}\psi^*(\nabla\psi(y))\frac{\partial}{\partial y_j}\Theta(y) - \sum_{i,j,k=1}^d  \frac{\partial^3}{\partial y_i \partial y_j \partial y_k}\psi^*(\nabla \psi(y)) \frac{\partial^2}{\partial y_i \partial y_k}\psi(y)\frac{\partial}{\partial y_j}\Theta(y)\Bigg) e^{-g(y)}\dd y.
			\end{align*}
			Define the vector 
			$\mathbf{u}_{\nabla\psi}^{\top}(y):=(u_{\nabla \psi,1}(y),\ldots ,u_{\nabla\psi,d}(y))$,
			where 
			$$u_{\nabla \psi,j}(y):=\sum_{i,k=1}^d \frac{\partial^3}{\partial y_i \partial y_j \partial y_k}\psi^*(\nabla \psi(y)) \frac{\partial^2}{\partial y_i \partial y_k}\psi(y).$$
			By combining the above observations, we get that 
			$$A_{\nabla\psi}(y)=(\nabla f)(\nabla\psi(y))-\nabla^2\psi^*(\nabla\psi(y))\nabla g(y)+\mathbf{u}_{\nabla\psi}(y).$$
			Therefore, the stationary condition for the optimization problem \eqref{eq:targetoptim} is given by 
			$$\nabla\tilde{\psi}(y)-\nabla\psi_1(y)=-\tau \nabla^2\psi_1(y)\left((\nabla f)(\nabla\psi_1(y))-\nabla^2\psi^*(\nabla\psi_1(y))\nabla g(y)+\mathbf{u}_{\nabla\psi_1}(y)\right).$$
			It thus suffices to show that $\nabla^2\psi_1(y)\mathbf{u}_{\nabla\psi_1}(y)=-\nabla(\log\det(\nabla^2\psi_1(y))$. To wit, note that by the chain rule, we get:
			$$\frac{\partial}{\partial y_j}(\log\det(\nabla^2\psi_1(y))=\sum_{i,k=1}^d \frac{\partial^2}{\partial y_i\partial y_k}\psi_1^*(\nabla\psi_1(y))\frac{\partial^3}{\partial y_i \partial y_k \partial y_j}\psi_1(y).$$
			Moreover, as $\nabla^2\psi_1^*(\nabla\psi_1(y))=(\nabla^2 \psi_1(y))^{-1}$ by \cref{obs:elemprop}, we have: 
			\begin{align*}
				\sum_{i,k=1}^d \frac{\partial^2}{\partial y_i\partial y_k}\psi_1^*(\nabla\psi_1(y))\frac{\partial^2}{\partial y_k\partial y_i}\psi_1(y)=d \qquad \implies \frac{\partial}{\partial y_j}\left(\sum_{i,k=1}^d \frac{\partial^2}{\partial y_i\partial y_k}\psi_1^*(\nabla\psi_1(y))\frac{\partial^2}{\partial y_k\partial y_i}\psi_1(y)\right)=0.
			\end{align*}
			By applying the chain rule, the LHS of the above display simplifies to 
			\begin{align*}
				&\;\;\;\; \frac{\partial}{\partial y_j}\left(\sum_{i,k=1}^d \frac{\partial^2}{\partial y_i\partial y_k}\psi_1^*(\nabla\psi_1(y))\frac{\partial^2}{\partial y_k\partial y_i}\psi_1(y)\right) \\ 
				& = \sum_{i,k=1}^d \frac{\partial^2}{\partial y_i\partial y_k}\psi_1^*(\nabla\psi_1(y))\frac{\partial^3}{\partial y_i\partial y_k\partial y_j}\psi_1(y) + \sum_{i,k,\ell=1}^d \frac{\partial^3}{\partial y_i \partial y_k\partial y_{\ell}}\psi_1^*(\nabla\psi_1(y))\frac{\partial^2}{\partial y_{\ell}\partial y_j}\psi_1(y)\frac{\partial^2}{\partial y_i\partial y_k}\psi_1(y) \\ & =  \frac{\partial}{\partial y_j}\log\det(\nabla^2\psi_1(y))+(\nabla^2 \psi_1(y)\mathbf{u}_{\nabla\psi_1}(y))_j.
			\end{align*}
			This completes the proof.
		\end{proof}
		
		\section{Additional discussions on implementation}\label{sec:numvar}
		
		\paragraph*{On variational inference} We provide a simple illustration of \cref{alg:VI} below. Suppose that the target $e^{-f}$ is the logistic distribution with location parameter $0$ and scale parameter $1$. Then $$f'(x)=2\sigma(x)-1, \quad \mbox{and} \quad f''(x)=2\sigma(x)(1-\sigma(x)), \quad \mbox{where} \quad \sigma(x)=(1+e^{-x})^{-1}.$$
		We run \cref{alg:VI} with $\lambda=1$. Expectations under $Y\sim N(0,\lambda)\equiv N(0,1)$ are approximated using $1000$ Monte-Carlo draws. The initializers are $m_0=10$ and $\sigma_0=1$. The step-sizes $\{\eta_k\}_{k\ge 0}$ are chosen adaptively as follows: 
		$$\eta_k=\frac{1}{2}\frac{\sigma_{k-1}}{1+|\sigma_{k-1}^2\E_{Y\sim N(0,1)}f''(\sigma_{k-1}Y+m_{k-1})-1|}.$$
		This will ensure that $\sigma_k>0$. Let $(m,\sigma)$ denote the mean and the variance of the Gaussian variational approximation for the logistic target. Then by \cite[Equation 1.9]{Katsevich2024}, $(m,\sigma)$ satisfy the stationary conditions $$\E_{X\sim N(m,\sigma^2)} f'(X)=0, \quad \mbox{and} \quad \E_{X\sim N(m,\sigma^2)} f''(X)=\sigma^{-2}.$$
		Let us define 
		$\mbox{Err}_{k,1}:=\E_{X\sim N(m_k,\sigma_k^2)} f'(X) \quad \mbox{and} \quad \mbox{Err}_{k,2}:=\E_{X\sim N(m_k,\sigma_k^2)} f''(X)-\sigma_k^{-2}$. 
		We should then expect $\mbox{Err}_{k,1}$ and $\mbox{Err}_{k,2}$ to converge to $0$ as $k\to\infty$. \cref{fig:vi} illustrates this phenomenon. Even with moderate number of iterations ($k=50$), we see that $\mbox{Err}_{k,1}$ and $\mbox{Err}_{k,2}$ are very close to $0$. This phenomenon persists across different initializations.

		\begin{figure}
			\centering
			\includegraphics[width=0.5\linewidth]{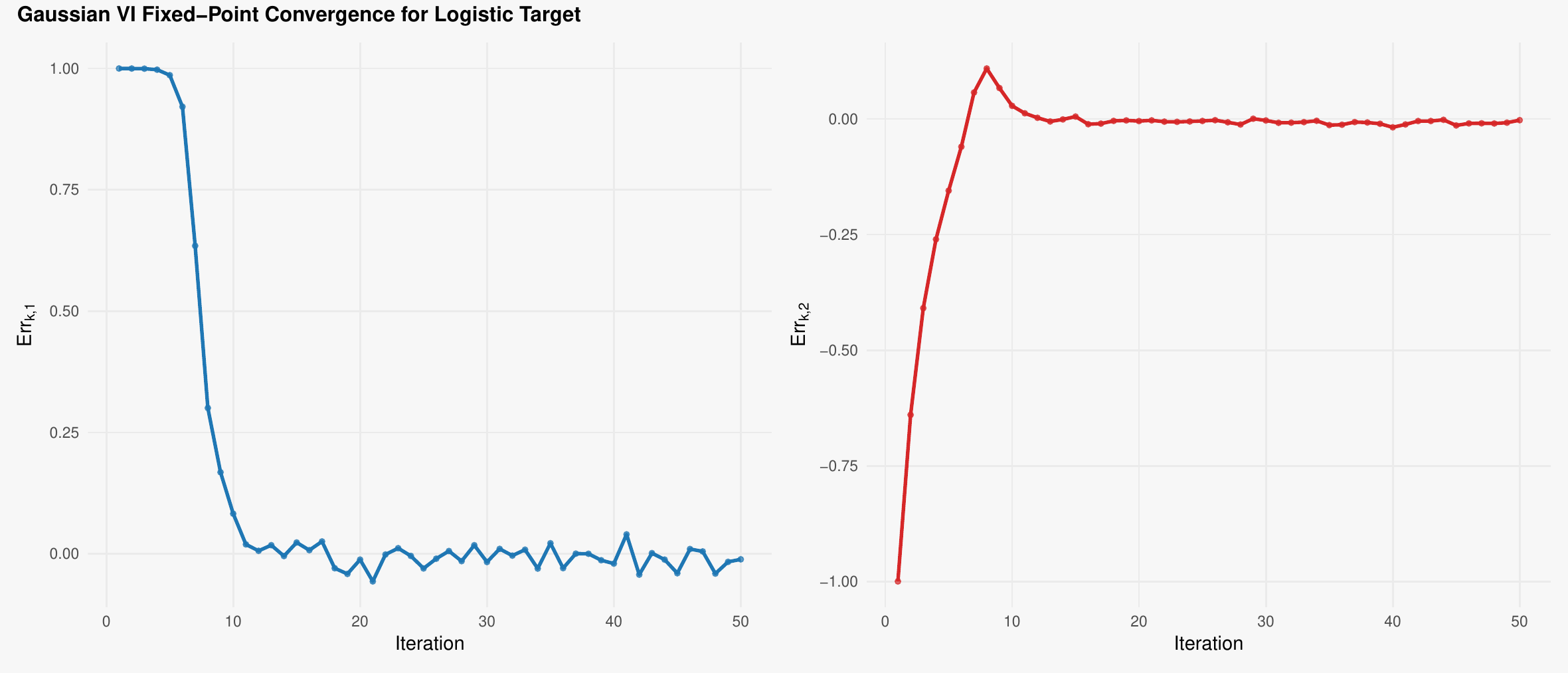}
			\caption{Convergence to stationarity of \cref{alg:VI} with logistic target}
			\label{fig:vi}
		\end{figure}
		
		\paragraph*{A distillation-free implementation} In \cref{sec:simulations}, distillation was introduced purely for memory overflow considerations and is not intrinsic to our algorithm. We discuss a computationally tractable method that avoids distillation in the context of our logistic regression based method (see \cref{alg:neural-pde-logistic-reg}). We hope to expand on this approach in future work. For now, we present it as an alternate algorithm which can be summarized as follows: 
		
		\noindent \emph{Step 1}: Run the logistic regression version of the parabolic Monge-Amp\`{e}re (PMA) flow (see \cref{alg:neural-pde-logistic-reg}) for a small number of steps, say $B$ (e.g., $5$--$6$), with reference distribution $e^{-g}$. Note that $B$ being small ensures that the memory does not overflow and that the numerical errors do not accumulate much.
		
		\noindent \emph{Step 2}: Suppose $\rho_B$ denotes the distribution after $B$ steps. Use that as the reference distribution, and \emph{restart} the parabolic Monge-Amp\`{e}re (PMA) flow with new reference distribution $\rho_B$ (instead of $e^{-g}$). Run it again for $B$ more steps to get $\rho_{2B}$. 
		
		\noindent \emph{Step 3}: Repeat step 2 by changing the reference distribution to $\rho_{2B}$. Keep repeating until convergence. 
		
		\begin{figure}[h]
			\centering
			\begin{tikzpicture}[
				node distance=4cm,
				every node/.style={font=\small},
				box/.style={draw, rounded corners, align=center, minimum height=1.1cm, minimum width=2.6cm}
				]
				
				\node (ref0) {Reference for block $0$: $e^{-g}$};
				\node[box, right of=ref0] (block0) {$B$ PMA steps};
				\node[right of=block0] (rhoB) {$\rho_B$};
				
				\draw[->] (ref0) -- (block0);
				\draw[->] (block0) -- (rhoB);
				
				\node[below=1.6cm of ref0] (ref1) {Reference for block $1$: $\rho_B$};
				\node[box, right of=ref1] (block1) {$B$ PMA steps};
				\node[right of=block1] (rho2B) {$\rho_{2B}$};
				
				\draw[->] (ref1) -- (block1);
				\draw[->] (block1) -- (rho2B);
				
				\node[below=1.6cm of ref1] (ref2) {Reference for block $2$: $\rho_{2B}$};
				\node[box, right of=ref2] (block2) {$B$ PMA steps};
				\node[right of=block2] (rho3B) {$\rho_{3B}$};
				
				\draw[->] (ref2) -- (block2);
				\draw[->] (block2) -- (rho3B);
				
			\end{tikzpicture}
			\caption{\small Parabolic Monge-Amp\`{e}re flow with block-wise refresh after every $B$ steps.}
			\label{fig:newrefresh}
		\end{figure}
		\noindent The above approach eliminates full backpropagation as the potentials from earlier updates can be discarded after each block, thereby avoiding memory overflow issues. Logistic regression itself provides regularization. We refer you to the figure above for an illustration of the new approach. We note that as \cref{alg:neural-pde-logistic-reg} requires only sample access (and not analytic density forms), refreshing the reference distribution after each block is straightforward. Intuitively, we believe that restarting after every $B$ steps brings the reference distribution close to the target $e^{-f}$ as we move from one block to the next. A detailed theoretical treatment of this algorithm  would be an interesting direction for future work.
    \end{appendix}

\end{document}